\documentclass[12pt]{article} % For LaTeX2e
\usepackage{epicpaper,times}

\title{Epistemic Learning from Imprecise Annotation}

\author{Kaizheng Wang \and Siu Lun Chau
}

\date{%
    \small
    Epistemic Intelligence \& Computation Lab\\
    College of Computing \& Data Science\\
    Nanyang Technological University, Singapore \\
\texttt{\{kaizheng.wang,siulun.chau\}@ntu.edu.sg} \\[2ex]
{\normalsize \today}
}

\usepackage{caption}
\usepackage{multirow}
\usepackage{colortbl}
\usepackage{layouts}
\usepackage{graphicx}
\usepackage[normalem]{ulem}
\usepackage{framed}
\usepackage{multicol}
\usepackage{booktabs}
\usepackage{enumitem}
\usepackage{titletoc}
\usepackage{xcolor}
\colorlet{OxfordBlueLight}{green!20}
\usepackage{hyperref}  
\usepackage{cleveref}
\newcommand{\inner}[2]{\left\langle #1,#2\right\rangle}
\usepackage{amsthm}
\usepackage{thmtools}
\usepackage{thm-restate}

\theoremstyle{definition}
\newtheorem{definition}[theorem]{Definition}

\begin{document}
\maketitle

\begin{abstract}
Imprecise annotations may support several plausible labelling distributions, yet learning methods often resolve this ambiguity into a single predictive distribution. This can obscure what the annotation evidence leaves unresolved. We introduce \emph{epistemic learning from credal supervision}, a framework that uses convex sets of plausible labelling distributions, called credal sets, as supervision and learns sets of predictive distributions. We instantiate the framework with the \emph{pessimistic--optimistic credal classifier} (POCC), which combines a shared backbone with two classification heads trained to minimise worst-case and best-case losses over the supervision sets. Their outputs define a predictive credal set whose spread provides an uncertainty score. We also show how credal labels can be obtained through a simple relaxation of existing probabilistic labels, reducing commitment to their precise probability assignments. This construction admits closed-form inner optimisation under cross-entropy loss, enabling efficient training. Assuming the supervision sets contain the true conditional label distributions, and other regularity assumptions, we establish a finite-sample generalisation bound for the averaged predictor with an explicit penalty for supervision imprecision. We evaluate POCC using human annotator disagreement and teacher predictions, alongside label smoothing as a controlled proxy for annotation imprecision. Across these settings, POCC achieves a favourable balance of predictive accuracy, calibration, and uncertainty-based selective classification versus competitive baselines.
\end{abstract}

% ----------------------------------------
\section{Introduction}

Supervised classifiers usually learn from a single target label or label distribution per instance. Yet annotator disagreement, incomplete label information, and differing interpretations may support several plausible targets~\citep{peterson2019human,schmarje2022one,destercke2022uncertain}. Reducing this evidence to a single target can erase distinctions relevant to deciding whether to predict or defer for further assessment. Under majority voting, for example, unanimous agreement and substantial annotator disagreement may produce the same hard label, despite providing different support for that decision.

Existing approaches accommodate annotation imprecision in different ways. Soft labels represent supervision through a single probability distribution, obtained, for example, by aggregating annotations or smoothing hard labels~\citep{geng2016label,szegedy2016rethinking}. These targets retain information about the relative support for different classes, but they do not explicitly represent which alternative label distributions remain plausible. Another approach represents supervision through a \emph{credal set}: a convex set of admissible label distributions~\citep{lienen2021credal,lienen2023conformal}. Such sets allow the annotation evidence to leave some probabilities unspecified. However, they are often used to select a training target or define a robust learning objective, after which the model returns a single predictive distribution. The remaining challenge is to carry information about unresolved alternatives from the supervision into the predictions.

Our aim is to develop an \emph{epistemic learner} that represents both what the annotation evidence supports and what it leaves unresolved. We call this objective \emph{epistemic learning from credal supervision}: learning sets of predictive distributions from credal labels. Here, the predictive epistemic uncertainty (EU) of interest concerns ambiguity about which labelling distribution is warranted for a given input, as reflected in the imprecision of the observed annotations. This supervision-induced uncertainty is distinct from outcome variability under a specified distribution and complements uncertainty arising from fitting a model to finite data~\citep{hullermeier2021aleatoric}. Repeated fitting, as in deep ensembles~\citep{lakshminarayanan2017simple}, captures variation across fitted models but does not, by itself, recover plausible labelling distributions discarded when constructing fixed training targets.

To represent this supervision-induced uncertainty efficiently, we introduce the \emph{pessimistic--optimistic credal classifier} (POCC), whose two prediction heads share a backbone. The pessimistic head minimises the worst-case loss over the distributions admitted by the supervision, while the optimistic head minimises the best-case loss. Their outputs and all mixtures between them form a predictive credal set. The separation between the heads supplies an uncertainty score, while their average provides a probability prediction. Figure~\ref{FIG: Concept} provides an overview of our framework and proposed method. We summarise our contributions below:
\begin{figure}[t]
\centering
\includegraphics[width=\linewidth]{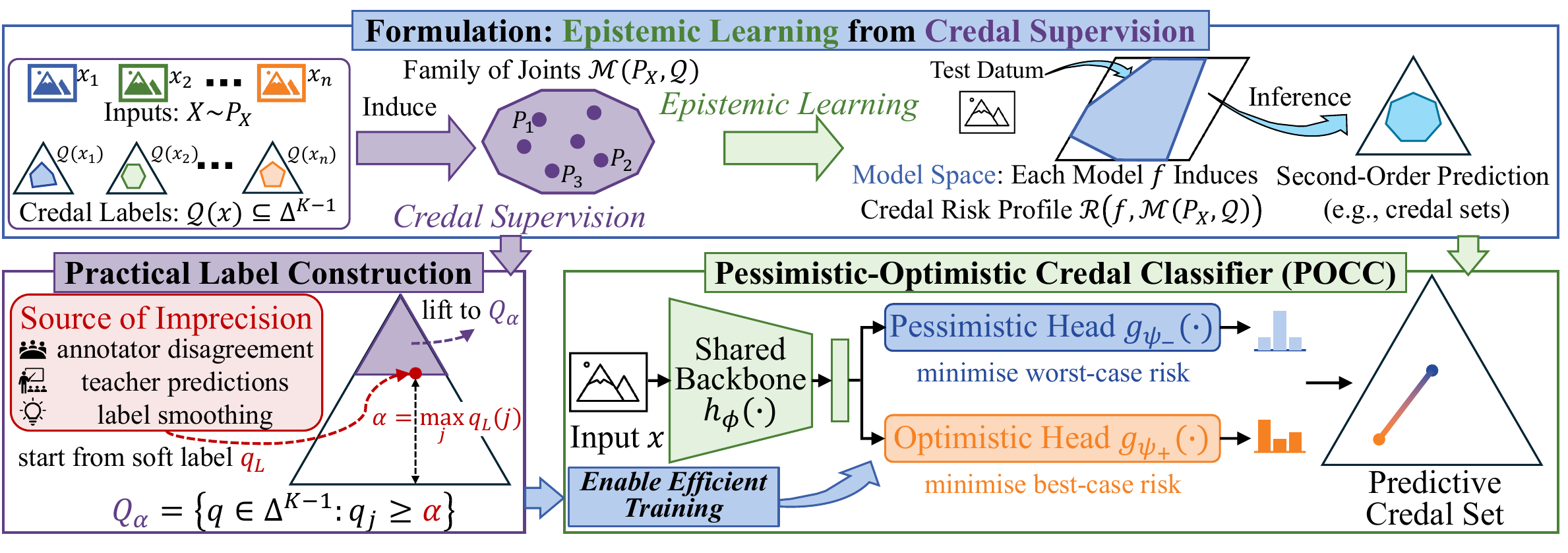}
\caption{Overview of epistemic learning from credal supervision. Top: the framework relates credal supervision to sets of predictive distributions. Bottom left: practical construction of credal labels from imprecise annotations. Bottom right: POCC uses a shared backbone and pessimistic and optimistic heads to form a predictive credal set.} 
\label{FIG: Concept}
\end{figure}

$\bullet\quad$ \textbf{A framework for epistemic learning from credal supervision.} We formalise learning from sets of admissible label distributions and identify the representation of unresolved annotation evidence at prediction time as a central learning objective.

$\bullet\quad$ \textbf{An efficient classifier with a learning-theoretic analysis.} We derive closed-form solutions for POCC's inner optimisation problems under cross-entropy loss and the proposed credal-label construction. Assuming the supervision sets contain the true conditional label distributions, and under appropriate boundedness conditions, we establish a finite-sample generalisation bound for the averaged predictor with an explicit penalty for supervision imprecision.

$\bullet\quad$ \textbf{An empirical evaluation across three supervision settings.} We construct credal labels from human annotator disagreement and teacher predictions, and use label smoothing as a controlled proxy for annotation imprecision. Across the evaluated settings, POCC achieves a favourable overall balance among predictive accuracy, calibration, and uncertainty-based selective classification compared with competitive baselines.

% ----------------------------------------
\section{Background}
\label{sec: BG}
% ----------------------------------------
\textbf{Classical Supervised Classification.}
% ----------------------------------------
Consider a \(K\)-class classification problem with input space \(\mathcal{X}\subseteq\mathbb{R}^d\), label space \(\mathcal{Y}=\{1,\dots,K\}\), and the \((K\!-\!1)\)-dimensional probability simplex
\(
\Delta^{K-1} = \big\{q\in\mathbb{R}_{\ge 0}^K\!: \sum_{k=1}^K q_k = 1\big\}.
\)
Assuming data \((X, Y)\) are drawn i.i.d. from a true but unknown joint distribution \(P_0\) over \(\mathcal{X}\times\mathcal{Y}\), the standard objective is to learn a first-order probabilistic classifier \(f:\mathcal{X}\to\Delta^{K-1}\) minimising the population risk
\begin{equation}
R_0(f;P_0) = \mathbb{E}_{(X,Y)\sim P_0}\big[\ell(f(X), Y)\big],
\label{eq:PopulationRisk}
\end{equation}
where \(\ell:\Delta^{K-1}\times\mathcal{Y}\to\mathbb{R}\) denotes the loss function (e.g., cross-entropy). In practice, given a finite dataset \(\{(x_i, y_i)\}_{i=1}^N\), the empirical risk \(\hat{R}(f;\hat{P}_N) = \frac{1}{N}\sum_{i=1}^N \ell(f(x_i), y_i)\) is minimised instead. To assess performance on unseen data, learning theory studies the gap between empirical and population risks, known as the \emph{generalisation gap}. Under suitable assumptions on the loss and model complexity, bounds on this gap establish when increasing data and improving optimisation bring the learned predictor's risk towards the best attainable in its model class~\citep{shalev2014understanding,bartlett2002rademacher}, motivating our analysis in Section~\ref{sec: Analysis}. 

For a distributional target \(q\in\Delta^{K-1}\), let
\(\ell(p,q)=\sum\nolimits_{k}q_k\,\ell(p,k)\), which recovers \(\ell(p,k)\) at the one-hot target \(q=e_k\); For cross-entropy, this gives \(\ell(p,q)=-\sum_{k=1}^{K}q_k\log p_k\).

\textbf{Learning from Imprecise Annotations.}
We consider single-label classification in which the annotation evidence may leave several conditional labelling distributions plausible. This differs from label-noise formulations, which model discrepancies between observed and underlying clean labels~\citep{song2022learning,song2025survey}, and from multi-label learning, where several classes may apply simultaneously~\citep{tarekegn2024deep}. A missing annotation supplies no direct label constraint, although missing-label recovery lies outside our scope.

Soft labels and label distribution learning summarise class support in a single probability vector~\citep{geng2016label,schmarje2022one}, while label smoothing deliberately softens hard targets~\citep{szegedy2016rethinking}. Credal labels retain a set of admissible distributions, allowing the supervision to leave some target probabilities unresolved~\citep{lienen2021label}. They have also been used as pseudo-labels in semi-supervised and self-supervised learning~\citep{lienen2023conformal,rodriguez2023self,bordini2024self}.

Theoretical work has studied learning guarantees under different credalised settings~\citep{campagner2023credal,caprio2024credal}, such as learning from multiple data generating sources. Closest to our work is \citet{venkatesh2026structured}, where they developed a structured credal framework that distinguishes uncertainty in label information from uncertainty in the input distribution. Much of this work learns ordinary probabilistic predictors; our interest also concerns how supervision imprecision informs predictive epistemic uncertainty.

% ----------------------------------------
\textbf{Representing Predictive Epistemic Uncertainty.}
% ----------------------------------------
A probability vector describes predictions over class outcomes. Uncertainty about which probability vector is warranted can itself be represented by a distribution over probability vectors or by a set of plausible vectors~\citep{hullermeier2021aleatoric,tpamiWANG}. We use \emph{second-order representation} to refer broadly to these richer descriptions of predictive uncertainty. A credal set, in particular, specifies which distributions are admissible without assigning probabilities to its members.

Bayesian and ensemble methods obtain such representations through uncertainty over model parameters or variation across fitted models~\citep{gal2016dropout,lakshminarayanan2017simple,hobbhahn2022fast}. Evidential and related approaches learn representations of uncertainty over class probabilities directly~\citep{sensoy2018evidential,nipossibilistic}, while credal predictors return sets of plausible probability vectors~\citep{wang2024CredalEnsembles,hofmanefficient}. These predictive representations do not, by themselves, specify how the constraints supplied by imprecise annotations should inform predictive epistemic uncertainty.

% Section~\ref{sec: Method} formalises learning predictive credal sets from credal supervision via admissible conditional labelling distributions and their population risks. Further related work appears in Appendix~\ref{App:RelatedWork}. 

% ----------------------------------------
\section{Epistemic Learning from Credal Supervision}
\label{sec: Method}
% ----------------------------------------
Imprecise annotations constrain plausible conditional labelling distributions. To carry these constraints into prediction, we first show that they induce a credal set of joint distributions and a range of population risks. The endpoints then define complementary learning objectives.

For each input \(x\in\mathcal X\), let \(Q(x)\subseteq\Delta^{K-1}\) contain the conditional labelling distributions compatible with the annotation evidence. A conventional Markov kernel assigns one conditional distribution to each input; we extend this representation by allowing a set of distributions.

\begin{definition}[Credal Markov Kernel]
\label{def:CMV}
A \emph{credal Markov kernel} is a measurable set-valued map \(Q:\mathcal X\rightrightarrows\Delta^{K-1}\) such that \(Q(x)\) is non-empty, closed, and convex for every \(x\in\mathcal X\).
\end{definition}

Hard and soft labels correspond to singleton sets: \(Q(x)=\{e_k\}\) for class \(k\), where \(e_k\) places all mass on \(k\), and \(Q(x)=\{q\}\) for a soft target \(q\). Non-singleton sets leave several conditional laws admissible. Fixing the input marginal \(P_X\) induces the following joint credal set and range of risks.

\begin{proposition}[Induced Joint Credal Set and Risk Profile]
\label{prop:main}
Assume that \(\mathcal X\) is a Polish space---a space admitting a complete separable metric---equipped with its Borel \(\sigma\)-algebra. For a fixed input marginal \(P_X\) and a credal Markov kernel \(Q\), the family
\[
\mathcal M(P_X,Q)
=
\left\{
P\in\mathcal P(\mathcal X\times\mathcal Y):
P(\,\cdot\,\times\mathcal Y)=P_X,\;
P_{Y\mid X=x}\in Q(x)
\text{ for }P_X\text{-a.e. }x
\right\}
\]
is non-empty, convex, and weakly compact, and hence is a closed convex credal set. Here, \(\mathcal P(\mathcal X\times\mathcal Y)\) denotes the probability distributions on the input--label space. Suppose a measurable predictor \(f\) has measurable per-class losses. If \(\mathbb E_X[\max_k|\ell(f(X),k)|]<\infty\) for \(X\sim P_X\), its \emph{credal risk profile} is the compact interval
{\small
\begin{equation}
\mathcal R(f,\mathcal M)
:=
\big\{
\mathbb E_P[\ell(f(X),Y)]:P\in\mathcal M
\big\}
=
[R_+(f),R_-(f)],
\qquad \mathcal M:=\mathcal M(P_X,Q),
\label{eq:risk_profile}
\end{equation}
}
where
{\small
\begin{equation}
\begin{aligned}
R_+(f)&=\min_{P\in{\mathcal M(P_X,Q)}}E_P\big[\ell(f(X),Y)\big] = \mathbb E_X\big[\min_{q\in Q(X)}\ell(f(X),q)\big],\\
R_-(f)&=\max_{P\in{\mathcal M(P_X,Q)}}E_P\big[\ell(f(X),Y)\big] = \mathbb E_X\big[\max_{q\in Q(X)}\ell(f(X),q)\big].
\end{aligned}
\label{eq: instance_risk}
\end{equation}
}
Both endpoints are attained by distributions in \(\mathcal M(P_X,Q)\).
\end{proposition}
The joint credal set supports a behavioural interpretation: its members are permissible probability models for evaluating decisions through expected loss, without commitment to a single model~\citep{levi1980enterprise}. This formulation also connects to the ambiguity sets used in distributionally robust optimisation (DRO)~\citep{ben2013robust,duchi2021statistics,chen2026bulkcalibrated}. Whereas standard DRO formulations minimise worst-case expected loss to seek robustness, our credal formulation retains the full range of expected losses induced by the permissible models, without privileging either endpoint as the sole learning criterion. In our setting, the input marginal remains fixed, and the geometry of the joint credal set is determined by the conditional annotation constraints \(Q(x)\). The proof is provided in Appendix~\ref{App:ProofProposition1}.

% ----------------------------------------
\subsection{Pessimistic--Optimistic Credal Classifier}
\label{sec:EfficientModel}
One natural approach to learning from credal supervision is to train a separate predictor for each admissible conditional labelling law. This leads to an infinite-task learning problem~\citep{pmlr-v89-brault19a,chen2024imprecise,singh2024domain}: the family of laws is generally uncountable, and approximating it through sampled tasks can introduce substantial computational overhead. We instead introduce the \emph{pessimistic--optimistic credal classifier} (POCC), which uses the two endpoints of the credal risk profile to define complementary learning objectives.

\paragraph{Pessimistic and Optimistic Learning.}
The upper risk \(R_-\) and lower risk \(R_+\) measure worst- and best-case performance over the set of admissible joint distributions, yielding classical pessimistic and optimistic learning criteria~\citep{guillaume2017maximum,hullermeier2019learning}. POCC minimises these risks with two predictors and retains both outputs in a predictive credal set. This represents the consequences of unresolved annotation evidence, rather than selecting a single compromise between robustness and optimism.

\paragraph{Architecture and Training.}
\label{sec:model}
As illustrated in Figure~\ref{FIG: Concept}, POCC combines a shared backbone with two classification heads. Let \(h_\phi\) denote the backbone and \(g_{\psi_-}\) and \(g_{\psi_+}\) the pessimistic and optimistic heads. For \(\theta=(\phi,\psi_-,\psi_+)\), their predictive distributions are
\begin{equation}
p_{-,\theta}(x)
:=
f_{-,\theta}(x)
=
g_{\psi_-}(h_\phi(x)),
\qquad
p_{+,\theta}(x)
:=
f_{+,\theta}(x)
=
g_{\psi_+}(h_\phi(x)),
\label{eq:pocc_heads}
\end{equation}
with both outputs taking values in \(\Delta^{K-1}\). Given training data \(\{(x_i,Q_i)\}_{i=1}^{N}\), where \(Q_i:=Q(x_i)\), we jointly train the two heads by solving
\begin{equation}
\min_\theta\;
\frac{1}{N}
\sum_{i=1}^{N}
\Big[
\max_{q\in Q_i}
\ell(p_{-,\theta}(x_i),q)
+
\min_{q\in Q_i}
\ell(p_{+,\theta}(x_i),q)
\Big].
\label{eq:total_loss}
\end{equation}
For each instance, the inner maximisation selects the least favourable admissible label distribution for the pessimistic head, while the inner minimisation selects the most favourable one for the optimistic head. Both heads minimise their respective losses, and both contribute to learning the shared representation.

This objective applies whenever the inner extrema over the credal labels can be evaluated. For the construction introduced in Section~\ref{sec: CredalLabel}, both admit closed-form solutions under cross-entropy loss, allowing efficient training without iterative inner optimisation.

\paragraph{Prediction and Epistemic Uncertainty.}
To retain both learned views at prediction time, we return the set of all mixtures of the two head outputs:
\begin{equation}
C_\theta(x)
=
\big\{
(1-\lambda)p_{-,\theta}(x)
+
\lambda p_{+,\theta}(x)
:
\lambda\in[0,1]
\big\}.
\label{eq:credal_prediction}
\end{equation}
This line segment is the smallest closed convex set that contains both predictions, called their convex hull, and it collapses to a singleton when they agree~\citep{huber1992robust}. When a single predictive distribution is required, we use the midpoint and its most probable class:
\[
\bar f_\theta(x)=\big({p_{-,\theta}(x)+p_{+,\theta}(x)}\big)/{2},
\qquad
\widehat y_\theta(x)=\arg\max_k\bar f_{\theta,k}(x).
\]
We use the spread of \(C_\theta(x)\) as an epistemic uncertainty score reflecting disagreement between the two learning criteria. Specifically, we adopt the \emph{Maximum Mean Imprecision} (MMI) under the total variation distance~\citep{chau2025integral,gonzalezgarcia2026quantification}. For this line segment, MMI reduces to the total variation distance between its endpoints:
% {\small
\begin{equation}
\mathrm{MMI}(x)
=
\frac{1}{2}
\sum_{k=1}^{K}
\left|
p_{+,\theta,k}(x)-p_{-,\theta,k}(x)
\right|.
\label{eq:eu_mmi}
\end{equation}
% }
The score is available directly from the two head outputs and requires no further optimisation. Appendix~\ref{App:mmi} provides the derivation, while Appendix~\ref{App:EUQ} compares it with classical entropy-based alternatives~\citep{abellan2006disaggregated}.

\subsection{Practical Credal-Label Construction and Efficient Optimisation}
\label{sec: CredalLabel}
Credal supervision allows annotation evidence to leave some probabilities unspecified. For example, a statement such as ``at least 70\% probability of cat'' naturally defines a set of compatible labelling distributions. Common supervision signals instead often provide a single reference distribution \(q_L(x)\in\Delta^{K-1}\), obtained from annotator frequencies, teacher predictions, or smoothed labels. We show how to obtain credal labels through a simple relaxation of these readily available targets, reducing commitment to their precise probability assignments while retaining information useful for prediction.

Following~\citet{lienen2021credal}, we retain the probability of the most supported class as a lower bound and leave the allocation of the remaining mass unresolved. Let \(j(x)\in\arg\max_k q_L(k\mid x)\), with ties resolved by a fixed rule, and write \(\alpha(x)=q_L(j(x)\mid x)\). The resulting credal label is
\begin{equation}
\mathcal Q_\alpha(x)
:=
\left\{
q\in\Delta^{K-1}:q_{j(x)}\geq\alpha(x)
\right\}.
\label{eq:LabelConstruction}
\end{equation}
By construction, \(\mathcal Q_\alpha(x)\) contains \(q_L(x)\), as illustrated in Figure~\ref{FIG: Concept}. The relaxation therefore preserves the original target as an admissible distribution while allowing alternatives to its precise probability assignments.

\paragraph{Interpreting Supervision Imprecision.}
To understand what this relaxation represents, observe that
\begin{equation}
\mathcal Q_\alpha(x)
=
\alpha(x)e_{j(x)}
+
\bigl(1-\alpha(x)\bigr)\Delta^{K-1},
\label{eq:CredalGeometry}
\end{equation}
where \(e_{j(x)}\) places all probability mass on class \(j(x)\).
Every admissible distribution reserves mass \(\alpha(x)\) for this
class and leaves the remaining mass unrestricted, including allocation
back to the same class. This is an \(\epsilon\)-contamination set with
\(\epsilon(x)=1-\alpha(x)\): the mass originally assigned outside the
most supported class determines how much probability mass is left
free to be reallocated.

The remaining imprecision is quantified by the largest total variation distance between members of the set: for \(K\geq2\), $$\operatorname{diam}_{\mathrm{TV}}
\bigl(\mathcal Q_\alpha(x)\bigr)
:=
\sup_{q,r\in\mathcal Q_\alpha(x)}
\operatorname{TV}(q,r)
=
1-\alpha(x).$$
For fixed \(j(x)\), decreasing \(\alpha(x)\) enlarges the set. A one-hot target \(q_L(x)=e_{j(x)}\) gives \(\alpha(x)=1\) and the singleton \(\{e_{j(x)}\}\). Although \(q_L(x)\) belongs to the set by design, coverage of the true conditional distribution remains an assumption of Section~\ref{sec: Analysis}.

\paragraph{Closed-Form Inner Optimisation.}
The same mixture representation makes the POCC training objective efficient to evaluate. Cross-entropy is linear in the target distribution, so each head's extremal target is obtained by assigning the unrestricted mass to a single class.

\begin{proposition}[Closed-Form Extreme Label Distributions]
\label{prop:closed_solution}
Fix an input \(x\), abbreviate \(j=j(x)\) and \(\alpha=\alpha(x)\), and let \(p_-\) and \(p_+\) denote the two head predictions with strictly positive coordinates. Choose \(k_{\min}\in\arg\min_k p_{-,k}\) and \(k_{\max}\in\arg\max_k p_{+,k}\). Under cross-entropy loss, an inner maximiser for the pessimistic head and an inner minimiser for the optimistic head are
\begin{equation}
q_-^*=\alpha e_j+(1-\alpha)e_{k_{\min}},\qquad q_+^*=
\alpha e_j+(1-\alpha)e_{k_{\max}}.
\label{eq:closed_solution}
\end{equation}
These expressions also cover \(k_{\min}=j\) or \(k_{\max}=j\), in which case the corresponding target is \(e_j\).
\end{proposition}
The pessimistic target places the unrestricted mass on the class its head considers least probable, thereby maximising the loss. The optimistic target places this mass on the class its head considers most probable, thereby minimising the loss. A proof is provided in Appendix~\ref{App:ProofProposition2}. Each inner problem requires an \(\mathcal O(K)\) scan, with no sampling or iterative optimisation. The dependence on the number of classes matches ordinary cross-entropy training; the additional model cost comes primarily from the second head.

\section{Learning Guarantees under Credal Supervision}
\label{sec: Analysis}
Our analysis examines when imprecise supervision supports informative probabilistic prediction. As a first step, we consider \emph{valid but imprecise supervision}: the annotation constraints may leave the true conditional labelling distribution unresolved, but do not exclude it. Under suitable regularity conditions, we establish a finite-sample cross-entropy risk bound for POCC's averaged predictor, making explicit the contribution of supervision imprecision.

\textbf{Coverage of the True Conditional Distribution.}
Let $p_0(x)\!\!=\!P_0(Y\!=\!\cdot\!\mid\! X\!=\!x)$ denote the true conditional labelling distribution. We formalise supervision validity through the following assumption.
\begin{assumption}[Almost-sure Conditional Coverage]
\label{ass:conditional-coverage}
The credal supervision satisfies \(p_0(x)\in Q(x)\) for \(P_X\)-almost every \(x\).
\end{assumption}
Coverage is a \emph{well-specification assumption on the supervision}, allowing us to isolate the effect of annotation imprecision from that of misspecified constraints. It relaxes exact distributional supervision, $Q(x)=\{p_0(x)\}$, to set membership. This is distinct from the usual sampled-label assumption: the latter specifies how observed labels are drawn from the target conditional distribution, whereas coverage specifies the validity of the supplied annotation constraints. In particular, coverage neither requires deterministic outcomes nor follows merely from i.i.d.\ sampling. When labels are deterministic, $p_0(x)=e_{y^*(x)}$, a correct hard annotation is recovered as the singleton $Q(x)=\{e_{y^*(x)}\}$. More general supervision models that account for stochastic annotation mechanisms and coverage violations, for example through random-set formulations, are left to future work.

By ensuring $P_0\in\mathcal M(P_X,Q)$, coverage connects the credal objectives to the true risk $R_0(f):=\mathbb E_X[\ell(f(X),p_0(X))]$, yielding $R_+(f)\le R_0(f)\le R_-(f)$. Coverage alone, however, does not ensure informative supervision: taking $Q(x)=\Delta^{K-1}$ always satisfies the assumption but leaves the conditional distribution unrestricted. We therefore quantify supervision imprecision through the total-variation diameter $$d_Q(x):=\sup_{q,r\in Q(x)}\tfrac12\|q-r\|_1$$.

\textbf{Generalisation of the Averaged Predictor.}
To relate the training objective in~\cref{eq:total_loss} to this risk, we divide it by two, leaving its minimisers unchanged. Define
\begin{equation}
j_\theta(x):=\frac{1}{2}\big(\max_{q\in Q(x)}\ell(f_{-,\theta}(x),q)+
\min_{q\in Q(x)}\ell(f_{+,\theta}(x),q)
\big),
\label{eq:empirical-joint-objective}
\end{equation}
and let \(\widehat J_{Q,n}(\theta)\!:=\!n^{-1}\sum_{i=1}^{n}j_\theta(X_i)\). Suppose that training returns a parameter \(\widehat\theta_n\) satisfying
$$
\widehat J_{Q,n}(\widehat\theta_n)
\!\le\!\inf_{\theta\in\boldsymbol\Theta}\widehat J_{Q,n}(\theta)+\eta_n,
\nonumber
$$
where \(\eta_n\ge0\) measures the optimisation error. Its averaged prediction is \(\widehat{\bar f}_n\!=\!(f_{-,\widehat\theta_n}\!+\!f_{+,\widehat\theta_n})/2\). To assess this predictor against an ordinary probabilistic model of comparable capacity, we use the predictors that the architecture can represent with identical heads:
\[
\mathcal F_\Delta:=\left\{f:\exists\theta_f\in\boldsymbol\Theta\text{ such that }f_{-,\theta_f}=f_{+,\theta_f}=f\right\}.
\]
The learned heads may differ; only the comparison class is restricted in this way.

\begin{theorem}[Generalisation under Credal Coverage]
\label{theorem:main}
Assume almost-sure conditional coverage, a fixed credal kernel \(Q\), and inputs \(X_1,\ldots,X_n\) drawn independently from \(P_X\). Suppose that \(\mathcal F_\Delta\) is non-empty and, on one common measurable set of inputs with \(P_X\)-probability one, cross-entropy is bounded by \(B\) for all parameters, both heads, and all admissible targets, while \(\max_k[-\log f_k(x)]-\min_k[-\log f_k(x)]\le L\) for every head prediction. Then, with probability at least \(1-\delta\),
\begin{equation}
R_0(\widehat{\bar f}_n)-\inf_{f\in\mathcal F_\Delta}R_0(f)
\le
4\mathfrak R_n(\mathcal J_Q)
+2B\sqrt{\frac{\log(2/\delta)}{2n}}
+\eta_n
+L\,\mathbb E[d_Q(X)],
\label{eq:general_bound}
\end{equation}
where \(\mathcal J_Q:=\{j_\theta:\theta\in\boldsymbol\Theta\}\), and \(\mathfrak R_n(\mathcal J_Q)\) measures the capacity of this loss class through its expected absolute Rademacher complexity, defined in Appendix~\ref{App:ProofTheorem}.
\end{theorem}
The bound separates the effects of model complexity, finite-sample variation, imperfect optimisation, and annotation imprecision. The final term is specific to the unresolved supervision: even when coverage holds, different admissible conditional distributions can assign different risks to the same predictor. Its contribution decreases as the credal labels become more precise and vanishes when \(Q(x)\!=\!\{p_0(x)\}\) almost surely. The proof uses convexity of cross-entropy to compare the averaged predictor with its heads, coverage and diameter bounds to control discrepancies between true and credal risks, and uniform convergence to relate empirical and population objectives; see Appendix~\ref{App:ProofTheorem}.

\textbf{Specialisation to our Credal Construction.}
For \(Q_{\alpha,j}(x)=\{q\in\Delta^{K-1}:q_{j(x)}\ge\alpha(x)\}\) from~\cref{eq:LabelConstruction}, coverage has the explicit interpretation \(p_{0,j(x)}(x)\ge\alpha(x)\) almost surely: the probability lower bound supplied by the annotation must be valid for the true conditional distribution. If all logits in the model class are bounded in absolute value by \(A\), substituting \(B=\log K+2A\), \(L=2A\), and \(d_Q(x)=1-\alpha(x)\) into~\cref{eq:general_bound} gives the imprecision penalty \(2A\,\mathbb E[1-\alpha(X)]\). Appendix~\ref{App:specialization} provides the complete specialised bound and its derivation.

This specialisation makes the relationship between validity and precision explicit. Increasing \(\alpha(x)\) narrows the credal label and reduces the imprecision penalty, but also strengthens the coverage requirement. Containing the reference target \(q_L(x)\) does not by itself establish coverage of \(p_0(x)\). The guarantee therefore applies when the annotation constraints are valid, and quantifies the effect of the uncertainty they leave unresolved. It concerns the cross-entropy risk of the averaged predictor; coverage or calibration of the predictive credal set requires a separate analysis.

% ----------------------------------------
\section{Experiments}
\label{sec: Exp}
% ----------------------------------------
\textbf{Evaluation and Baselines.}
We study annotation imprecision from two sources---annotators' disagreement and teacher predictions in a knowledge-distillation setting---and include label smoothing as a controlled proxy. We report test accuracy (ACC), expected calibration error (ECE)~\citep{guo2017calibration}, and the usefulness of EU estimates for \emph{selective classification}, in which predictions with high uncertainty are rejected. The area under the accuracy--rejection curve (AUARC)~\citep{jaegercall} summarises accuracy on retained predictions across rejection rates. Higher ACC and AUARC and lower ECE are better.
For EU evaluation, we adopt selective classification rather than out-of-distribution (OOD) detection as the latter targets distributional shift in input space, whereas our EU stems from label-space annotation imprecision. OOD evaluation was also recently criticised as an inappropriate evaluation for EU estimation~\citep{li2025position} (see Appendix~\ref{App:WhyNotOOD}). 
To summarise performance, the \emph{balanced quality score} (BQS) averages ACC, \(1-\mathrm{ECE}\), and AUARC after separate min-max normalisation to \([0,1]\) across methods for each seed. 
All experiments are repeated over \(10\) independent seeds (\(\mathrm{seed}\!=\!1, \!\ldots\!, 10\)).
% All experiments are repeated over \(10\) seeds.

Our work introduces a new problem setting, \emph{epistemic learning from credal supervision}, with POCC as one instantiation. To our knowledge, no existing method is designed specifically for this setting. To provide informative comparisons, we adapt established methods for \emph{epistemic uncertainty quantification} to the same supervision sources. Specifically, the baselines are trained on the reference distributional target \(q_L\), while POCC uses the credal label constructed from that same target via~\cref{eq:LabelConstruction}. This design gives all methods access to the same underlying annotation information while accommodating their different supervision representations.
Specifically, we consider four single-pass second-order predictors: evidential deep learning (EDL)~\citep{sensoy2018evidential}, Laplace bridge with Bayesian deep networks (LbBnn)~\citep{hobbhahn2022fast}, efficient credal predictions via decalibration (Decali)~\citep{hofmanefficient}, and Dirichlet-approximated possibilistic posterior predictions (DAPPr)~\citep{nipossibilistic}. We also include inference-heavy but widely used baselines: Monte Carlo dropout (MCDO)~\citep{gal2016dropout} with five forward passes and deep ensembles (DE)~\citep{lakshminarayanan2017simple} with five members. To ensure a fair comparison, batch size, training epochs, optimiser, and learning rate schedule are matched across methods. Implementation details are in Appendix~\ref{App:ImplementationDetails}.

\textbf{Annotator Disagreement.}
For class labels \(y_1,\ldots,y_S\) supplied by \(S\) annotators, we use their empirical distribution \(q_L=S^{-1}\sum_{s=1}^{S}e_{y_s}\) as the distributional training target \(q_L\).

\textbf{\emph{Setup.}}
CIFAR-10H~\citep{peterson2019human} supplies about 50 annotations per image for the CIFAR-10 test set. We train all methods from scratch with a ResNet-18~\citep{he2016deep} backbone, and evaluate them on the last 10,000 samples of the original CIFAR-10 training set to avoid data leakage. On MiceBone and TreeVersity~\citep{schmarje2022one}, we use DenseNet-121~\citep{huang2017densely} pre-trained 
on ImageNet. For these two datasets, MCDO and LbBnn are excluded due to the lack of a compatible pre-trained model and out-of-memory errors on an NVIDIA RTX 5090, respectively.

\textbf{\emph{Results.}}
\tablename~\ref{tab:result} reports performance on CIFAR-10H, with BQS comparison across datasets in \figurename~\ref{fig: DisaggreeResult} and 
additional results in \tablename~\ref{Table: DisagreementDetails}. POCC achieves the highest BQS across all three datasets. Although MCDO and Decali obtain lower ECE, this comes at the cost of substantially worse ACC and AUARC, reflected in their markedly lower BQS.
\begin{table}[!htbp]
\centering
\captionof{table}{CIFAR-10H results (\%). Best in bold; second best underlined.}
\label{tab:result}
\small
\setlength\tabcolsep{2.5pt}
\begin{tabular}{lccc|c}
\toprule
Method & ACC $\uparrow$ & ECE $\downarrow$ & AUARC $\uparrow$ & BQS $\uparrow$ \\
\midrule
EDL    & $85.0 \scriptstyle{\pm 0.4}$ & $9.5 \scriptstyle{\pm 0.3}$ & $96.7 \scriptstyle{\pm 0.1}$ & $33.1 \scriptstyle{\pm 8.3}$ \\
LbBnn  & $\mathbf{87.0 \scriptstyle{\pm 0.8}}$ & $6.6 \scriptstyle{\pm 0.3}$ & $95.9 \scriptstyle{\pm 0.3}$ & $54.8 \scriptstyle{\pm 6.3}$ \\
Decali & $84.3 \scriptstyle{\pm 1.0}$ & $\underline{1.0 \scriptstyle{\pm 0.2}}$ & $94.9 \scriptstyle{\pm 0.5}$ & $35.4 \scriptstyle{\pm 4.5}$ \\
DAPPr  & $86.6 \scriptstyle{\pm 0.3}$ & $6.9 \scriptstyle{\pm 0.2}$ & $\mathbf{97.2 \scriptstyle{\pm 0.1}}$ & $68.3 \scriptstyle{\pm 5.9}$ \\
\rowcolor{OxfordBlueLight}
Ours   & $\mathbf{87.0 \scriptstyle{\pm 0.4}}$ & $2.9 \scriptstyle{\pm 0.3}$ & $\underline{97.1 \scriptstyle{\pm 0.2}}$ & $\mathbf{86.0 \scriptstyle{\pm 6.1}}$ \\
\midrule
MCDO   & $85.6 \scriptstyle{\pm 0.7}$ & $\mathbf{0.9 \scriptstyle{\pm 0.1}}$ & $96.6 \scriptstyle{\pm 0.3}$ & $71.0 \scriptstyle{\pm 10.6}$ \\
DE     & $\mathbf{87.0 \scriptstyle{\pm 0.3}}$ & $4.3 \scriptstyle{\pm 0.2}$ & $96.9 \scriptstyle{\pm 0.1}$ & $\underline{77.6 \scriptstyle{\pm 5.4}}$ \\
\bottomrule
\end{tabular}
\end{table}

\begin{figure}
\centering
\includegraphics[width=0.5\linewidth]{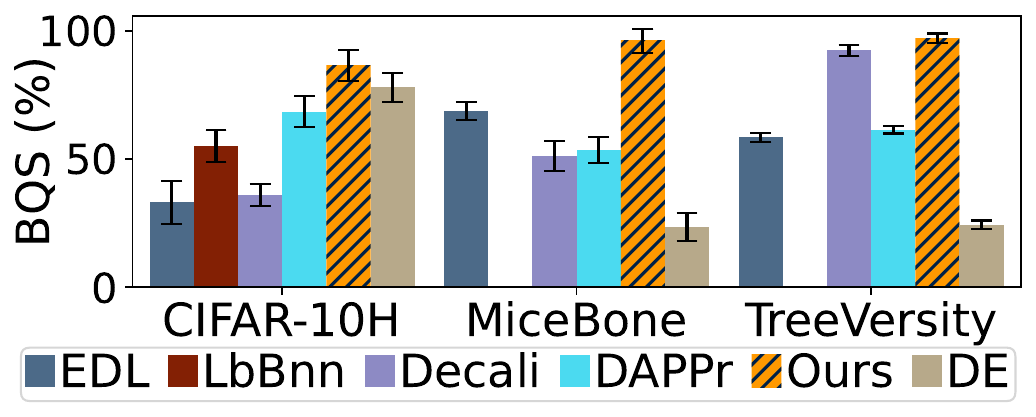}
\captionof{figure}{BQS comparison under annotator disagreement on distinct datasets.}
\label{fig: DisaggreeResult}
\end{figure}

\textbf{Teacher Predictions.}
In the knowledge distillation setting~\citep{hinton2015distilling}, a pre-trained teacher's soft predictions serve as distributional labels \(q_L\).

\textbf{\emph{Setup.}}
We use publicly available teachers~\citep{chenyaofo_pytorch_cifar_models} pre-trained on CIFAR-10 and CIFAR-100~\citep{krizhevsky2009learning}. To assess the generalisability of our findings, we consider three teacher architectures: ResNet-20, ShuffleNetV2-0.5×~\citep{ma2018shufflenet}, and MobileNetV2-x0-5~\citep{sandler2018mobilenetv2}. Student models adopt ResNet-18 for CIFAR-10 and WRN-28-4~\citep{WidResNet} for CIFAR-100, with temperature \(T=2.5\)~\citep{hinton2015distilling}. We additionally ablate on unscaled teacher predictions (\(T=1\)). Implementation details are provided in Appendix~\ref{sub-App:Teacher}.

\textbf{\emph{Results.}}
\tablename~\ref{Table: LabelFromTeacherMain} and \figurename~\ref{Figure: TeacherAblation} show that POCC achieves the best mean ACC, ECE, and BQS across all three teacher architectures and both datasets under temperature-scaled predictions (\(T=2.5\)). Under unscaled predictions (\(T=1\)), POCC achieves the best mean BQS across all three teacher architectures and both datasets. Detailed results for all settings are provided in Tables~\ref{Table: AblationTeacher} and~\ref{Table: AblationTeacher_original}.

\begin{table}[!htbp]
\caption{Teacher-prediction results (\%) with ResNet-20. Best in bold; second best underlined.}
\label{Table: LabelFromTeacherMain}
\centering
\small
\setlength\tabcolsep{4pt}
\begin{tabular}{lccc|c||ccc|c}
\toprule
\multirow{2}{*}{Method} & \multicolumn{4}{c||}{CIFAR-10} & \multicolumn{4}{c}{CIFAR-100} \\
\cmidrule(lr){2-5} \cmidrule(lr){6-9}
& ACC $\uparrow$ & ECE $\downarrow$ & AUARC $\uparrow$ & BQS $\uparrow$ & ACC $\uparrow$ & ECE $\downarrow$ & AUARC $\uparrow$ & BQS $\uparrow$ \\
\midrule
EDL    & $93.4 \scriptstyle{\pm 0.2}$ & $9.0 \scriptstyle{\pm 0.2}$ & $99.1 \scriptstyle{\pm 0.1}$ & $51.6 \scriptstyle{\pm 3.8}$ & $33.1 \scriptstyle{\pm 5.7}$ & $\underline{23.3 \scriptstyle{\pm 3.9}}$ & $65.8 \scriptstyle{\pm 6.1}$ & $30.2 \scriptstyle{\pm 2.9}$ \\
LbBnn  & $\underline{95.1 \scriptstyle{\pm 0.1}}$ & $28.1 \scriptstyle{\pm 0.2}$ & $\underline{99.2 \scriptstyle{\pm 0.0}}$ & $54.2 \scriptstyle{\pm 3.3}$ & $\underline{73.1 \scriptstyle{\pm 0.2}}$ & $63.1 \scriptstyle{\pm 0.2}$ & $63.4 \scriptstyle{\pm 0.4}$ & $31.9 \scriptstyle{\pm 2.4}$ \\
Decali & $93.5 \scriptstyle{\pm 0.1}$ & $8.5 \scriptstyle{\pm 0.1}$ & $98.0 \scriptstyle{\pm 0.1}$ & $29.7 \scriptstyle{\pm 2.5}$ & $71.0 \scriptstyle{\pm 0.2}$ & $27.8 \scriptstyle{\pm 0.2}$ & $73.3 \scriptstyle{\pm 0.3}$ & $64.2 \scriptstyle{\pm 1.8}$ \\
DAPPr  & $93.4 \scriptstyle{\pm 0.2}$ & $14.6 \scriptstyle{\pm 0.2}$ & $99.0 \scriptstyle{\pm 0.0}$ & $42.8 \scriptstyle{\pm 3.1}$ & $72.5 \scriptstyle{\pm 0.4}$ & $52.8 \scriptstyle{\pm 0.3}$ & $78.3 \scriptstyle{\pm 0.3}$ & $55.1 \scriptstyle{\pm 1.6}$ \\
\rowcolor{OxfordBlueLight}
Ours   & $\mathbf{95.5 \scriptstyle{\pm 0.1}}$ & $\mathbf{3.0 \scriptstyle{\pm 0.1}}$ & $\mathbf{99.5 \scriptstyle{\pm 0.0}}$ & $\mathbf{99.9 \scriptstyle{\pm 0.3}}$ & $\mathbf{77.2 \scriptstyle{\pm 0.2}}$ & $\mathbf{10.6 \scriptstyle{\pm 0.4}}$ & $\mathbf{90.6 \scriptstyle{\pm 0.2}}$ & $\mathbf{100.0 \scriptstyle{\pm 0.0}}$ \\
\midrule
MCDO   & $93.5 \scriptstyle{\pm 0.2}$ & $\underline{8.5 \scriptstyle{\pm 0.2}}$ & $99.0 \scriptstyle{\pm 0.1}$ & $50.2 \scriptstyle{\pm 4.8}$ & $71.1 \scriptstyle{\pm 0.3}$ & $28.2 \scriptstyle{\pm 0.3}$ & $83.2 \scriptstyle{\pm 0.7}$ & $75.5 \scriptstyle{\pm 1.3}$ \\
DE     & $94.2 \scriptstyle{\pm 0.1}$ & $9.5 \scriptstyle{\pm 0.1}$ & $99.2 \scriptstyle{\pm 0.0}$ & $\underline{65.3 \scriptstyle{\pm 3.0}}$ & $72.7 \scriptstyle{\pm 0.3}$ & $29.9 \scriptstyle{\pm 0.3}$ & $\underline{86.5 \scriptstyle{\pm 0.3}}$ & $\underline{79.5 \scriptstyle{\pm 0.7}}$ \\
\bottomrule
\end{tabular}
\end{table}
\begin{figure}[!htbp]
\centering
\includegraphics[width=\linewidth]{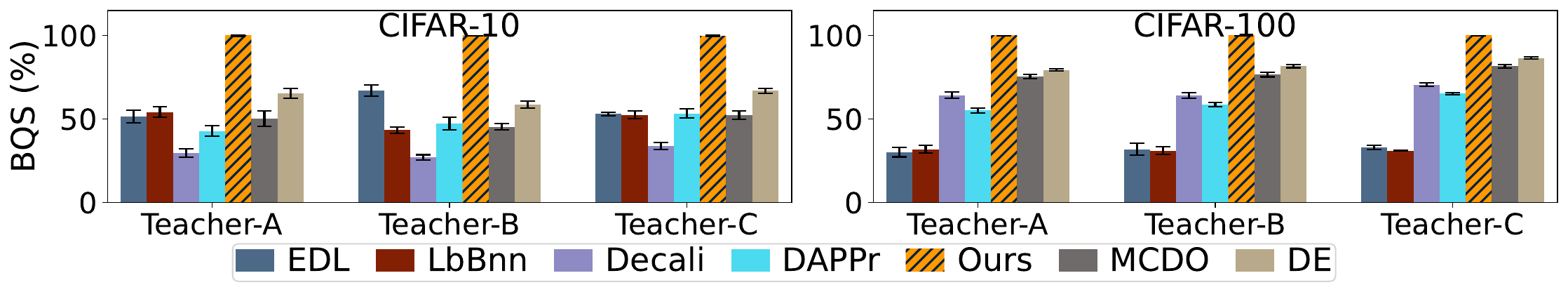}
\caption{BQS across teachers: (A) ResNet-20; (B) ShuffleNetV2-0.5×; (C) MobileNetV2-x0-5.}
\label{Figure: TeacherAblation}
\end{figure}

\textbf{Label Smoothing.}
Label smoothing~\citep{szegedy2016rethinking} redistributes a fraction \(\epsilon\) of a one-hot target's probability mass uniformly across the other classes. Unlike annotator disagreement or teacher predictions, which reflect instance-dependent ambiguity, it applies a fixed perturbation across all samples. It thus serves as a \emph{proxy} for annotation imprecision rather than a direct instantiation of it.

\textbf{\emph{Setup.}}
We use ResNet-18 on CIFAR-10 and WRN-28-4 on CIFAR-100, with label smoothing parameter \(\epsilon\in\{0.05,0.1,0.15,0.2\}\). Configurations are detailed in Appendix~\ref{sub-App:Smooth}.

\textbf{\emph{Results.}}
\tablename~\ref{Table: LabelSmoothingMain} reports results for \(\epsilon=0.05\), with BQS across all \(\epsilon\) visualized in \figurename~\ref{Figure: SmoothingAblation} and detailed results provided in \tablename~\ref{Table: AblationSmoothing}. POCC achieves the highest mean BQS in seven of eight settings; the sole exception is \(\epsilon=0.1\) on CIFAR-10, where it ranks second behind DE (5 members), with the gap attributable to marginal differences in accuracy and AUARC. Statistical significance tests (\tablename~\ref{Table: ST-smoothing}, Appendix~\ref{App:StatisticalTest}) confirm that POCC ranks first statistically across all settings.
\begin{table}[!htpb]
\caption{Label-smoothing results (\%) at \(\epsilon=0.05\) over 10 runs. Best in bold; second best underlined.}
\label{Table: LabelSmoothingMain}
\centering
\small
\begin{tabular}{lccc|c||ccc|c}
\toprule
\multirow{2}{*}{Method} & \multicolumn{4}{c||}{CIFAR-10} & \multicolumn{4}{c}{CIFAR-100} \\
\cmidrule(lr){2-5} \cmidrule(lr){6-9}
& ACC $\uparrow$ & ECE $\downarrow$ & AUARC $\uparrow$ & BQS $\uparrow$ & ACC $\uparrow$ & ECE $\downarrow$ & AUARC $\uparrow$ & BQS $\uparrow$ \\
\midrule
EDL    & $93.4 \scriptstyle{\pm 0.3}$ & $9.1 \scriptstyle{\pm 0.2}$ & $99.1 \scriptstyle{\pm 0.1}$ & $48.6 \scriptstyle{\pm 4.3}$ & $26.3 \scriptstyle{\pm 1.2}$ & $18.7 \scriptstyle{\pm 1.0}$ & $58.5 \scriptstyle{\pm 1.6}$ & $24.3 \scriptstyle{\pm 0.6}$ \\
LbBnn  & $\underline{95.4 \scriptstyle{\pm 0.2}}$ & $29.9 \scriptstyle{\pm 0.3}$ & $99.1 \scriptstyle{\pm 0.1}$ & $53.6 \scriptstyle{\pm 2.6}$ & $\underline{79.3 \scriptstyle{\pm 0.3}}$ & $61.6 \scriptstyle{\pm 0.5}$ & $74.5 \scriptstyle{\pm 0.9}$ & $47.7 \scriptstyle{\pm 1.1}$ \\
Decali & $93.9 \scriptstyle{\pm 0.3}$ & $3.6 \scriptstyle{\pm 0.1}$ & $98.3 \scriptstyle{\pm 0.2}$ & $39.0 \scriptstyle{\pm 5.2}$ & $76.2 \scriptstyle{\pm 0.3}$ & $\underline{3.5 \scriptstyle{\pm 0.4}}$ & $90.7 \scriptstyle{\pm 0.3}$ & $94.0 \scriptstyle{\pm 0.7}$ \\
DAPPr  & $93.9 \scriptstyle{\pm 0.3}$ & $8.9 \scriptstyle{\pm 0.2}$ & $99.1 \scriptstyle{\pm 0.1}$ & $58.5 \scriptstyle{\pm 6.4}$ & $75.1 \scriptstyle{\pm 0.4}$ & $35.3 \scriptstyle{\pm 0.4}$ & $91.9 \scriptstyle{\pm 0.2}$ & $76.5 \scriptstyle{\pm 0.4}$ \\
\rowcolor{OxfordBlueLight}
Ours   & $\mathbf{95.5 \scriptstyle{\pm 0.1}}$ & $\mathbf{1.5 \scriptstyle{\pm 0.1}}$ & $\underline{99.2 \scriptstyle{\pm 0.1}}$ & $\mathbf{90.6 \scriptstyle{\pm 4.7}}$ & $\underline{79.3 \scriptstyle{\pm 0.2}}$ & $\mathbf{2.7 \scriptstyle{\pm 0.2}}$ & $\underline{93.4 \scriptstyle{\pm 0.1}}$ & $\mathbf{99.0 \scriptstyle{\pm 0.3}}$ \\
\midrule
MCDO   & $93.7 \scriptstyle{\pm 0.2}$ & $\underline{3.5 \scriptstyle{\pm 0.1}}$ & $99.1 \scriptstyle{\pm 0.1}$ & $60.2 \scriptstyle{\pm 6.3}$ & $76.1 \scriptstyle{\pm 0.3}$ & $3.8 \scriptstyle{\pm 0.4}$ & $91.8 \scriptstyle{\pm 0.2}$ & $94.9 \scriptstyle{\pm 0.6}$ \\
DE     & $95.2 \scriptstyle{\pm 0.1}$ & $5.5 \scriptstyle{\pm 0.1}$ & $\mathbf{99.4 \scriptstyle{\pm 0.0}}$ & $\underline{89.7 \scriptstyle{\pm 2.3}}$ & $\mathbf{80.3 \scriptstyle{\pm 0.2}}$ & $9.1 \scriptstyle{\pm 0.2}$ & $\mathbf{93.8 \scriptstyle{\pm 0.1}}$ & $\underline{96.4 \scriptstyle{\pm 0.2}}$ \\
\bottomrule
\end{tabular}
\end{table}
\begin{figure}[!htbp]
\centering
\includegraphics[width=\linewidth]{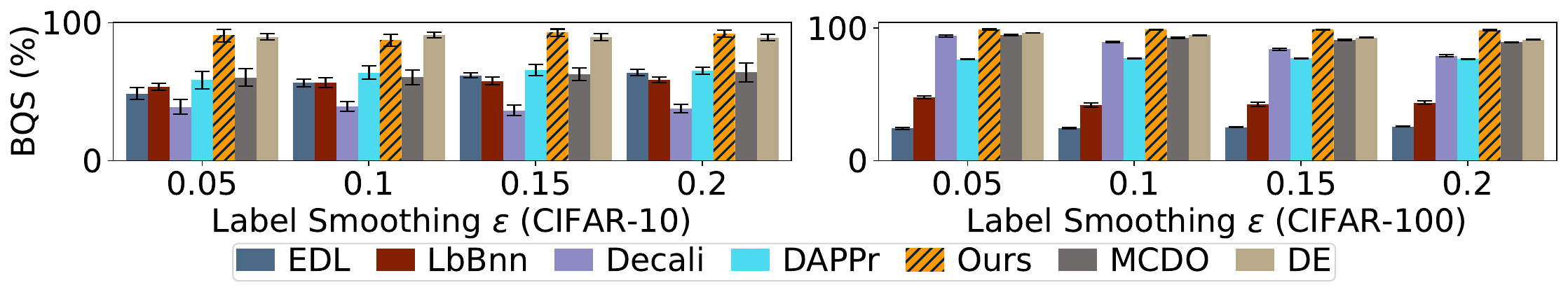}
\caption{BQS comparison across label-smoothing strengths \(\epsilon\).}
\label{Figure: SmoothingAblation}
\end{figure}

\textbf{Discussion.}
Across the three supervision settings, POCC achieves a consistently favourable balance among predictive accuracy, calibration, and uncertainty ranking, obtaining the highest mean BQS in most evaluated configurations. Individual baselines remain competitive and sometimes perform better on particular metrics, but their relative strengths vary across settings.

These findings are consistent with the motivation for explicitly modelling annotation imprecision. Although a distributional target \(q_L\) retains information about relative class support, treating it as fixed does not explicitly represent uncertainty about which alternative labelling distributions remain plausible. POCC incorporates these alternatives into learning through complementary pessimistic and optimistic objectives, and retains both heads' predictions in a predictive credal set. This provides a mechanism through which unresolved supervision can influence the predictions and their associated uncertainty score. The observed balance across metrics suggests that this treatment of supervision ambiguity is practically useful, even when POCC does not achieve the best performance on every individual metric.

\textbf{Additional Analyses.} We provide further empirical evidence that supports the merit of our framework and method. Appendix~\ref{App:StatisticalTest} reports pairwise one-sided Wilcoxon signed-rank tests at the 5\% significance level, confirming that POCC is statistically ranked first in balanced performance across distinct settings. Appendix~\ref{App:FurtherComparison} compares POCC with baselines trained on collapsed hard labels, also showing the best trade-off among ACC, calibration, and EU quantification. Appendices~\ref{App:normalisedAUARC} and~\ref{App:SP} assess balanced performance via normalised AUARC and the EU--Prediction Error correlation, respectively, where POCC consistently achieves the best balance score. Appendix~\ref{App:EUQ} ablates POCC across different EU measures, demonstrating robustness to the choice of measure.

% ----------------------------------------
\section{Conclusion and Future Work}
\label{Sec: conclude}
% ----------------------------------------
We introduced \emph{epistemic learning from credal supervision}, which uses sets of admissible labelling distributions to retain unresolved annotation information in prediction. POCC implements this framework with complementary pessimistic and optimistic heads and closed-form inner optimisation. Under conditional coverage and boundedness assumptions, its averaged predictor admits a finite-sample risk bound with an explicit imprecision term. Across annotator disagreement, teacher predictions, and label smoothing as a controlled proxy, POCC achieves a favourable overall balance of accuracy, calibration, and selective classification.

\textbf{Limitations and Future Work.}
POCC is one instantiation of this framework. Future work will explore alternative predictors and credal-label constructions, and extensions to large foundation models. Our analysis provides a first guarantee under almost-sure conditional coverage, separating the effect of supervision imprecision from that of invalid annotation constraints. This condition need not hold for credal labels constructed from finite annotations or imperfect teacher predictions. A natural extension is to model credal labels as random sets generated by a stochastic annotation mechanism, replacing almost-sure coverage with probabilistic coverage or compatibility in conditional expectation. Such an analysis would account for annotation variability and controlled misspecification alongside the imprecision retained by the supervision.

% \newpage 
% \section*{AI use statement}
% In this work, we used generative AI tools to help polish the writing and check mathematical proofs. We have reviewed all AI-assisted work.

% \section*{Reproducibility Statement}
% Experimental configurations, hyperparameters, and dataset preprocessing procedures are detailed in Appendix~\ref{App:ImplementationDetails}. The datasets used in this 
% work are drawn from existing publicly available benchmarks. To further support reproducibility, the core implementation codes are included in the supplementary materials and will be released 
% publicly upon acceptance under a licence permitting free use for research purposes.

% ----------------------------------------
\bibliography{main}
\bibliographystyle{plainnat}
% ----------------------------------------
\newpage
\appendix
\renewcommand{\thefigure}{A\arabic{figure}}
\setcounter{figure}{0} % Reset the figure counter
\renewcommand{\thetable}{A\arabic{table}}
\setcounter{table}{0} % Reset the table counter
\renewcommand{\theequation}{A\arabic{equation}}
\setcounter{equation}{0} % Reset the table counter
\startcontents[sections]
\printcontents[sections]{}{1}{\section*{Table of Appendix Content}}
\newpage
\section{Related Work}
\label{App:RelatedWork}
Credal sets have attracted increasing attention in the community as a principled framework for representing epistemic uncertainty (EU), with a broad body of work predating their adoption in deep learning and in other statistical applications~\citep{zaffalon2002naive, corani2008learning, corani2012bayesian, maua2017credal, meier2021ensemble,pmlr-v258-manchingal25a, chau2025credaltest,singh2025truthful,yang2026verbalizing}. A central advantage of credal sets lies in their ability to unify set-based and distributional reasoning within a single coherent framework, making them a more natural representation of EU than any single probability distribution~\citep{corani2012bayesian, hullermeier2021aleatoric}. This is because sets more directly encode ignorance as an absence of knowledge~\citep{dubois2002representing}: probability distributions, by contrast, carry implicit assumptions that go beyond merely separating plausible from implausible hypotheses~\citep{lohr2025credal, caprio2024credal}. 

In deep learning, credal predictors have been proposed across several paradigms, though all existing approaches assume that training annotations are precise. These include efficient credal predictions via decalibration (Decali)~\citep{hofmanefficient}, ensemble-based credal predictors~\citep{wang2024CredalEnsembles, wang2025credalwrapper, NguyenCredalEns25, lohr2025credal, wang2026credaldis}, computationally more intensive credal Bayesian neural networks~\citep{caprio2024credal}, and credal predictors designed primarily for interval-valued inputs~\citep{wang2025creinns}. We exclude these methods as baselines for two reasons: (1) they are relatively heavyweight classifiers that conflict with our efficiency objective, and (2) including them would render comparisons unfair given the fundamental difference in computational budget. We also exclude the recent random-set neural network~\citep{manchingal2025randomset}, which outputs belief functions to define credal predictions, as its training requires converting one-hot labels into belief-function form---a procedure incompatible with our experimental setting.
% ----------------------------------------
\section{Proof}
\subsection{Proof of Proposition~\ref{prop:main}}
\label{App:ProofProposition1}
\subsubsection{Setup and Standing Assumptions}
Recall that $\mathcal{X}$ is a Polish space with Borel $\sigma$-algebra $\mathcal{B}(\mathcal{X})$, $\mathcal{Y} = \{1, \dots, K\}$, and 
$$
\Delta^{K-1} = \Big\{ {p} \in \mathbb{R}^K : p_k \ge 0,\, \textstyle\sum\nolimits_{k=1}^{K} p_k = 1 \Big\}.
$$
Since $\mathcal{X}$ is Polish and $\mathcal{Y}$ is finite, $\mathcal{X} \times \mathcal{Y}$ is Polish. We equip $\mathcal{P}(\mathcal{X} \times \mathcal{Y})$ with the \emph{weak topology} (the coarsest making $P \mapsto \int f \, dP$ continuous for every $f \in C_b(\mathcal{X} \times \mathcal{Y})$), which is metrisable. 

We consider a \emph{Credal Markov Kernel} $Q: \mathcal{X}\rightrightarrows \Delta^{K-1}$ subject to the following.
\begin{enumerate}[label=\textbf{Assumption \arabic*.}, leftmargin=*]
  \item $Q(x) \neq \emptyset$ for every $x \in \mathcal{X}$;
  \item $Q(x)$ is closed and convex for every $x \in \mathcal{X}$;
  \item $Q$ is \emph{measurable}, i.e., its graph $\mathrm{Gr}(Q) = \{(x, {p}) : {p} \in Q(x)\}$ belongs to $\mathcal{B}(\mathcal{X}) \otimes \mathcal{B}(\Delta^{K-1})$.
\end{enumerate}
For a fixed input marginal $P_X \in \mathcal{P}(\mathcal{X})$ and a \emph{Credal Markov Kernel} $Q$, the induced family of joint laws is defined as follows:
$$
\mathcal{M}(P_X, Q) =
\bigl\{
P \in \mathcal{P}(\mathcal{X} \times \mathcal{Y}) :
(\pi_{\mathcal{X}})_{\#} P = P_X,\;
P_{Y \mid X = x} \in Q(x),\;
P_X\text{-a.e.\ } x
\bigr\}
$$
Under \emph{Assumptions 1--3}, we prove that the set $\mathcal{M}(P_X, Q)$ is nonempty, convex, and weakly compact in $\mathcal{P}(\mathcal{X} \times \mathcal{Y})$. In particular, it is a \emph{closed convex credal set}. We then show that its admissible risks form a compact interval: the pointwise best-case and worst-case losses can both be realised by joint laws in this set. For this claim, fix a measurable predictor $f$ with measurable per-class losses satisfying $\mathbb{E}_{X\sim P_X}[\max_{k\in\mathcal{Y}}|\ell(f(X),k)|]<\infty$.

\subsubsection{Reparametrisation by Selections}
\begin{lemma}\label{lem:selection}
For $P \in \mathcal{P}(\mathcal{X} \times \mathcal{Y})$ with $(\pi_{\mathcal{X}})_{\#} P = P_X$, define the sub-probability measures $\nu_k(A) = P(A \times \{k\})$ for $A \in \mathcal{B}(\mathcal{X})$.
Then $\nu_k \ll P_X$, and the Radon--Nikodym derivatives $\eta_k = \frac{d\nu_k}{dP_X} \in [0,1]$ satisfy $\sum_{k=1}^K \eta_k = 1$ $P_X$-a.e.
The map
\[
\Phi : P \;\longmapsto\; {\eta} = (\eta_1, \dots, \eta_K) \in L^1(P_X;\, \mathbb{R}^K)
\]
is an affine bijection from $\{P: (\pi_{\mathcal{X}})_{\#} P = P_X\}$ onto $\{{\eta} \in L^1(P_X;\mathbb{R}^K) : {\eta}(x) \in \Delta^{K-1} \; P_X\text{-a.e.}\}$, with conditional $P_{Y \mid X = x} = {\eta}(x)$ and, for every $f \in C_b(\mathcal{X} \times \mathcal{Y})\cong C_b(\mathcal{X})^K$,
\begin{equation}\label{eq:integral-rep}
\int_{\mathcal{X} \times \mathcal{Y}} f \, dP
= \sum_{k=1}^{K} \int_{\mathcal{X}} f(x, k)\, \eta_k(x)\, dP_X(x).
\end{equation}
Consequently, $P \in \mathcal{M}(P_X, Q)$ if and only if ${\eta} \in \mathcal{S}$, where
\[
\mathcal{S}
:= \Bigl\{{\eta} \in L^1(P_X;\mathbb{R}^K) : {\eta}(x) \in Q(x)\; P_X\text{-a.e.}\Bigr\}
\]
is the set of integrable selections of $Q$. 
\end{lemma}

\begin{proof}
Each $\nu_k$ is absolutely continuous w.r.t. $P_X$ because $\nu_k(A) \le P(A \times \mathcal{Y}) = P_X(A)$; the same bound gives $\eta_k \le 1$ a.e., and $\sum_k \nu_k = P_X$ gives $\sum_k \eta_k = 1$ a.e.
Finiteness of $\mathcal{Y}$ is precisely what permits the conditional to be a bounded Radon--Nikodym derivative rather than a regular conditional probability.
Equation~\ref{eq:integral-rep} follows by writing $f = (f(\cdot, 1), \dots, f(\cdot, K))$ and integrating each slice against $\nu_k$; it shows $P$ is determined by ${\eta}$ through integrals that are affine in ${\eta}$, so $\Phi$ is an affine bijection.
The final equivalence is immediate from $P_{Y \mid X = x} = {\eta}(x)$.
\end{proof}

It therefore suffices to show that $\mathcal{S}$ is nonempty, convex, and weakly compact in $L^1$, and then transfer the topology along $\Phi$.
\subsubsection{Formal Proof of Proposition~\ref{prop:main}}
\begin{proof}
\emph{\textbf{Step 1: nonemptiness.}} By \emph{Assumptions 1--3}, $Q$ is a measurable multifunction with nonempty closed values, so the Kuratowski--Ryll-Nardzewski selection theorem yields a measurable ${\eta} : \mathcal{X} \to \Delta^{K-1}$ with ${\eta}(x) \in Q(x)$ for all $x$.
Since ${\eta}$ is bounded, ${\eta} \in \mathcal{S}$; hence $\mathcal{S} \neq \emptyset$ and $\mathcal{M}(P_X, Q) \neq \emptyset$.

\emph{\textbf{Step 2: convexity.}} Let ${\eta}, {\eta}' \in \mathcal{S}$ and $\lambda \in [0,1]$. For a.e.\ $x$, we have ${\eta}(x), {\eta}'(x) \in Q(x)$, and $Q(x)$ is convex by \emph{Assumption 2}, so $\lambda\, {\eta}(x) + (1-\lambda)\, {\eta}'(x) \in Q(x)$.
Thus $\mathcal{S}$ is convex; since $\Phi^{-1}$ is affine by \cref{eq:integral-rep}, $\mathcal{M}(P_X, Q)$ is convex.

\emph{\textbf{Step 3: $\mathcal{S}$ is norm-closed in $L^1$.}} Let ${\eta}^{(m)} \to {\eta}$ in $L^1(P_X;\mathbb{R}^K)$.
Pass to a subsequence with ${\eta}^{(m_j)}(x) \to {\eta}(x)$ for a.e.\ $x$. For such $x$, ${\eta}^{(m_j)}(x) \in Q(x)$ and $Q(x)$ is closed, whence ${\eta}(x) \in Q(x)$.
Therefore ${\eta} \in \mathcal{S}$.

\emph{\textbf{Step 4: $\mathcal{S}$ is weakly closed in $L^1$.}} $\mathcal{S}$ is convex (Step 2) and norm-closed (Step 3). By Mazur's theorem, a convex norm-closed subset of a Banach space is weakly closed. 
Here the convexity established in Step 2 permits the passage from norm closedness to weak closedness.

\emph{\textbf{Step 5: weak closedness of $\mathcal{M}(P_X, Q)$.}} Let $P^{(n)} \in \mathcal{M}(P_X, Q)$ with $P^{(n)} \rightharpoonup P$ weakly in $\mathcal{P}(\mathcal{X} \times \mathcal{Y})$.
Since $P \mapsto (\pi_{\mathcal{X}})_{\#} P$ is weakly continuous, $(\pi_{\mathcal{X}})_{\#} P = P_X$, so ${\eta} := \Phi(P)$ and ${\eta}^{(n)} := \Phi(P^{(n)})$ are well defined.
By \cref{eq:integral-rep}, for every $g \in C_b(\mathcal{X})$ and each $k$,
\begin{equation}
\int_{\mathcal{X}} g\, \eta^{(n)}_k \, dP_X
\;\longrightarrow\;
\int_{\mathcal{X}} g\, \eta_k \, dP_X.
\label{eq: weak-closedness}
\end{equation}
The family $({\eta}^{(n)})$ takes values in $\Delta^{K-1}$, hence is bounded in $L^\infty$ and uniformly integrable; by the Dunford--Pettis theorem it is relatively weakly compact in $L^1$.
Let ${\zeta}$ be any weak-$L^1$ cluster point.
Then~\cref{eq: weak-closedness} gives 
$$\int g(\zeta_k - \eta_k)\, dP_X = 0$$
for all $g \in C_b(\mathcal{X})$. As $C_b(\mathcal{X})$ is $L^1(P_X)$-determining---for any $A \in \mathcal{B}(\mathcal{X})$, Lusin's theorem provides $g_j \in C_b(\mathcal{X})$ with $\|g_j\|_\infty \le 1$ and $g_j \to \mathbf{1}_A$ in $L^1(P_X)$, whence $\int_{A}(\zeta_k-\eta_k)dP_X=0$---we conclude ${\zeta} = {\eta}$ a.e. Thus every cluster point equals ${\eta}$, so ${\eta}^{(n)} \rightharpoonup {\eta}$ weakly in $L^1$.
Since ${\eta}^{(n)} \in \mathcal{S}$ and $\mathcal{S}$ is weakly closed (Step~4), ${\eta} \in \mathcal{S}$, i.e.\ $P \in \mathcal{M}(P_X, Q)$.
The metrisability of the weak topology in $\mathcal{P}(\mathcal{X} \times \mathcal{Y})$ makes sequential closedness sufficient.

\emph{\textbf{Step 6: tightness and weak compactness.}} Fix $\varepsilon > 0$. By Ulam's theorem ($P_X$ is Radon on a Polish space), choose a compact $K_\varepsilon \subseteq \mathcal{X}$ with $P_X(K_\varepsilon) \ge 1 - \varepsilon$. As $\mathcal{Y}$ is finite, $K_\varepsilon \times \mathcal{Y}$ is compact, and for every $P \in \mathcal{M}(P_X, Q)$,
\[
P(K_\varepsilon \times \mathcal{Y}) = P_X(K_\varepsilon) \ge 1 - \varepsilon.
\]
Hence $\mathcal{M}(P_X, Q)$ is \emph{uniformly tight}, so relatively weakly compact by Prokhorov's theorem.
Combined with weak closedness (Step~5), $\mathcal{M}(P_X, Q)$ is weakly compact.

\emph{\textbf{Step 7: the credal risk profile.}} To identify the full range of admissible risks, we first construct joint laws attaining its endpoints. Write $c_k(x)=\ell(f(x),k)$ and $a(x)=\max_k|c_k(x)|$. The integrability assumption makes $a$ finite $P_X$-a.e.; if necessary, set all $c_k$ to zero on the common exceptional null set, which does not affect any risk because every admissible joint law has marginal $P_X$. The function $L(x,q)=\sum_{k=1}^K q_kc_k(x)$ is measurable in $x$, continuous in $q$, and satisfies $|L(x,q)|\le a(x)$ for every $q\in\Delta^{K-1}$, up to the same null-set modification.

Since $Q$ is measurable with nonempty compact values, the measurable maximum theorem gives measurable minimum and maximum values and nonempty compact sets of minimisers and maximisers. Measurable selection therefore yields $q_+$ and $q_-$ such that
\[
q_+(x)\in\operatorname*{arg\,min}_{q\in Q(x)}L(x,q),
\qquad
q_-(x)\in\operatorname*{arg\,max}_{q\in Q(x)}L(x,q).
\]
Their induced joint laws, defined by $P_\pm(A\times\{k\})=\int_A q_{\pm,k}(x)\,dP_X(x)$, belong to $\mathcal{M}(P_X,Q)$. The integrable envelope $a$ ensures that all admissible risks are finite and that the integral representation in Lemma~\ref{lem:selection} also applies to these losses. Consequently,
\[
\mathbb{E}_{P_+}[\ell(f(X),Y)]=R_+(f),
\qquad
\mathbb{E}_{P_-}[\ell(f(X),Y)]=R_-(f),
\]
with $R_+$ and $R_-$ as defined in~\cref{eq: instance_risk}. Every $P\in\mathcal{M}(P_X,Q)$ has a conditional selection $\eta(x)\in Q(x)$ $P_X$-a.e., so integrating the pointwise inequalities $\min_{q\in Q(x)}L(x,q)\le L(x,\eta(x))\le\max_{q\in Q(x)}L(x,q)$ gives
\[
R_+(f)\le\mathbb{E}_{P}[\ell(f(X),Y)]\le R_-(f).
\]
Conversely, convexity of $\mathcal{M}(P_X,Q)$ ensures that $(1-t)P_++tP_-$ is admissible for each $t\in[0,1]$, with risk $(1-t)R_+(f)+tR_-(f)$. Thus every value between the attained endpoints is admissible, proving the compact-interval identity in~\cref{eq:risk_profile}.
\end{proof}
\begin{remark}
The compactness argument uses Polishness through metrisability of the weak topology, the determining properties of $C_b(\mathcal{X})$, and tightness of $P_X$. Finiteness of $\mathcal{Y}$ gives the selection representation, while convexity of the values yields both convexity and weak closedness of the selection set. For a bounded continuous loss on $\mathcal{X}\times\mathcal{Y}$, weak compactness alone ensures that the extreme risks are attained. Step 7 establishes attainment more generally through measurable selection and an integrable envelope; it requires no continuity of the loss in $x$.
\end{remark}

\subsection{Proof of Proposition~\ref{prop:closed_solution}}
\label{App:ProofProposition2}
\begin{proof}
The mixture representation reduces each inner optimisation to allocating the unrestricted probability mass. Fix an input \(x\), write \(j=j(x)\) and \(\alpha=\alpha(x)\), and consider a prediction \(p\) with strictly positive coordinates. Set \(c_k=-\log p_k\). Since \(\mathcal Q_\alpha(x)=\alpha e_j+(1-\alpha)\Delta^{K-1}\), every admissible target has the form \(q=\alpha e_j+(1-\alpha)r\) for some \(r\in\Delta^{K-1}\). This set of mixtures is an \(\epsilon\)-contamination class with \(\epsilon=1-\alpha\). Its cross-entropy loss is therefore
\[
\ell(p,q)=\alpha c_j+(1-\alpha)\sum_{k=1}^{K}r_kc_k.
\]
The sum is a weighted average of the per-class losses, so it lies between \(\min_k c_k\) and \(\max_k c_k\). Both bounds are attained by placing all mass of \(r\) on a class attaining the corresponding extremum. As \(-\log\) is decreasing, the largest loss corresponds to a least probable class, and the smallest loss to a most probable class.

Applying this argument to the two heads, choose \(k_{\min}\in\arg\min_k p_{-,k}\) and \(k_{\max}\in\arg\max_k p_{+,k}\). An inner maximiser for the pessimistic head and an inner minimiser for the optimistic head are, respectively,
\[
q_-^*=\alpha e_j+(1-\alpha)e_{k_{\min}},
\qquad
q_+^*=\alpha e_j+(1-\alpha)e_{k_{\max}}.
\]
Any tied extremal class gives a valid solution; uniqueness is not required. If a selected class is \(j\), its target reduces to \(e_j\). When \(\alpha=1\), both expressions reduce to the only admissible target \(e_j\).
\end{proof}

\subsection{Proof of Theorem~\ref{theorem:main}}
\label{App:ProofTheorem}
\subsubsection{Setup and Standing Assumptions}
The proof connects the joint training objective to the true risk of the averaged predictor. Throughout, \(Q\) is a fixed measurable credal kernel satisfying the regularity conditions of Section~\ref{sec: Method}, and we retain the head convention of Section~\ref{sec:EfficientModel}: the pessimistic head \(f_{-,\theta}\) minimises an upper risk, while the optimistic head \(f_{+,\theta}\) minimises a lower risk. We use the normalised population and empirical objectives corresponding to~\cref{eq:total_loss}:
\begin{equation}
J_Q(\theta) := \frac{1}{2}\left\{R_-(f_{-,\theta}) + R_+(f_{+,\theta})\right\}, \quad
\widehat J_{Q,n}(\theta) := \frac{1}{2}\left\{\widehat R_{-,n}(f_{-,\theta}) + \widehat R_{+,n}(f_{+,\theta})\right\}.
\label{eq: joint-objectives}
\end{equation}
For \(q\in\Delta^{K-1}\) and a probabilistic prediction \(f(x)\in\Delta^{K-1}\), the cross-entropy loss is
\begin{equation}
H(q,f(x)) := -\sum\nolimits_{k=1}^K q_k \log f_k(x).
\label{eq:cross-entropy}
\end{equation}
The true risk of any ordinary first-order predictor \(f\) under cross-entropy is then
\begin{equation}
R_0(f) := \mathbb E_X H(p_0(X), f(X)) = \mathbb E_{(X,Y)\sim P_0}[-\log f_Y(X)].
\label{eq:true-risk}
\end{equation}
For clarity, define the pointwise upper and lower losses by
\begin{equation}
\ell_Q^-(f,x):=\max_{q\in Q(x)}H(q,f(x)),\qquad
\ell_Q^+(f,x):=\min_{q\in Q(x)}H(q,f(x)).
\label{eq:pointwise-credal-losses}
\end{equation}
Their population risks are \(R_-(f)=\mathbb E_X[\ell_Q^-(f,X)]\) and \(R_+(f)=\mathbb E_X[\ell_Q^+(f,X)]\), as in Section~\ref{sec: Method}; their empirical counterparts are \(\widehat R_{\pm,n}(f)=n^{-1}\sum_{i=1}^n\ell_Q^\pm(f,X_i)\).

We adopt the following assumptions throughout the proof.
\begin{enumerate}[label=\textbf{Assumption \arabic*.}, leftmargin=*]
  \item (Conditional coverage) For \(P_X\)-almost every \(x\),
        \begin{equation}
        p_0(x)\in Q(x).
        \label{eq:coverage}
        \end{equation}
  \item (Sampling) The inputs \(X_1,\dots,X_n\overset{\mathrm{iid}}{\sim}P_X\), and the corresponding credal labels \(Q(X_i)\) are observed. The credal kernel \(Q\) and the model class are fixed independently of this sample.
  \item (Uniformly bounded cross-entropy) There exist \(B<\infty\) and a measurable set \(\mathcal X_0\subseteq\mathcal X\) with \(P_X(\mathcal X_0)=1\) such that, simultaneously for every \(x\in\mathcal X_0\), \(\theta\in\boldsymbol\Theta\), and \(q\in Q(x)\),
        \begin{equation}
        0\le H(q,f_{-,\theta}(x))\le B,\qquad 0\le H(q,f_{+,\theta}(x))\le B.
        \label{eq:bounded-loss}
        \end{equation}
        The common set \(\mathcal X_0\) ensures that this bound also applies to parameters selected from the observed sample.
  \item (Bounded variation in classwise losses) The difference between the largest and smallest classwise log losses, called their oscillation, is uniformly bounded: there exists \(L<\infty\) such that every prediction used by the pair class satisfies
        \begin{equation}
        \operatorname{osc}(-\log f(x)):=\max_k[-\log f_k(x)]-\min_k[-\log f_k(x)]\le L
        \label{eq:oscillation-bound}
        \end{equation}
        for every \(x\in\mathcal X_0\), simultaneously for all parameters and both heads.
  \item (Approximate joint ERM) The learned parameter \(\widehat\theta_n\) satisfies
        \begin{equation}
        \widehat J_{Q,n}(\widehat\theta_n)\le \inf_{\theta\in\boldsymbol\Theta}\widehat J_{Q,n}(\theta)+\eta_n,
        \label{eq:approx-joint-erm}
        \end{equation}
        where \(\eta_n\ge0\) is the optimisation error in the normalised objective. If the unnormalised objective is optimised to tolerance \(\widetilde\eta_n\), then \(\eta_n=\widetilde\eta_n/2\).
  \item (Nonempty comparator class) At least one ordinary predictor can be represented by making the two heads identical. Thus, the corresponding comparator class is nonempty:
        \begin{equation}
        \mathcal F_{\Delta}:=\big\{f:\exists\theta_f\in\boldsymbol\Theta \text{ such that }f_{-,\theta_f}=f_{+,\theta_f}=f\big\}\ne\varnothing.
        \label{eq:diagonal-class}
        \end{equation}
\end{enumerate}

Let \(\widehat{\theta}_n\) be an \(\eta_n\)-approximate empirical minimiser of~\cref{eq: joint-objectives}, and let
\[
\widehat{\bar f}_n := \bar f_{\widehat{\theta}_n}
= \frac{f_{-,\widehat{\theta}_n} + f_{+,\widehat{\theta}_n}}{2}
\]
denote the corresponding averaged predictor defined in Section~\ref{sec:EfficientModel}. The learned pair need not lie on the diagonal \(f_+=f_-\), so we compare its true risk with that of the best ordinary predictor in \(\mathcal F_{\Delta}\).

Under assumptions in~\cref{eq:coverage,eq:bounded-loss,eq:oscillation-bound,eq:approx-joint-erm,eq:diagonal-class}, we prove the generic generalisation bound of POCC stated in Theorem~\ref{theorem:main}:
\[
\boxed{
R_0(\widehat{\bar f}_n) \!-\! \inf_{f\in\mathcal F_{\Delta}}R_0(f)
\!\leq\!\!\!\!\!
\underbrace{4\mathfrak R_n(\mathcal J_Q)}_{\text{model class complexity}}
\!+\!\!\!
\underbrace{2B\sqrt{\frac{\log(2/\delta)}{2n}}}_{\text{finite-sample estimation error}}
\!+\!\!\!
\underbrace{\eta_n}_{\text{optimisation gap}}
\!+\!\!\!
\underbrace{L\overline d_Q}_{\text{annotation imprecision}}.
}
\]
\subsubsection{Elementary Ingredients for the Proof}
We first establish several elementary results used in the proof of Theorem~\ref{theorem:main}. Coverage makes the true conditional distribution feasible in both pointwise optimisations, so \(R_+(f)\le R_0(f)\le R_-(f)\). The following argument controls how far either endpoint can lie from the true risk.
\paragraph{Credal Diameter Controls Ambiguity Gaps.}
Define the pointwise total-variation (TV) diameter
\begin{equation}
d_Q(x):=\sup_{q,r\in Q(x)}\operatorname{TV}(q,r), \qquad \overline d_Q:=\mathbb E_X[d_Q(X)],
\label{eq:credal-diameter}
\end{equation}
where, on the finite label space,
\[\operatorname{TV}(q,r)=\frac12\|q-r\|_1.\]

\begin{lemma}[Total Variation--Oscillation Inequality]
For any $q,r\in\Delta^{K-1}$ and $a\in\mathbb R^K$,
\begin{equation}
|\inner{q-r}{a}|
\le
\operatorname{TV}(q,r)\,\operatorname{osc}(a),
\qquad
\operatorname{osc}(a)=\max_k a_k-\min_k a_k.
\label{eq:tv-oscillation}
\end{equation}
\end{lemma}

\begin{proof}
Let
\[
m:=\frac{\max_k a_k+\min_k a_k}{2}.
\]
Because $q$ and $r$ both sum to one,
\[
\inner{q-r}{\bm 1}=0.
\]
Consequently,
\[
\inner{q-r}{a}
=\inner{q-r}{a-m\bm 1}.
\]
By H\"older's inequality,
\[
|\inner{q-r}{a-m\bm 1}|
\le
\|q-r\|_1\,\|a-m\bm 1\|_\infty.
\]
The two factors satisfy
\[
\|q-r\|_1=2\operatorname{TV}(q,r),
\qquad
\|a-m\bm 1\|_\infty=\frac{\operatorname{osc}(a)}2.
\]
Multiplying them gives~\cref{eq:tv-oscillation}.
\end{proof}
The upper and lower ambiguity gaps measure the two endpoint deviations from the true risk:
\begin{equation}
A_Q^-(f):=R_-(f)-R_0(f)\ge0, \quad A_Q^+(f):=R_0(f)-R_+(f)\ge0.
\label{eq:upper-lower-gaps}
\end{equation}
\begin{proposition}[Generic Ambiguity-gap Bounds]
Under conditional coverage, for every $f$,
\begin{align}
0\le A_Q^-(f)
&\le
\mathbb E_X\!\left[d_Q(X)\,\operatorname{osc}(-\log f(X))\right],
\label{eq:upper-gap-diameter}\\
0\le A_Q^+(f)
&\le
\mathbb E_X\!\left[d_Q(X)\,\operatorname{osc}(-\log f(X))\right].
\label{eq:lower-gap-diameter}
\end{align}
Under the uniform oscillation bound~\cref{eq:oscillation-bound}, both gaps are at most $L\overline d_Q$.
\end{proposition}

\begin{proof}
Set
\[
a_f(x):=-\log f(x)
=\bigl(-\log f_1(x),\dots,-\log f_K(x)\bigr).
\]
For the upper gap, use coverage and write
\begin{align*}
\ell_Q^-(f,x)-H(p_0(x),f(x))=\sup_{q\in Q(x)}q^\top a_f(x)-p_0(x)^\top a_f(x)=\sup_{q\in Q(x)}(q-p_0(x))^\top a_f(x).
\end{align*}
For each feasible $q$, both $q$ and $p_0(x)$ belong to $Q(x)$, so
\[
\operatorname{TV}(q,p_0(x))\le d_Q(x).
\]
Applying \cref{eq:tv-oscillation},
\[
(q-p_0(x))^\top a_f(x)
\le d_Q(x)\operatorname{osc}(a_f(x)).
\]
The right-hand side no longer depends on $q$, so the same bound holds after taking the supremum. Integrating proves~\cref{eq:upper-gap-diameter}.

For the lower gap,
\begin{align*}
H(p_0(x),f(x))-\ell_Q^+(f,x)&=p_0(x)^\top a_f(x)-\inf_{q\in Q(x)}q^\top a_f(x)\\
&=\sup_{q\in Q(x)}(p_0(x)-q)^\top a_f(x),
\end{align*}
and the identical argument proves~\cref{eq:lower-gap-diameter}. The final assertion follows directly from $\operatorname{osc}(-\log f(X))\le L$ in~\cref{eq:oscillation-bound}.
\end{proof}
\paragraph{Population Risk of an Arbitrary Averaged Dual-head Predictor.}
Next, we establish the central bridge between the algorithm's joint credal objective and the true risk of its averaged prediction.
\begin{proposition}[Population Risk Bound for the Average]
For every $\theta\in\boldsymbol\Theta$, under conditional coverage,
\begin{equation}
R_0(\bar f_\theta)\le J_Q(\theta) +\frac12\left\{A_Q^+(f_{+,\theta})-A_Q^-(f_{-,\theta})\right\}.
\label{eq:exact-average-decomposition}
\end{equation}
Consequently,
\begin{align}
R_0(\bar f_\theta)
&\le J_Q(\theta)+\frac12A_Q^+(f_{+,\theta}),
\label{eq:average-lower-gap-only}\\
&\le J_Q(\theta)
+\frac12\mathbb E_X\!\left[d_Q(X)\operatorname{osc}(-\log f_{+,\theta}(X))\right],
\label{eq:average-diameter-local}\\
&\le J_Q(\theta)+\frac L2\overline d_Q.
\label{eq:average-diameter-uniform}
\end{align}
\end{proposition}
\begin{proof}
Start with convexity of the true risk:
\begin{equation}
R_0(\bar f_\theta)
\le
\frac12R_0(f_{-,\theta})
+\frac12R_0(f_{+,\theta}).
\label{eq:step-convexity}
\end{equation}
Use the definitions of the ambiguity gaps in~\cref{eq:upper-lower-gaps}:
\begin{align*}
R_0(f_{-,\theta}) =R_-(f_{-,\theta})-A_Q^-(f_{-,\theta}),\quad
R_0(f_{+,\theta}) =R_+(f_{+,\theta})+A_Q^+(f_{+,\theta}).
\end{align*}
Substitute both identities into~\cref{eq:step-convexity}:
\begin{align*}
R_0(\bar f_\theta)
&\le
\frac12\bigl[R_-(f_{-,\theta})-A_Q^-(f_{-,\theta})\bigr]
+\frac12\bigl[R_+(f_{+,\theta})+A_Q^+(f_{+,\theta})\bigr]\\
&=
\frac12\bigl[R_-(f_{-,\theta})+R_+(f_{+,\theta})\bigr]
+\frac12\bigl[A_Q^+(f_{+,\theta})-A_Q^-(f_{-,\theta})\bigr]\\
&=J_Q(\theta)
+\frac12\bigl[A_Q^+(f_{+,\theta})-A_Q^-(f_{-,\theta})\bigr].
\end{align*}
This proves~\cref{eq:exact-average-decomposition}. Because $A_Q^-(f_{-,\theta})\ge0$, dropping the negative term proves~\cref{eq:average-lower-gap-only}. Applying~\cref{eq:lower-gap-diameter} to the optimistic head gives~\cref{eq:average-diameter-local}. Finally, the uniform oscillation bound gives~\cref{eq:average-diameter-uniform}.
\end{proof}
\paragraph{Uniform Convergence for the Joint Infinite Hypothesis Class.}

Define
\begin{equation}
j_\theta(x)
:=\frac12\left\{
\ell_Q^-(f_{-,\theta},x)
+\ell_Q^+(f_{+,\theta},x)
\right\},
\label{eq:joint-pointwise-loss}
\end{equation}
and the induced real-valued class
\begin{equation}
\mathcal J_Q
:=\{j_\theta:\theta\in\boldsymbol\Theta\}.
\label{eq:joint-loss-class}
\end{equation}
Then
\[
J_Q(\theta)=\mathbb E_X[j_\theta(X)],
\qquad
\widehat J_{Q,n}(\theta)=\frac1n\sum\nolimits_{i=1}^n j_\theta(X_i).
\]
By the common bound in~\cref{eq:bounded-loss}, every $j_\theta$ takes values in $[0,B]$ on the same set $\mathcal X_0$ simultaneously for every $\theta\in\Theta$.

For a real-valued class $\mathcal G$, define the expected absolute Rademacher complexity
\begin{equation}
\mathfrak R_n(\mathcal G)
:=\mathbb E_{X_{1:n},\sigma_{1:n}}
\left[
\sup_{g\in\mathcal G}
\left|
\frac1n\sum_{i=1}^n\sigma_i g(X_i)
\right|
\right],
\label{eq:rademacher-definition}
\end{equation}
where $\sigma_1,\dots,\sigma_n$ are independent Rademacher signs.

\begin{theorem}[Uniform Convergence of the Joint Objective]
Under i.i.d. sampling and the uniform boundedness assumption~\cref{eq:bounded-loss}, with probability at least $1-\delta$,
\begin{equation}
\sup_{\theta\in\boldsymbol\Theta}
\left|J_Q(\theta)-\widehat J_{Q,n}(\theta)\right|
\le
\varepsilon_{n,\mathrm{joint}}(\delta),
\label{eq:joint-uniform-convergence}
\end{equation}
where
\begin{equation}
\varepsilon_{n,\mathrm{joint}}(\delta)
:=2\mathfrak R_n(\mathcal J_Q)
+B\sqrt{\frac{\log(2/\delta)}{2n}}.
\label{eq:joint-epsilon}
\end{equation}
No finiteness or countability of $\boldsymbol\Theta$ is required, provided the relevant suprema are measurable and the displayed Rademacher complexity is well defined.
\end{theorem}

\begin{proof}
Write $Pj:=\mathbb E[j(X)]$ and $P_nj:=n^{-1}\sum_{i=1}^n j(X_i)$. Define the random uniform deviation
\[
\Phi(X_{1:n})
:=\sup_{j\in\mathcal J_Q}|Pj-P_nj|.
\]
We prove the result in two stages.

\emph{\textbf{Stage 1: symmetrisation.}}
Let $X_1',\dots,X_n'$ be an independent ghost sample from $P_X$, and let
\[
P_n'j:=\frac1n\sum_{i=1}^n j(X_i').
\]
Since $Pj=\mathbb E_{X_{1:n}'}[P_n'j]$, Jensen's inequality gives
\begin{align*}
\mathbb E_{X_{1:n}}\Phi
&=
\mathbb E_{X_{1:n}}
\sup_{j\in\mathcal J_Q}
\left|
\mathbb E_{X_{1:n}'}[P_n'j]-P_nj
\right|\\
&\le
\mathbb E_{X_{1:n},X_{1:n}'}
\sup_{j\in\mathcal J_Q}
|P_n'j-P_nj|\\
&=
\mathbb E_{X,X'}
\sup_{j\in\mathcal J_Q}
\left|
\frac1n\sum_{i=1}^n\bigl(j(X_i')-j(X_i)\bigr)
\right|.
\end{align*}
For each $i$, the pair $(X_i,X_i')$ is exchangeable. Introducing independent Rademacher signs therefore leaves the joint distribution of the difference unchanged:
\begin{align*}
\mathbb E\Phi
&\le
\mathbb E_{X,X',\sigma}
\sup_{j\in\mathcal J_Q}
\left|
\frac1n\sum_{i=1}^n
\sigma_i\bigl(j(X_i')-j(X_i)\bigr)
\right|.
\end{align*}
Use the triangle inequality inside the supremum:
\begin{align*}
\sup_{j\in\mathcal J_Q}
\left|
\frac1n\sum_{i=1}^n
\sigma_i\bigl(j(X_i')-j(X_i)\bigr)
\right|
\quad\le
\sup_{j\in\mathcal J_Q}
\left|
\frac1n\sum_{i=1}^n\sigma_i j(X_i')
\right|
+
\sup_{j\in\mathcal J_Q}
\left|
\frac1n\sum_{i=1}^n\sigma_i j(X_i)
\right|.
\end{align*}
Both terms have the same expectation. Hence
\begin{equation}
\mathbb E\Phi\le2\mathfrak R_n(\mathcal J_Q).
\label{eq:expected-uniform-deviation}
\end{equation}

\emph{\textbf{Stage 2: concentration around the expectation.}}
% Since $P_X(\mathcal X_0)=1$, we may regard the inputs as taking values in $\mathcal X_0$. 
Since $P_X(\mathcal X_0)=1$, we may regard the inputs as taking values in $\mathcal X_0$. By assumption, the bound $B$ holds simultaneously for every 
$\theta\in\Theta$, both heads, and every admissible target $q\in\mathcal Q(x)$ on this common set, i.e.,
\[
\sup_{\theta\in\Theta}\,\sup_{q\in\mathcal Q(x)}
\ell\bigl(p_{\pm,\theta}(x),q\bigr)\le B
\qquad\text{for all }x\in\mathcal X_0.
\]
% On $\mathcal X_0^n$, replace one observation $X_i$ by any $\widetilde X_i\in\mathcal X_0$. Every $j\in\mathcal J_Q$ is bounded by the same interval $[0,B]$ on this set, so
% \[
% |P_nj-P_n^{(i)}j|
% \le\frac{B}{n}.
% \]
Consequently, every $j\in\mathcal J_Q$ is bounded in $[0,B]$ on $\mathcal X_0^n$ simultaneously across all parameters, which is the bounded-differences hypothesis required by McDiarmid's inequality. On $\mathcal X_0^n$, replacing one observation $X_i$ by any $\widetilde X_i\in\mathcal X_0$ therefore changes $P_n j$ by at most $B/n$:
\[
|P_nj-P_n^{(i)}j|
\le\frac{B}{n}.
\]
Taking a supremum cannot increase the replacement sensitivity beyond $B/n$, so
\[
|\Phi(X_{1:n})-\Phi(X_1,\dots,\widetilde X_i,\dots,X_n)|
\le\frac{B}{n}.
\]
McDiarmid's inequality therefore yields
\[
\mathbb P\{\Phi-\mathbb E\Phi\ge t\}
\le\exp\!\left(-\frac{2nt^2}{B^2}\right).
\]
Choose
\[
t=B\sqrt{\frac{\log(2/\delta)}{2n}}.
\]
Then, with probability at least $1-\delta$,
\[
\Phi
\le\mathbb E\Phi+B\sqrt{\frac{\log(2/\delta)}{2n}}.
\]
Combining this with~\cref{eq:expected-uniform-deviation} proves~\cref{eq:joint-uniform-convergence}--~\cref{eq:joint-epsilon}.
\end{proof}
\subsubsection{Proof of the Generic Generalisation Bound}
\paragraph{Step 1.}
We first combine the population relation in~\cref{eq:average-diameter-uniform} with uniform convergence and approximate empirical risk minimisation.
\begin{theorem}[Risk Bound Relative to the Best Joint Credal Objective]
\label{theorem: risk_bound}
Under the assumptions in~\cref{eq:coverage,eq:bounded-loss,eq:oscillation-bound,eq:approx-joint-erm}, with probability at least $1-\delta$,
\begin{equation}
R_0(\widehat{\bar f}_n)
\le
\inf_{\theta\in\boldsymbol\Theta}J_Q(\theta)
+2\varepsilon_{n,\mathrm{joint}}(\delta)
+\eta_n
+\frac L2\overline d_Q.
\label{eq:joint-objective-oracle}
\end{equation}
\end{theorem}

\begin{proof}
Work on the event~\cref{eq:joint-uniform-convergence}. We proceed one line at a time.

First, apply~\cref{eq:average-diameter-uniform} to the learned pair:
\begin{equation}
R_0(\widehat{\bar f}_n)
\le
J_Q(\widehat\theta_n)+\frac L2\overline d_Q.
\label{eq:erm-proof-1}
\end{equation}
Uniform convergence at $\widehat\theta_n$ gives
\begin{equation}
J_Q(\widehat\theta_n)
\le
\widehat J_{Q,n}(\widehat\theta_n)
+\varepsilon_{n,\mathrm{joint}}.
\label{eq:erm-proof-2}
\end{equation}
Approximate empirical minimisation gives
\begin{equation}
\widehat J_{Q,n}(\widehat\theta_n)
\le
\inf_{\theta\in\boldsymbol\Theta}\widehat J_{Q,n}(\theta)+\eta_n.
\label{eq:erm-proof-3}
\end{equation}
Uniform convergence also gives, simultaneously for every comparator $\theta$,
\[
\widehat J_{Q,n}(\theta)
\le J_Q(\theta)+\varepsilon_{n,\mathrm{joint}}.
\]
Taking the infimum over $\theta$,
\begin{equation}
\inf_{\theta\in\boldsymbol\Theta}\widehat J_{Q,n}(\theta)
\le
\inf_{\theta\in\boldsymbol\Theta}J_Q(\theta)
+\varepsilon_{n,\mathrm{joint}}.
\label{eq:erm-proof-4}
\end{equation}
Combining~\cref{eq:erm-proof-1}--~\cref{eq:erm-proof-4},
\begin{align*}
R_0(\widehat{\bar f}_n)
&\le
J_Q(\widehat\theta_n)+\frac L2\overline d_Q\\
&\le
\widehat J_{Q,n}(\widehat\theta_n)
+\varepsilon_{n,\mathrm{joint}}
+\frac L2\overline d_Q\\
&\le
\inf_{\theta}\widehat J_{Q,n}(\theta)
+\eta_n
+\varepsilon_{n,\mathrm{joint}}
+\frac L2\overline d_Q\\
&\le
\inf_{\theta}J_Q(\theta)
+2\varepsilon_{n,\mathrm{joint}}
+\eta_n
+\frac L2\overline d_Q.
\end{align*}
This is~\cref{eq:joint-objective-oracle}.
\end{proof}
\paragraph{Step 2.}
Theorem~\ref{theorem: risk_bound} compares against the best two-head credal objective. We now convert it into an excess true-risk bound.
\begin{lemma}[A Diagonal Pair is Close to Its True Risk]
For every $f\in\mathcal F_{\Delta}$,
\begin{equation}
\frac12\bigl[R_-(f)+R_+(f)\bigr]
\le
R_0(f)+\frac L2\overline d_Q.
\label{eq:diagonal-comparator-bound}
\end{equation}
\end{lemma}
\begin{proof}
Using the ambiguity gaps,
\begin{align*}
\frac12\bigl[R_-(f)+R_+(f)\bigr]-R_0(f)
&=\frac12\bigl[A_Q^-(f)-A_Q^+(f)\bigr]\\
&\le\frac12A_Q^-(f)\\
&\le\frac12\mathbb E_X[d_Q(X)\operatorname{osc}(-\log f(X))]\\
&\le\frac L2\overline d_Q.
\end{align*}
This proves~\cref{eq:diagonal-comparator-bound}.
\end{proof}
\paragraph{Step 3.} We are now ready to complete the proof of Theorem~\ref{theorem:main}:
\begin{equation}
\boxed{
R_0(\widehat{\bar f}_n) - \inf_{f\in\mathcal F_{\Delta}}R_0(f)
\leq
4\mathfrak R_n(\mathcal J_Q)
+ 2B\sqrt{\frac{\log(2/\delta)}{2n}}
+ \eta_n
+ L\overline d_Q
}
\label{eq:main-generic-averaged-bound}
\end{equation}
\begin{proof}
For every $f\in\mathcal F_{\Delta}$, by definition there exists a parameter $\theta_f$ with
\[
f_{-,\theta_f}=f_{+,\theta_f}=f.
\]
Therefore
\[
J_Q(\theta_f)
=\frac12\bigl[R_-(f)+R_+(f)\bigr].
\]
By~\cref{eq:diagonal-comparator-bound},
\[
J_Q(\theta_f)
\le R_0(f)+\frac L2\overline d_Q.
\]
Taking the infimum over $f\in\mathcal F_{\Delta}$ and noting that the infimum over all $\theta$ is no larger than the infimum over diagonal parameters,
\begin{equation}
\inf_{\theta\in\boldsymbol\Theta}J_Q(\theta)
\le
\inf_{f\in\mathcal F_{\Delta}}R_0(f)
+\frac L2\overline d_Q.
\label{eq:joint-to-true-oracle}
\end{equation}
Substitute~\cref{eq:joint-to-true-oracle} into~\cref{eq:joint-objective-oracle}:
\begin{align*}
R_0(\widehat{\bar f}_n)
&\le
\inf_{f\in\mathcal F_{\Delta}}R_0(f)
+\frac L2\overline d_Q
+2\varepsilon_{n,\mathrm{joint}}
+\eta_n
+\frac L2\overline d_Q\\
&=
\inf_{f\in\mathcal F_{\Delta}}R_0(f)
+2\varepsilon_{n,\mathrm{joint}}
+\eta_n
+L\overline d_Q.
\end{align*}
Subtract $\inf_{f\in\mathcal F_{\Delta}}R_0(f)$ and expand
\[
2\varepsilon_{n,\mathrm{joint}}(
\delta)
=4\mathfrak R_n(\mathcal J_Q)
+2B\sqrt{\frac{\log(2/\delta)}{2n}}.
\]
This gives~\cref{eq:main-generic-averaged-bound}.
\end{proof}
\subsubsection{Specialisation to the Practical Credal Label Construction}
\label{App:specialization}
For the credal construction in~\cref{eq:LabelConstruction}, i.e.,
\begin{equation}
\mathcal{Q}_{\alpha}(x)\!=\!\big\{q\!\in\!\Delta^{K-1}\!:q_{j(x)} \ge \alpha(x)\big\},
\nonumber
\end{equation}
the conditional coverage becomes the transparent condition
\begin{equation}
p_{0,j(x)}(x)\ge\alpha(x)
\qquad P_X\text{-a.s.}
\label{eq:special-coverage}
\end{equation}
Suppose each head is a softmax of logits $s_{\pm,\theta}(x)$ satisfying, simultaneously for every $\theta\in\boldsymbol\Theta$, both heads, and every $x$ in a common measurable set $\mathcal X_0$ with $P_X(\mathcal X_0)=1$,
\begin{equation}
\|s_{\pm,\theta}(x)\|_\infty\le A.
\label{eq:bounded-logits}
\end{equation}
Then
\begin{equation}
f_{\pm,\theta,k}(x)
\ge\frac{e^{-2A}}{K},
\qquad
B=\log K+2A.
\label{eq:B-bounded-logits}
\end{equation}
Moreover,
\begin{equation}
\operatorname{osc}(-\log f_{\pm,\theta}(x))
=\max_k s_{\pm,\theta,k}(x)-\min_k s_{\pm,\theta,k}(x)
\le2A,
\label{eq:L-bounded-logits}
\end{equation}
so $L=2A$.

\begin{proof}[Verification]
For the probability lower bound,
\[
f_k(x)
=\frac{e^{s_k(x)}}{\sum_{\ell=1}^K e^{s_\ell(x)}}
\ge\frac{e^{-A}}{Ke^A}
=\frac{e^{-2A}}K.
\]
Thus $-\log f_k(x)\le\log K+2A$, and any cross-entropy, being a convex combination of these coordinate losses, is bounded by the same $B$.

Also,
\[
-\log f_k(x)
=\log\!\left(\sum\nolimits_{\ell=1}^K e^{s_\ell(x)}\right)-s_k(x).
\]
The common log-sum-exp term cancels when taking the oscillation, leaving $\max_k s_k-\min_k s_k\le2A$.
\end{proof}
Applying the generic bound~\cref{eq:main-generic-averaged-bound} with
$B=\log K+2A$, $L=2A$, and
$$
\overline d_Q
=
\mathbb E_X[d_{Q_{\alpha,j}}(X)]
=
\mathbb E_X[1-\alpha(X)],
$$

we obtain the specialised bound:
$$
\boxed{
R_0(\widehat{\bar f}_n)-\inf_{f\in\mathcal F_{\Delta}}R_0(f)
\leq
4\mathfrak R_n(\mathcal J_{Q_{\alpha,j}})
+2(\log K+2A)\sqrt{\frac{\log(2/\delta)}{2n}}
+\eta_n
+2A\,\mathbb E[1-\alpha(X)].
}
$$

\section{Experiment Implementation Details}
\label{App:ImplementationDetails}
This section details the experimental setup for our experiments. The datasets are drawn from existing literature. \emph{To further support reproducibility, the core implementation codes are included in the supplementary materials and will be released publicly upon acceptance under a licence permitting free use for research purposes.}
\subsection{Datasets}
\label{sub-App:datasets}
\textbf{CIFAR-10.} The CIFAR-10 dataset~\citep{krizhevsky2009learning} consists of 60,000 colour images (32×32) across 10 classes, split into 50,000 training and 10,000 test images.

\textbf{CIFAR-10H.} The CIFAR-10H~\citep{peterson2019human} provides human soft labels for the CIFAR-10 test set, with 511400 annotations from 2571 workers and approximately 51 labels per image, capturing human perceptual variability.

\textbf{CIFAR-100.} The CIFAR-100 dataset~\citep{krizhevsky2009learning} contains 60,000 colour images (32×32) across 100 fine-grained classes, grouped into 20 superclasses, with 600 images per class and a 50,000/10,000 train/test split.

\textbf{MiceBone.} The MiceBone dataset~\citep{schmarje2022one} comprises 2D slices from 3D scans of mouse bone collagen fibres, originally introduced in~\citep{schmarje20192d}. It consists of three classes: similar (``g'') and dissimilar (``ug'') collagen fibre orientations, and not relevant (``nr'') regions due to noise or background. Annotations were collected from four annotators (three hired workers and one domain expert), who received reference-guided training before annotation. An initial familiarisation phase was discarded, and the final labels are based on majority vote.

\textbf{TreeVersity.} The Treeversity\#1 dataset~\citep{schmarje2022one} originates from a crowdsourced plant image collection of \href{https://arboretum.harvard.edu/research/data-resources/}{Arnold Arboretum of Harvard University} with up to 22 original tags, simplified into six unified classes. Images were cropped to remove metadata. To capture annotator disagreement, only images with at least three annotations are retained; annotations from different users often yield conflicting tags for the same image. The total number of annotations and the agreement with the majority vote per image are detailed in~\citep{schmarje2022one}.

\subsection{Baselines}
\label{App:Baselines}
\textbf{Evidential Deep Learning (EDL).} 
Our implementation of EDL~\citep{sensoy2018evidential} follows~\citet{hofmanefficient}. A single neural network is employed with a SoftPlus activation on the final layer to ensure non-negative evidence outputs. The model is trained using the Type II maximum likelihood loss with a Kullback--Leibler divergence regularisation term, whose coefficient follows a warm-up schedule \(\lambda_i = \min(1, i/10)\) at epoch \(i\). 

At inference time, the predicted evidence parameterises a Dirichlet distribution. For an input \(x\), let \({e}(x) = [e_1(x), \ldots, e_K(x)] \in \mathbb{R}_{\geq 0}^{K}\) denotes the output evidence, where \(K\) is the number of classes. The Dirichlet parameters and total strength are then given by
\begin{equation}
{\alpha}(x) = {e}(x) + \mathbf{1}, \qquad
S(x) = \sum\nolimits_{k=1}^{K} \alpha_k(x).
\end{equation}
Epistemic uncertainty (EU) is quantified as follows~\citep{sensoy2018evidential}:
\begin{equation}
\mathrm{EU}_{\mathrm{EDL}}(x) = \frac{K}{S(x)}
= \frac{K}{\sum_{k=1}^{K} \bigl(e_k(x) + 1\bigr)}.
\label{eq:edl_uncertainty}
\end{equation}

\textbf{Laplace Bridge with Bayesian Deep Networks (LbBnn).}
LbBnn~\citep{hobbhahn2022fast} starts from a standard neural network and uses the Laplace Bridge to approximate a Bayesian posterior efficiently. It builds a Gaussian approximation around the network's logits and then applies the Laplace Bridge to analytically map it to a Dirichlet distribution over softmax outputs. This closed-form approximation eliminates the need for sampling, yielding efficient uncertainty estimates. Given the Dirichlet prediction, the EU is quantified similarly to EDL via~\cref{eq:edl_uncertainty}. Our implementation is based on~\citet{software:pytorch-laplace}, adopting the last-layer and Kronecker-factored approximation for Laplace approximation to balance computational complexity and model performance. 

\textbf{Efficient Credal Predictions through Decalibration (Decali).}
Decali~\citep{hofmanefficient} is implemented as a post-hoc method following the procedure described in the original work. Given the logits of a maximum likelihood estimation predictor, Decali solves two convex optimisation problems for each class to determine the lower and upper bounds of a plausible probability interval under a relative likelihood constraint, referred to as an \(\alpha\)-cut. The relative likelihood budget is estimated from the training data. Each optimisation problem perturbs the logit of a single class by adding a constant, thereby obtaining the corresponding lower or upper probability bound. We set \(\alpha = 1.0\), which was reported to achieve the best epistemic uncertainty quantification performance in downstream tasks~\citep{hofmanefficient}. 

At inference time, given an input \(x\), let \(\bigl[\underline{p}_k(x), \overline{p}_k(x)\bigr]\) denote the resulting class-wise probability intervals. The EU is then quantified as~\citep{wang2024CredalEnsembles}:
\begin{equation}
\mathrm{EU}_{\mathrm{Decali}}(x)
=
\max_{{p} \in C(x)} \sum\nolimits_{k=1}^{K} -p_k \log_2 p_k
-
\min_{{p} \in C(x)} \sum\nolimits_{k=1}^{K} -p_k \log_2 p_k,
\end{equation}
where
\[
C(x) = \Big\{ {p} \in \Delta^{K-1} : \underline{p}_k(x) \le p_k \le \overline{p}_k(x), \ k=1,\dots,K \Big\}.
\]
Test accuracy and ECE are evaluated using the original softmax probabilities before perturbation~\citep{hofmanefficient}.

\textbf{Dirichlet-Approximated Possibilistic Posterior Predictions (DAPPr).}
DAPPr~\citep{nipossibilistic} defines a possibilistic posterior over model parameters, projects it to the prediction space via supremum operators to obtain an epistemic uncertainty function, and approximates this projection using learned Dirichlet possibility functions. This approximation yields a closed-form solution under the high-capacity assumption of modern overparameterised networks with cross-entropy loss, resulting in a simple training objective with minimal regularisation. We implement DAPPr in our evaluation by following the original work. At inference time, DAPPr outputs a Dirichlet distribution, with EU quantified via the same formulation as EDL in~\cref{eq:edl_uncertainty}.

\textbf{Monte Carlo Dropout (MCDO).} MCDO~\citep{gal2016dropout} approximates Bayesian inference by retaining dropout during inference. Multiple stochastic forward passes are performed, each sampling a different sub-network via distinct dropout masks. Following the practice, we set the dropout rate to \(0.3\) and use \(M = 5\) forward passes. Epistemic uncertainty is quantified using the approximate mutual information (MI):
\begin{equation}
\mathrm{EU}_{\mathrm{MCDO}}(x)
=
\sum\nolimits_{k=1}^{K} -\tilde{p}_k(x) \log_2 \tilde{p}_k(x)
-
\frac{1}{M} \sum\nolimits_{i=1}^{M} \sum\nolimits_{k=1}^{K} -p_{i,k}(x) \log_2 p_{i,k}(x),
\label{eq:eu_mcdo}
\end{equation}
where \(\tilde{p}_k\) is the \(k\)-th component of the averaged prediction \(\tilde{{p}} = \frac{1}{M} \sum_{i=1}^{M} {p}_i\). This MI corresponds to the standard decomposition of total uncertainty into aleatoric and epistemic components~\citep{tpamiWANG}.

\textbf{Deep Ensembles (DE).}
DE~\citep{lakshminarayanan2017simple} quantifies prediction uncertainty by training an ensemble of \(M\) neural networks independently with different random initialisations. At inference time, each ensemble member produces a probability vector; the ensemble prediction is obtained by averaging these outputs. Epistemic uncertainty is quantified via the approximate mutual information as in~\cref{eq:eu_mcdo}, where the \(M\) sampled probability vectors correspond to the outputs of the \(M\) ensemble members. In our experiment, \(M=5\).
\subsection{Training Configurations}
\label{App:TrainingConfigurations}
\subsubsection{Annotation Imprecision: Annotator Disagreement}
\label{sub-App:Disagreement}
% For the \textbf{CIFAR-10H} experiments, all methods are trained from scratch on the CIFAR-10H dataset using a ResNet-18~\citep{he2016deep} backbone. Evaluation is performed on the original train split of CIFAR-10 (samples indexed 40000 to 50000, totalling 10000 samples). 
For the \textbf{CIFAR-10H} experiments, all methods are trained from scratch on the CIFAR-10H dataset using a ResNet-18~\citep{he2016deep} backbone. Evaluation is performed on the last 10{,}000 images of the CIFAR-10 training set (indices 40{,}000--50{,}000 under standard torchvision ordering, without shuffling), which do not overlap with the CIFAR-10H data drawn from the original CIFAR-10 test set.
For the \textbf{MiceBone} and \textbf{TreeVersity} datasets, all methods adopt an 80\%:20\% train--test split and use a DenseNet-121~\citep{huang2017densely} backbone initialised with ImageNet pretrained weights. For these two datasets, MCDO and LbBnn are excluded due to the lack of
a compatible pre-trained model and out-of-memory errors on an NVIDIA RTX 5090, respectively. All methods are trained with cross-entropy loss using soft labels. Comprehensive hyperparameter configurations are provided in \tablename~\ref{tab:training_hyperparamsA}.
\begin{table}[!htbp]
\centering
\caption{Hyperparameters used for annotator-disagreement experiments.}
\label{tab:training_hyperparamsA}
\small
\begin{tabular}{lccc}
\toprule
\textbf{Hyperparameter} & \textbf{CIFAR-10H} & \textbf{MiceBone} & \textbf{TreeVersity} \\
\midrule
Backbone      & ResNet-18                    & DenseNet-121                  & DenseNet-121            \\
Classes       & 10                           & 3                             & 6                       \\
Input size    & $(3\times32\times32)$        & $(3\times224\times224)$       & $(3\times224\times224)$ \\
Batch size    & 128                          & 128                           & 128                     \\
Epochs        & 150                          & 60                            & 60                      \\
Learning rate & 0.1                          & 0.00001                       & 0.00001                 \\
Weight decay  & 0.0005                       & 0.0005                        & 0.0005                  \\
Optimiser     & SGD                          & Adam                          & Adam                    \\
Learning rate scheduler  & Cosine Annealing              & --                            & --          \\
\bottomrule
\end{tabular}
\end{table}
\subsubsection{Annotation Imprecision: Teacher Predictions}
\label{sub-App:Teacher}
In this section, we use predictions from a teacher model~\citep{hinton2015distilling} on the training data to simulate annotation imprecision. Specifically, given an input \(x\), let \(h_{\mathrm{teacher}}(\cdot)\) denote a well-trained teacher model, and let \({z}(x) = h_{\mathrm{teacher}}(x)\) denotes its logit output. The soft label distribution \({q}_L\) is then obtained via the standard softmax with temperature:
\begin{equation}
q_{L, k}(x) = \frac{\exp\bigl(z_k(x) / T\bigr)}{\sum_{j=1}^{K} \exp\bigl(z_j(x) / T\bigr)}, \quad k = 1, \dots, K,
\label{eq: teacher_model}
\end{equation}
where \(T\) is the temperature parameter. We experiment with temperatures \(T \in \{1.0, 2.5\}\) and three publicly available teacher backbones~\citep{chenyaofo_pytorch_cifar_models} (ResNet-20, ShuffleNetV2-0.5×~\citep{ma2018shufflenet}, and MobileNetV2-x0-5~\citep{sandler2018mobilenetv2}) for both CIFAR-10 and CIFAR-100 experiments.

An 80\%:20\% train--test split is adopted for all experiments. For \textbf{CIFAR-10}, all methods are trained from scratch using a ResNet-18 backbone with teacher-predicted labels, and a cosine annealing learning rate scheduler is applied. For \textbf{CIFAR-100}, all methods are trained from scratch using a WRN-28-4 backbone with teacher-predicted labels, and a learning rate schedule with linear warmup and cosine annealing is employed: the learning rate is linearly ramped up from \(0.001\) to the base learning rate during the first 5 epochs, and then gradually decays to \(10^{-6}\) following a cosine curve for the remaining training. Detailed hyperparameter configurations are provided in \tablename~\ref{tab:training_hyperparamsBC}.
\begin{table}[!htbp]
\centering
\caption{Hyperparameters used for both teacher-prediction and label-smoothing experiments.}
\label{tab:training_hyperparamsBC}
\small
\setlength\tabcolsep{20pt}
\begin{tabular}{lcc}
\toprule
\textbf{Hyperparameter} & \textbf{CIFAR-10} & \textbf{CIFAR-100} \\
\midrule
Backbone      & ResNet-18                    & WRN-28-4                 \\
Classes       & 10                           & 100                        \\
Input size    & $(3\times32\times32)$        & $(3\times32\times32)$     \\
Batch size    & 128                          & 128                          \\
Epochs        & 200                          & 200        \\
Learning rate & 0.1                          & 0.1                     \\
Weight decay  & 0.0005                       & 0.0005    \\
Optimiser     & SGD                          & SGD                         \\
Learning rate scheduler  & Cosine Annealing              & Warmup Cosine Annealing  \\
\bottomrule
\end{tabular}
\end{table}

\subsubsection{Annotation Imprecision: Label Smoothing}
\label{sub-App:Smooth}
In this section, we adopt label smoothing~\citep{szegedy2016rethinking} to simulate annotation imprecision in the training data. Given a hard one-hot label \({e}_y\), the soft label distribution \({q}_L\) is constructed as:
\begin{equation}
q_{L, k} = (1 - \epsilon) \, \mathbf{1}_{k=y} + \frac{\epsilon}{K},
\label{eq:label_smoothing}
\end{equation}
where \(\epsilon\) is the smoothing coefficient, varied over \(\{0.05, 0.1, 0.15, 0.2\}\) in our experiments, and \(\mathbf{1}_{k=y}\) is the indicator function. The resulting credal label for supervision is then obtained via~\cref{eq:LabelConstruction}.

For the \textbf{CIFAR-10} experiments, all methods are trained from scratch on the CIFAR-10 dataset using a ResNet-18 backbone with label smoothing. We apply cosine annealing as the learning rate scheduler. For the \textbf{CIFAR-100} experiments, all methods are trained from scratch on the CIFAR-100 dataset using a WRN-28-4 backbone with label smoothing. We employ a learning rate schedule with linear warmup and cosine annealing: the learning rate is linearly ramped up from 0.001 to the base learning rate during the first 5 epochs, and then gradually decays to \(10^{-6}\) following a cosine curve for the remaining training. An 80\%:20\% train--test split is adopted for all experiments. Comprehensive hyperparameter configurations are provided in \tablename~\ref{tab:training_hyperparamsBC}.

\subsection{Compute Resources}
\label{sub-App:resource}
All experiments in this work are conducted using the computing resources 
listed in \tablename~\ref{tab:computing}.
\begin{table}[!htbp]
\centering
\small
\caption{Specifications of Computing Resources}
\label{tab:computing}
\begin{tabular}{ll}
\toprule
\textbf{Component} & \textbf{Specification} \\
\midrule
CPU     & Intel Core Ultra 24C/24T, 5.8GHz Max, 36MB L3 Cache \\
GPU     & $1 \times$ NVIDIA GeForce RTX 5090 (32 GB GDDR7) \\
RAM     & 192 GB DDR4 ECC DIMM \\
\bottomrule
\end{tabular}
\end{table}
\subsection{Evaluation Metrics Implementation}
\label{App:metrics}
\textbf{Expected Calibration Error.} 
In classification settings, a well-calibrated prediction is expected to have a confidence value of 80\% and be correct in approximately 80\% of the test cases. To calculate expected calibration error (ECE), probabilistic predictions are split into a predetermined number $Q$ of bins $B$ of equal confidence range. The ECE is defined as the average absolute difference between accuracy and confidence over \(G\) equally spaced bins~\citep{guo2017calibration}:
\begin{equation}
\mathrm{ECE} := \sum\nolimits_{g=1}^{G} \frac{|B_g|}{n} \Big| \mathrm{acc}(B_g) - \mathrm{conf}(B_g) \Big|,
\label{Eq: ECE}
\end{equation}
where \(|B_g|\) is the number of samples in the \(g\)-th bin and \(n\) is the total number of samples. Following standard practice, we use \(G = 10\) bins for evaluation.

\textbf{Selective Classification.} 
Selective classification evaluates epistemic uncertainty (EU) quality by leveraging the expectation that misclassified samples receive higher EU values than correctly classified ones. Instances with high EU quantification are abstained and deferred to an expert to reduce the risk of misclassification (Algorithm~\ref{alg: SelectivePrediction}). We quantify performance using the accuracy–rejection curve~\citep{huhn2008fr3}, which plots accuracy on retained samples against the rejection rate. Reliable uncertainty estimates yield a monotonically increasing curve, whereas random rejection yields a flat curve. The area under this curve (AUARC) serves as a scalar summary~\citep{jaegercall}, with larger values indicating better selective classification performance.
\begin{algorithm}[!hbtp]
\begin{algorithmic}
\State\textbf{Input:} Test dataset $\mathcal{D}_{\text{test}}$; rejection rate $\beta$; 
epistemic-uncertainty-aware predictor $h(\cdot)$
\State \textbf{1.} Quantify epistemic predictive uncertainty given each sample, denoted by $u_{\text{EU},n}$
\State $u_{\text{EU},n} \leftarrow h(x_n) \ \forall n $ 
\State \textbf{2.} Sort uncertainty estimates in ascending order
\State let ${\pi}=\{1, 2, ..., N_t\}$ with $N_t = |\mathcal{D}_{\text{test}}|$ so that $u_{\text{EU},\pi(1)}\leq u_{\text{EU},\pi(2)} \leq ... \leq u_{\text{EU},\pi(N_t)}$
\State \textbf{3.} Select top $\lfloor (1-\beta)N_t \rfloor$ certain averaged probabilities
\State $\tilde{{p}}_{\pi(1)}, ..., \tilde{{p}}_{\pi(\lfloor (1-\beta)N_t \rfloor)}$ for class prediction 
\caption{selective classification procedure} 
\label{alg: SelectivePrediction}
\end{algorithmic}
\end{algorithm}

\textbf{Balanced Quality Score.} 
To \emph{strike a balance among accuracy, calibration, and selective classification performance}, we introduce the \emph{balanced quality score} (BQS). For each seed, we independently normalise ACC, \(1-\text{ECE}\), and AUARC to \([0,1]\) via min-max scaling across all compared methods within each dataset. BQS is then computed as the average of the three normalised components:
\begin{equation}
\mathrm{BQS} = \frac{1}{3} \Bigl( \widehat{\mathrm{ACC}} + \widehat{(1-\mathrm{ECE})} + \widehat{\mathrm{AUARC}} \Bigr),
\label{eq:bqs}
\end{equation}
where each \(\widehat{\,\cdot\,}\) denotes the min-max normalised value computed across all compared methods for that seed and dataset. We report the mean and standard deviation of \(\mathrm{BQS}\) across seeds.

\section{Additional Experiment Results}
\label{App:AddExperiments}
We report performance comparisons and accuracy–rejection curves under the three annotation imprecision scenarios---annotator disagreement, teacher predictions, and label smoothing---in Tables~\ref{Table: DisagreementDetails}, \ref{Table: AblationTeacher}, and \ref{Table: AblationSmoothing}, and Figures~\ref{Figure: arc_dis}, \ref{Figure: arc_teacher}, and \ref{Figure: arc_label}, respectively. 

\tablename~\ref{Table: AblationTeacher_original} further reports results using the teachers' original predictions as the source of annotation imprecision. These results consistently confirm the robustness of POCC in achieving the highest BQS scores across settings.
\begin{table}[!htbp]
\caption{Performance comparison (in $\%$)  under annotator-disagreement annotation imprecision across various datasets, averaged over 10 runs.}
\label{Table: DisagreementDetails}
\centering
\small
\setlength\tabcolsep{5pt}
\begin{tabular}{lcccc|ccccc}
\toprule
& \multicolumn{4}{c|}{MiceBone Dataset} & \multicolumn{4}{c}{TreeVersity Dataset} \\
\cmidrule(lr){2-5} \cmidrule(lr){6-9}
Method & ACC $\uparrow$ & ECE $\downarrow$ & AUARC $\uparrow$ & BQS $\uparrow$ & ACC $\uparrow$ & ECE $\downarrow$ & AUARC $\uparrow$ & BQS $\uparrow$ \\
\midrule \midrule
EDL
  & $88.5 \scriptstyle{\pm 0.8}$ & $9.9 \scriptstyle{\pm 0.8}$ & $\mathbf{97.4 \scriptstyle{\pm 0.3}}$ & $\underline{68.7 \scriptstyle{\pm 3.6}}$
  & $85.5 \scriptstyle{\pm 0.6}$ & $17.1 \scriptstyle{\pm 0.7}$ & $\underline{95.7 \scriptstyle{\pm 0.3}}$ & $58.4 \scriptstyle{\pm 1.9}$ \\
Decali
  & $88.7 \scriptstyle{\pm 1.2}$ & $\underline{5.8 \scriptstyle{\pm 1.0}}$ & $93.6 \scriptstyle{\pm 0.9}$ & $51.2 \scriptstyle{\pm 6.0}$
  & $\mathbf{86.9 \scriptstyle{\pm 0.4}}$ & $\mathbf{2.3 \scriptstyle{\pm 0.3}}$ & $95.0 \scriptstyle{\pm 0.4}$ & $\underline{92.4 \scriptstyle{\pm 2.2}}$ \\
DAPPr
  & $\underline{89.0 \scriptstyle{\pm 1.1}}$ & $13.2 \scriptstyle{\pm 0.9}$ & $96.4 \scriptstyle{\pm 0.5}$ & $53.6 \scriptstyle{\pm 5.1}$
  & $\underline{86.6 \scriptstyle{\pm 0.5}}$ & $17.0 \scriptstyle{\pm 0.4}$ & $95.5 \scriptstyle{\pm 0.3}$ & $61.5 \scriptstyle{\pm 1.5}$ \\
\rowcolor{OxfordBlueLight}
Ours
  & $\mathbf{89.1 \scriptstyle{\pm 0.9}}$ & $\mathbf{2.7 \scriptstyle{\pm 0.4}}$ & $\underline{97.3 \scriptstyle{\pm 0.3}}$ & $\mathbf{96.1 \scriptstyle{\pm 4.6}}$
  & $86.2 \scriptstyle{\pm 0.6}$ & $\underline{2.5 \scriptstyle{\pm 0.5}}$ & $\mathbf{96.7 \scriptstyle{\pm 0.2}}$ & $\mathbf{97.1 \scriptstyle{\pm 1.9}}$ \\
\midrule
DE
  & $85.5 \scriptstyle{\pm 0.9}$ & $7.5 \scriptstyle{\pm 0.8}$ & $94.3 \scriptstyle{\pm 0.5}$ & $23.6 \scriptstyle{\pm 5.5}$
  & $75.7 \scriptstyle{\pm 1.1}$ & $6.2 \scriptstyle{\pm 0.9}$ & $88.6 \scriptstyle{\pm 0.8}$ & $24.4 \scriptstyle{\pm 1.7}$ \\
\bottomrule
\end{tabular}
\end{table}
\begin{figure}[!htbp]
\centering
\includegraphics[width=\linewidth]{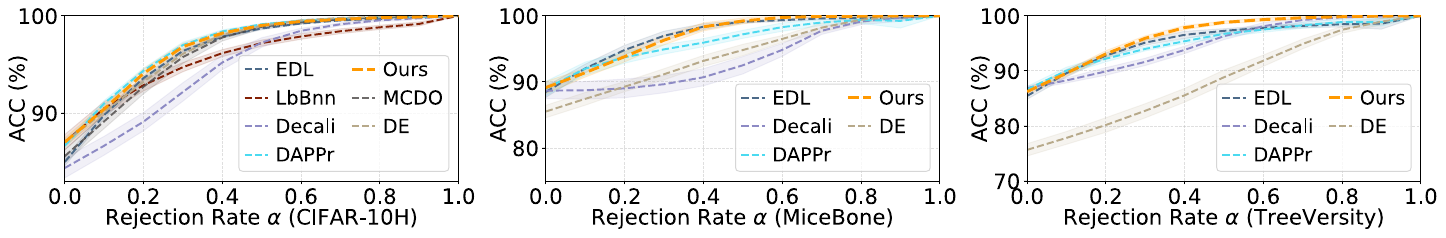}
\caption{Accuracy rejection curves under annotator-disagreement annotation imprecision.}
\label{Figure: arc_dis}
\end{figure}

\begin{table}[!htbp]
\caption{Performance comparison (in $\%$)  under teacher-prediction annotation imprecision with varying settings, averaged over 10 runs.}
\label{Table: AblationTeacher}
\centering
\small
\setlength\tabcolsep{4pt}
\begin{tabular}{lccc|c||ccc|c}
\toprule
\multirow{2}{*}{Method} & \multicolumn{4}{c||}{CIFAR-10} & \multicolumn{4}{c}{CIFAR-100} \\
\cmidrule(lr){2-5} \cmidrule(lr){6-9}
& ACC $\uparrow$ & ECE $\downarrow$ & AUARC $\uparrow$ & BQS $\uparrow$ & ACC $\uparrow$ & ECE $\downarrow$ & AUARC $\uparrow$ & BQS $\uparrow$ \\
\midrule \midrule
\multicolumn{9}{c}{Teacher: ResNet-20}\\ \midrule\midrule
EDL    & $93.4 \scriptstyle{\pm 0.2}$ & $9.0 \scriptstyle{\pm 0.2}$ & $99.1 \scriptstyle{\pm 0.1}$ & $51.6 \scriptstyle{\pm 3.8}$ & $33.1 \scriptstyle{\pm 5.7}$ & $\underline{23.3 \scriptstyle{\pm 3.9}}$ & $65.8 \scriptstyle{\pm 6.1}$ & $30.2 \scriptstyle{\pm 2.9}$ \\
LbBnn  & $\underline{95.1 \scriptstyle{\pm 0.1}}$ & $28.1 \scriptstyle{\pm 0.2}$ & $\underline{99.2 \scriptstyle{\pm 0.0}}$ & $54.2 \scriptstyle{\pm 3.3}$ & $\underline{73.1 \scriptstyle{\pm 0.2}}$ & $63.1 \scriptstyle{\pm 0.2}$ & $63.4 \scriptstyle{\pm 0.4}$ & $31.9 \scriptstyle{\pm 2.4}$ \\
Decali & $93.5 \scriptstyle{\pm 0.1}$ & $8.5 \scriptstyle{\pm 0.1}$ & $98.0 \scriptstyle{\pm 0.1}$ & $29.7 \scriptstyle{\pm 2.5}$ & $71.0 \scriptstyle{\pm 0.2}$ & $27.8 \scriptstyle{\pm 0.2}$ & $73.3 \scriptstyle{\pm 0.2}$ & $64.2 \scriptstyle{\pm 1.9}$ \\
DAPPr  & $93.4 \scriptstyle{\pm 0.2}$ & $14.6 \scriptstyle{\pm 0.2}$ & $99.0 \scriptstyle{\pm 0.0}$ & $42.8 \scriptstyle{\pm 3.1}$ & $72.5 \scriptstyle{\pm 0.4}$ & $52.8 \scriptstyle{\pm 0.3}$ & $78.3 \scriptstyle{\pm 0.3}$ & $55.1 \scriptstyle{\pm 1.6}$ \\
\rowcolor{OxfordBlueLight}
Ours   & $\mathbf{95.5 \scriptstyle{\pm 0.1}}$ & $\mathbf{3.0 \scriptstyle{\pm 0.1}}$ & $\mathbf{99.5 \scriptstyle{\pm 0.0}}$ & $\mathbf{99.9 \scriptstyle{\pm 0.3}}$ & $\mathbf{77.2 \scriptstyle{\pm 0.2}}$ & $\mathbf{10.6 \scriptstyle{\pm 0.4}}$ & $\mathbf{90.6 \scriptstyle{\pm 0.2}}$ & $\mathbf{100.0 \scriptstyle{\pm 0.0}}$ \\
\midrule
MCDO   & $93.5 \scriptstyle{\pm 0.2}$ & $\underline{8.5 \scriptstyle{\pm 0.2}}$ & $99.0 \scriptstyle{\pm 0.1}$ & $50.2 \scriptstyle{\pm 4.8}$ & $71.1 \scriptstyle{\pm 0.3}$ & $28.2 \scriptstyle{\pm 0.3}$ & $83.2 \scriptstyle{\pm 0.7}$ & $75.5 \scriptstyle{\pm 1.3}$ \\
DE     & $94.2 \scriptstyle{\pm 0.1}$ & $9.5 \scriptstyle{\pm 0.1}$ & $99.2 \scriptstyle{\pm 0.0}$ & $\underline{65.3 \scriptstyle{\pm 3.0}}$ & $72.7 \scriptstyle{\pm 0.3}$ & $29.9 \scriptstyle{\pm 0.3}$ & $\underline{86.5 \scriptstyle{\pm 0.3}}$ & $\underline{79.5 \scriptstyle{\pm 0.7}}$ \\
\midrule \midrule
\multicolumn{9}{c}{Teacher: ShuffleNetV2-0.5×}\\ \midrule\midrule 
EDL    & $92.9 \scriptstyle{\pm 0.3}$ & $\underline{9.1 \scriptstyle{\pm 0.1}}$ & $\underline{98.9 \scriptstyle{\pm 0.1}}$ & $\underline{67.0 \scriptstyle{\pm 3.2}}$ & $33.0 \scriptstyle{\pm 5.6}$ & $\underline{23.3 \scriptstyle{\pm 3.8}}$ & $65.9 \scriptstyle{\pm 6.3}$ & $31.9 \scriptstyle{\pm 3.6}$ \\
LbBnn  & $\underline{94.1 \scriptstyle{\pm 0.1}}$ & $33.2 \scriptstyle{\pm 0.2}$ & $98.3 \scriptstyle{\pm 0.1}$ & $43.3 \scriptstyle{\pm 1.8}$ & $72.6 \scriptstyle{\pm 0.2}$ & $65.4 \scriptstyle{\pm 0.2}$ & $62.6 \scriptstyle{\pm 0.4}$ & $31.0 \scriptstyle{\pm 2.4}$ \\
Decali & $92.1 \scriptstyle{\pm 0.2}$ & $12.1 \scriptstyle{\pm 0.2}$ & $96.8 \scriptstyle{\pm 0.1}$ & $26.9 \scriptstyle{\pm 1.7}$ & $71.0 \scriptstyle{\pm 0.2}$ & $27.4 \scriptstyle{\pm 0.2}$ & $71.9 \scriptstyle{\pm 0.3}$ & $64.0 \scriptstyle{\pm 1.7}$ \\
DAPPr  & $92.5 \scriptstyle{\pm 0.2}$ & $18.4 \scriptstyle{\pm 0.2}$ & $98.6 \scriptstyle{\pm 0.1}$ & $47.3 \scriptstyle{\pm 3.8}$ & $72.6 \scriptstyle{\pm 0.4}$ & $52.3 \scriptstyle{\pm 0.4}$ & $79.7 \scriptstyle{\pm 0.2}$ & $58.8 \scriptstyle{\pm 1.2}$ \\
\rowcolor{OxfordBlueLight}
Ours   & $\mathbf{94.9 \scriptstyle{\pm 0.1}}$ & $\mathbf{5.4 \scriptstyle{\pm 0.2}}$ & $\mathbf{99.4 \scriptstyle{\pm 0.0}}$ & $\mathbf{100.0 \scriptstyle{\pm 0.0}}$ & $\mathbf{77.8 \scriptstyle{\pm 0.3}}$ & $\mathbf{11.7 \scriptstyle{\pm 0.3}}$ & $\mathbf{90.2 \scriptstyle{\pm 0.3}}$ & $\mathbf{100.0 \scriptstyle{\pm 0.0}}$ \\
\midrule
MCDO   & $92.0 \scriptstyle{\pm 0.2}$ & $12.2 \scriptstyle{\pm 0.2}$ & $98.3 \scriptstyle{\pm 0.1}$ & $45.3 \scriptstyle{\pm 2.0}$ & $70.8 \scriptstyle{\pm 0.4}$ & $27.7 \scriptstyle{\pm 0.3}$ & $83.2 \scriptstyle{\pm 0.6}$ & $76.7 \scriptstyle{\pm 1.4}$ \\
DE     & $92.9 \scriptstyle{\pm 0.1}$ & $13.1 \scriptstyle{\pm 0.1}$ & $98.7 \scriptstyle{\pm 0.0}$ & $58.5 \scriptstyle{\pm 2.1}$ & $\underline{73.0 \scriptstyle{\pm 0.2}}$ & $29.9 \scriptstyle{\pm 0.1}$ & $\underline{87.2 \scriptstyle{\pm 0.2}}$ & $\underline{81.6 \scriptstyle{\pm 0.9}}$ \\
\midrule \midrule
\multicolumn{9}{c}{Teacher: MobileNetV2-x0-5}\\ \midrule\midrule
EDL    & $93.2 \scriptstyle{\pm 0.2}$ & $\underline{9.1 \scriptstyle{\pm 0.2}}$ & $99.1 \scriptstyle{\pm 0.1}$ & $53.1 \scriptstyle{\pm 1.1}$ & $37.8 \scriptstyle{\pm 1.5}$ & $26.6 \scriptstyle{\pm 1.1}$ & $71.4 \scriptstyle{\pm 1.4}$ & $32.9 \scriptstyle{\pm 1.3}$ \\
LbBnn  & $\underline{95.3 \scriptstyle{\pm 0.1}}$ & $37.4 \scriptstyle{\pm 0.3}$ & $99.0 \scriptstyle{\pm 0.0}$ & $52.5 \scriptstyle{\pm 2.3}$ & $76.1 \scriptstyle{\pm 0.2}$ & $68.2 \scriptstyle{\pm 0.3}$ & $64.5 \scriptstyle{\pm 0.4}$ & $31.0 \scriptstyle{\pm 0.3}$ \\
Decali & $93.6 \scriptstyle{\pm 0.1}$ & $10.3 \scriptstyle{\pm 0.1}$ & $97.9 \scriptstyle{\pm 0.0}$ & $33.7 \scriptstyle{\pm 2.1}$ & $74.4 \scriptstyle{\pm 0.3}$ & $\underline{25.8 \scriptstyle{\pm 0.3}}$ & $77.4 \scriptstyle{\pm 0.6}$ & $70.5 \scriptstyle{\pm 1.0}$ \\
DAPPr  & $93.8 \scriptstyle{\pm 0.2}$ & $16.1 \scriptstyle{\pm 0.2}$ & $99.0 \scriptstyle{\pm 0.0}$ & $53.3 \scriptstyle{\pm 2.8}$ & $75.3 \scriptstyle{\pm 0.4}$ & $51.9 \scriptstyle{\pm 0.4}$ & $84.6 \scriptstyle{\pm 0.2}$ & $65.2 \scriptstyle{\pm 0.6}$ \\
\rowcolor{OxfordBlueLight}
Ours   & $\mathbf{95.5 \scriptstyle{\pm 0.1}}$ & $\mathbf{4.2 \scriptstyle{\pm 0.2}}$ & $\mathbf{99.5 \scriptstyle{\pm 0.0}}$ & $\mathbf{99.8 \scriptstyle{\pm 0.4}}$ & $\mathbf{79.0 \scriptstyle{\pm 0.2}}$ & $\mathbf{10.9 \scriptstyle{\pm 0.2}}$ & $\mathbf{90.9 \scriptstyle{\pm 0.2}}$ & $\mathbf{100.0 \scriptstyle{\pm 0.0}}$ \\
\midrule
MCDO   & $93.4 \scriptstyle{\pm 0.1}$ & $10.2 \scriptstyle{\pm 0.2}$ & $99.0 \scriptstyle{\pm 0.0}$ & $52.5 \scriptstyle{\pm 2.5}$ & $73.9 \scriptstyle{\pm 0.3}$ & $25.8 \scriptstyle{\pm 0.2}$ & $86.5 \scriptstyle{\pm 0.5}$ & $81.7 \scriptstyle{\pm 1.1}$ \\
DE     & $94.2 \scriptstyle{\pm 0.1}$ & $11.4 \scriptstyle{\pm 0.1}$ & $\underline{99.1 \scriptstyle{\pm 0.0}}$ & $\underline{66.9 \scriptstyle{\pm 1.4}}$ & $\underline{76.4 \scriptstyle{\pm 0.1}}$ & $28.4 \scriptstyle{\pm 0.1}$ & $\underline{90.0 \scriptstyle{\pm 0.1}}$ & $\underline{86.5 \scriptstyle{\pm 0.6}}$ \\
\bottomrule
\end{tabular}
\end{table}
\begin{figure}[!htbp]
\centering
\includegraphics[width=\linewidth]{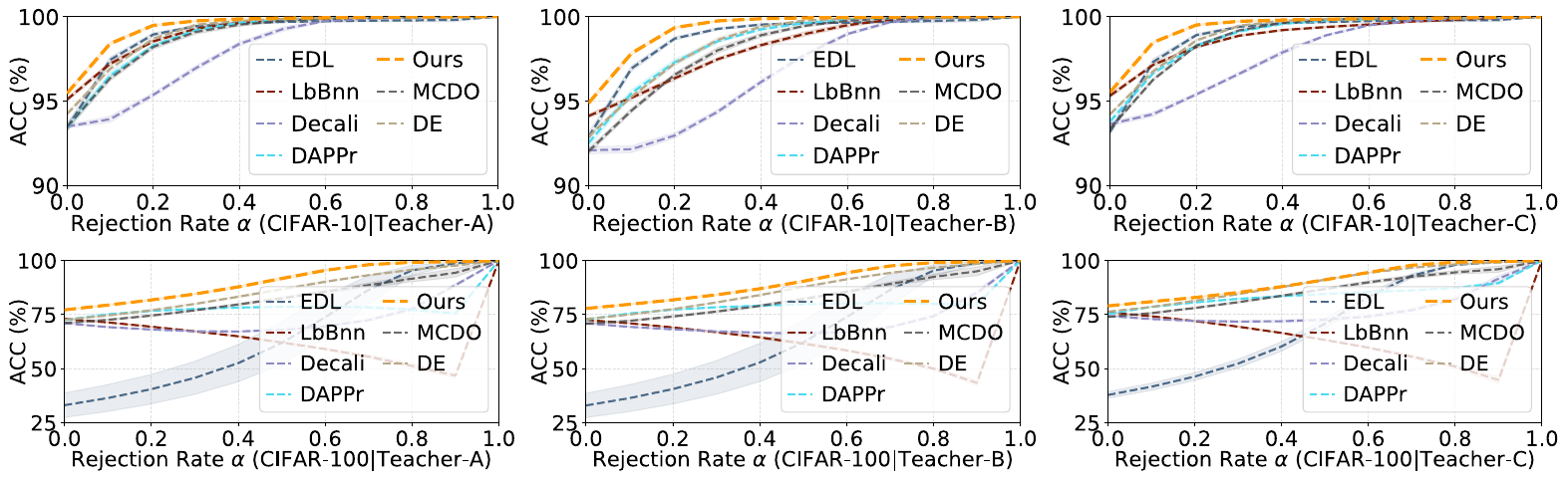}
\caption{Accuracy rejection curves under teacher-prediction annotation imprecision.}
\label{Figure: arc_teacher}
\end{figure}

\begin{table}[!htbp]
\caption{Label-smoothing results (\%) across \(\epsilon\). POCC achieves the highest mean BQS in seven of eight settings; the sole exception is \(\epsilon=0.1\) on CIFAR-10, where it ranks second behind DE (5 members), with the gap attributable to marginal differences in accuracy and AUARC. Statistical significance tests (\tablename~\ref{Table: ST-smoothing}, Appendix~\ref{App:StatisticalTest}) confirm that POCC ranks first statistically across all settings.}
\label{Table: AblationSmoothing}
\centering
\small
\setlength\tabcolsep{3.5pt}
\begin{tabular}{clccc|c||ccc|c}
\toprule
& \multirow{2}{*}{Method} & \multicolumn{4}{c||}{CIFAR-10} & \multicolumn{4}{c}{CIFAR-100} \\
\cmidrule(lr){3-6} \cmidrule(lr){7-10}
& & ACC $\uparrow$ & ECE $\downarrow$ & AUARC $\uparrow$ & BQS $\uparrow$ & ACC $\uparrow$ & ECE $\downarrow$ & AUARC $\uparrow$ & BQS $\uparrow$ \\
\midrule \midrule
\multirow{7}{*}{\rotatebox{90}{$\epsilon=0.05$}}
& EDL    & $93.4 \scriptstyle{\pm 0.3}$ & $9.1 \scriptstyle{\pm 0.2}$ & $99.1 \scriptstyle{\pm 0.1}$ & $48.6 \scriptstyle{\pm 4.3}$ & $26.3 \scriptstyle{\pm 1.2}$ & $18.7 \scriptstyle{\pm 1.0}$ & $58.5 \scriptstyle{\pm 1.6}$ & $24.3 \scriptstyle{\pm 0.6}$ \\
& LbBnn  & $\underline{95.4 \scriptstyle{\pm 0.2}}$ & $29.9 \scriptstyle{\pm 0.3}$ & $99.1 \scriptstyle{\pm 0.1}$ & $53.6 \scriptstyle{\pm 2.6}$ & $\underline{79.3 \scriptstyle{\pm 0.3}}$ & $61.6 \scriptstyle{\pm 0.5}$ & $74.5 \scriptstyle{\pm 0.9}$ & $47.7 \scriptstyle{\pm 1.1}$ \\
& Decali & $93.9 \scriptstyle{\pm 0.3}$ & $3.6 \scriptstyle{\pm 0.1}$ & $98.3 \scriptstyle{\pm 0.2}$ & $39.0 \scriptstyle{\pm 5.2}$ & $76.2 \scriptstyle{\pm 0.3}$ & $\underline{3.5 \scriptstyle{\pm 0.4}}$ & $90.7 \scriptstyle{\pm 0.3}$ & $94.0 \scriptstyle{\pm 0.7}$ \\
& DAPPr  & $93.9 \scriptstyle{\pm 0.3}$ & $8.9 \scriptstyle{\pm 0.2}$ & $99.1 \scriptstyle{\pm 0.1}$ & $58.5 \scriptstyle{\pm 6.4}$ & $75.1 \scriptstyle{\pm 0.4}$ & $35.3 \scriptstyle{\pm 0.4}$ & $91.9 \scriptstyle{\pm 0.2}$ & $76.5 \scriptstyle{\pm 0.4}$ \\
& \cellcolor{OxfordBlueLight}Ours   & \cellcolor{OxfordBlueLight}$\mathbf{95.5 \scriptstyle{\pm 0.1}}$ & \cellcolor{OxfordBlueLight}$\mathbf{1.5 \scriptstyle{\pm 0.1}}$ & \cellcolor{OxfordBlueLight}$\underline{99.2 \scriptstyle{\pm 0.1}}$ & \cellcolor{OxfordBlueLight}$\mathbf{90.6 \scriptstyle{\pm 4.7}}$ & \cellcolor{OxfordBlueLight}$\underline{79.3 \scriptstyle{\pm 0.2}}$ & \cellcolor{OxfordBlueLight}$\mathbf{2.7 \scriptstyle{\pm 0.2}}$ & \cellcolor{OxfordBlueLight}$\underline{93.4 \scriptstyle{\pm 0.1}}$ & \cellcolor{OxfordBlueLight}$\mathbf{99.0 \scriptstyle{\pm 0.3}}$ \\
\cmidrule{2-10}
& MCDO   & $93.7 \scriptstyle{\pm 0.2}$ & $\underline{3.5 \scriptstyle{\pm 0.1}}$ & $99.1 \scriptstyle{\pm 0.1}$ & $60.2 \scriptstyle{\pm 6.3}$ & $76.1 \scriptstyle{\pm 0.3}$ & $3.8 \scriptstyle{\pm 0.4}$ & $91.8 \scriptstyle{\pm 0.2}$ & $94.9 \scriptstyle{\pm 0.6}$ \\
& DE     & $95.2 \scriptstyle{\pm 0.1}$ & $5.5 \scriptstyle{\pm 0.1}$ & $\mathbf{99.4 \scriptstyle{\pm 0.0}}$ & $\underline{89.7 \scriptstyle{\pm 2.3}}$ & $\mathbf{80.3 \scriptstyle{\pm 0.2}}$ & $9.1 \scriptstyle{\pm 0.2}$ & $\mathbf{93.8 \scriptstyle{\pm 0.1}}$ & $\underline{96.4 \scriptstyle{\pm 0.2}}$ \\
\midrule \midrule
\multirow{7}{*}{\rotatebox{90}{$\epsilon=0.1$}}
& EDL    & $93.3 \scriptstyle{\pm 0.1}$ & $9.2 \scriptstyle{\pm 0.3}$ & $\underline{99.1 \scriptstyle{\pm 0.0}}$ & $56.4 \scriptstyle{\pm 2.8}$ & $27.5 \scriptstyle{\pm 1.1}$ & $19.6 \scriptstyle{\pm 0.9}$ & $60.0 \scriptstyle{\pm 1.4}$ & $24.6 \scriptstyle{\pm 0.4}$ \\
& LbBnn  & $\mathbf{95.4 \scriptstyle{\pm 0.1}}$ & $45.2 \scriptstyle{\pm 0.1}$ & $98.9 \scriptstyle{\pm 0.1}$ & $56.4 \scriptstyle{\pm 3.6}$ & $79.1 \scriptstyle{\pm 0.2}$ & $70.0 \scriptstyle{\pm 0.3}$ & $69.4 \scriptstyle{\pm 1.2}$ & $41.9 \scriptstyle{\pm 1.5}$ \\
& Decali & $93.9 \scriptstyle{\pm 0.3}$ & $\underline{7.3 \scriptstyle{\pm 0.1}}$ & $97.5 \scriptstyle{\pm 0.5}$ & $39.4 \scriptstyle{\pm 3.6}$ & $75.9 \scriptstyle{\pm 0.5}$ & $\underline{5.0 \scriptstyle{\pm 0.3}}$ & $87.7 \scriptstyle{\pm 0.7}$ & $89.4 \scriptstyle{\pm 0.7}$ \\
& DAPPr  & $93.9 \scriptstyle{\pm 0.2}$ & $11.8 \scriptstyle{\pm 0.2}$ & $\underline{99.1 \scriptstyle{\pm 0.1}}$ & $63.9 \scriptstyle{\pm 4.8}$ & $75.0 \scriptstyle{\pm 0.4}$ & $36.8 \scriptstyle{\pm 0.4}$ & $91.5 \scriptstyle{\pm 0.3}$ & $77.1 \scriptstyle{\pm 0.4}$ \\
& \cellcolor{OxfordBlueLight}Ours   & \cellcolor{OxfordBlueLight}$\mathbf{95.4 \scriptstyle{\pm 0.1}}$ & \cellcolor{OxfordBlueLight}$\mathbf{3.7 \scriptstyle{\pm 0.2}}$ & \cellcolor{OxfordBlueLight}$98.8 \scriptstyle{\pm 0.2}$ & \cellcolor{OxfordBlueLight}$\underline{87.2 \scriptstyle{\pm 4.4}}$ & \cellcolor{OxfordBlueLight}$\underline{79.4 \scriptstyle{\pm 0.2}}$ & \cellcolor{OxfordBlueLight}$\mathbf{1.5 \scriptstyle{\pm 0.2}}$ & \cellcolor{OxfordBlueLight}$\underline{93.3 \scriptstyle{\pm 0.1}}$ & \cellcolor{OxfordBlueLight}$\mathbf{98.8 \scriptstyle{\pm 0.2}}$ \\
\cmidrule{2-10}
& MCDO   & $93.6 \scriptstyle{\pm 0.3}$ & $7.8 \scriptstyle{\pm 0.2}$ & $99.0 \scriptstyle{\pm 0.1}$ & $60.5 \scriptstyle{\pm 5.3}$ & $76.4 \scriptstyle{\pm 0.3}$ & $7.3 \scriptstyle{\pm 0.3}$ & $91.8 \scriptstyle{\pm 0.2}$ & $92.7 \scriptstyle{\pm 0.4}$ \\
& DE     & $95.2 \scriptstyle{\pm 0.1}$ & $10.0 \scriptstyle{\pm 0.1}$ & $\mathbf{99.4 \scriptstyle{\pm 0.0}}$ & $\mathbf{91.2 \scriptstyle{\pm 2.1}}$ & $\mathbf{80.2 \scriptstyle{\pm 0.2}}$ & $12.8 \scriptstyle{\pm 0.2}$ & $\mathbf{93.9 \scriptstyle{\pm 0.1}}$ & $\underline{94.5 \scriptstyle{\pm 0.1}}$ \\
\midrule \midrule
\multirow{7}{*}{\rotatebox{90}{$\epsilon=0.15$}}
& EDL    & $93.5 \scriptstyle{\pm 0.3}$ & $\underline{9.2 \scriptstyle{\pm 0.3}}$ & $\underline{99.1 \scriptstyle{\pm 0.0}}$ & $61.7 \scriptstyle{\pm 1.7}$ & $26.8 \scriptstyle{\pm 1.2}$ & $19.1 \scriptstyle{\pm 0.9}$ & $59.1 \scriptstyle{\pm 1.5}$ & $25.3 \scriptstyle{\pm 0.4}$ \\
& LbBnn  & $\underline{95.3 \scriptstyle{\pm 0.1}}$ & $54.5 \scriptstyle{\pm 0.1}$ & $98.8 \scriptstyle{\pm 0.1}$ & $57.8 \scriptstyle{\pm 2.8}$ & $78.8 \scriptstyle{\pm 0.3}$ & $72.7 \scriptstyle{\pm 0.3}$ & $69.6 \scriptstyle{\pm 1.3}$ & $42.5 \scriptstyle{\pm 1.4}$ \\
& Decali & $93.9 \scriptstyle{\pm 0.2}$ & $11.6 \scriptstyle{\pm 0.3}$ & $96.2 \scriptstyle{\pm 0.7}$ & $36.3 \scriptstyle{\pm 3.8}$ & $75.7 \scriptstyle{\pm 0.5}$ & $\underline{8.9 \scriptstyle{\pm 0.3}}$ & $83.6 \scriptstyle{\pm 0.4}$ & $84.1 \scriptstyle{\pm 0.8}$ \\
& DAPPr  & $94.0 \scriptstyle{\pm 0.3}$ & $15.4 \scriptstyle{\pm 0.2}$ & $\underline{99.1 \scriptstyle{\pm 0.1}}$ & $65.7 \scriptstyle{\pm 4.2}$ & $74.8 \scriptstyle{\pm 0.5}$ & $38.6 \scriptstyle{\pm 0.4}$ & $91.5 \scriptstyle{\pm 0.3}$ & $77.1 \scriptstyle{\pm 0.4}$ \\
& \cellcolor{OxfordBlueLight}Ours   & \cellcolor{OxfordBlueLight}$\mathbf{95.5 \scriptstyle{\pm 0.1}}$ & \cellcolor{OxfordBlueLight}$\mathbf{6.0 \scriptstyle{\pm 0.2}}$ & \cellcolor{OxfordBlueLight}$98.8 \scriptstyle{\pm 0.2}$ & \cellcolor{OxfordBlueLight}$\mathbf{92.6 \scriptstyle{\pm 2.8}}$ & \cellcolor{OxfordBlueLight}$\underline{79.4 \scriptstyle{\pm 0.1}}$ & \cellcolor{OxfordBlueLight}$\mathbf{2.0 \scriptstyle{\pm 0.2}}$ & \cellcolor{OxfordBlueLight}$\underline{93.1 \scriptstyle{\pm 0.1}}$ & \cellcolor{OxfordBlueLight}$\mathbf{98.7 \scriptstyle{\pm 0.2}}$ \\
\cmidrule{2-10}
& MCDO   & $93.7 \scriptstyle{\pm 0.4}$ & $12.2 \scriptstyle{\pm 0.2}$ & $99.0 \scriptstyle{\pm 0.1}$ & $62.8 \scriptstyle{\pm 4.6}$ & $76.1 \scriptstyle{\pm 0.2}$ & $10.9 \scriptstyle{\pm 0.2}$ & $91.6 \scriptstyle{\pm 0.2}$ & $91.1 \scriptstyle{\pm 0.3}$ \\
& DE     & $95.2 \scriptstyle{\pm 0.1}$ & $14.3 \scriptstyle{\pm 0.1}$ & $\mathbf{99.4 \scriptstyle{\pm 0.0}}$ & $\underline{89.6 \scriptstyle{\pm 2.4}}$ & $\mathbf{80.2 \scriptstyle{\pm 0.2}}$ & $17.1 \scriptstyle{\pm 0.3}$ & $\mathbf{93.9 \scriptstyle{\pm 0.1}}$ & $\underline{92.9 \scriptstyle{\pm 0.2}}$ \\
\midrule \midrule
\multirow{7}{*}{\rotatebox{90}{$\epsilon=0.2$}}
& EDL    & $93.3 \scriptstyle{\pm 0.4}$ & $\underline{9.2 \scriptstyle{\pm 0.1}}$ & $\underline{99.1 \scriptstyle{\pm 0.1}}$ & $63.8 \scriptstyle{\pm 2.3}$ & $26.4 \scriptstyle{\pm 0.9}$ & $18.8 \scriptstyle{\pm 0.7}$ & $58.7 \scriptstyle{\pm 1.1}$ & $25.8 \scriptstyle{\pm 0.4}$ \\
& LbBnn  & $\mathbf{95.4 \scriptstyle{\pm 0.1}}$ & $61.0 \scriptstyle{\pm 0.1}$ & $98.7 \scriptstyle{\pm 0.1}$ & $58.8 \scriptstyle{\pm 2.0}$ & $78.7 \scriptstyle{\pm 0.3}$ & $74.2 \scriptstyle{\pm 0.3}$ & $70.6 \scriptstyle{\pm 1.7}$ & $43.7 \scriptstyle{\pm 1.4}$ \\
& Decali & $93.9 \scriptstyle{\pm 0.2}$ & $16.0 \scriptstyle{\pm 0.2}$ & $95.5 \scriptstyle{\pm 0.7}$ & $37.8 \scriptstyle{\pm 3.0}$ & $75.4 \scriptstyle{\pm 0.3}$ & $\underline{13.5 \scriptstyle{\pm 0.3}}$ & $80.3 \scriptstyle{\pm 0.7}$ & $79.2 \scriptstyle{\pm 1.0}$ \\
& DAPPr  & $93.9 \scriptstyle{\pm 0.2}$ & $18.9 \scriptstyle{\pm 0.2}$ & $99.0 \scriptstyle{\pm 0.1}$ & $65.3 \scriptstyle{\pm 2.5}$ & $75.0 \scriptstyle{\pm 0.4}$ & $40.8 \scriptstyle{\pm 0.3}$ & $91.4 \scriptstyle{\pm 0.3}$ & $76.6 \scriptstyle{\pm 0.5}$ \\
& \cellcolor{OxfordBlueLight}Ours   & \cellcolor{OxfordBlueLight}$\mathbf{95.4 \scriptstyle{\pm 0.1}}$ & \cellcolor{OxfordBlueLight}$\mathbf{7.6 \scriptstyle{\pm 0.1}}$ & \cellcolor{OxfordBlueLight}$98.7 \scriptstyle{\pm 0.1}$ & \cellcolor{OxfordBlueLight}$\mathbf{91.9 \scriptstyle{\pm 2.5}}$ & \cellcolor{OxfordBlueLight}$\underline{79.3 \scriptstyle{\pm 0.3}}$ & \cellcolor{OxfordBlueLight}$\mathbf{2.7 \scriptstyle{\pm 0.3}}$ & \cellcolor{OxfordBlueLight}$\underline{92.8 \scriptstyle{\pm 0.2}}$ & \cellcolor{OxfordBlueLight}$\mathbf{98.4 \scriptstyle{\pm 0.3}}$ \\
\cmidrule{2-10}
& MCDO   & $93.7 \scriptstyle{\pm 0.4}$ & $16.5 \scriptstyle{\pm 0.2}$ & $99.0 \scriptstyle{\pm 0.1}$ & $64.0 \scriptstyle{\pm 6.8}$ & $75.6 \scriptstyle{\pm 0.3}$ & $14.6 \scriptstyle{\pm 0.3}$ & $91.4 \scriptstyle{\pm 0.2}$ & $89.3 \scriptstyle{\pm 0.4}$ \\
& DE     & $95.2 \scriptstyle{\pm 0.1}$ & $18.5 \scriptstyle{\pm 0.1}$ & $\mathbf{99.4 \scriptstyle{\pm 0.0}}$ & $\underline{89.3 \scriptstyle{\pm 2.1}}$ & $\mathbf{80.1 \scriptstyle{\pm 0.2}}$ & $21.3 \scriptstyle{\pm 0.3}$ & $\mathbf{93.9 \scriptstyle{\pm 0.1}}$ & $\underline{91.3 \scriptstyle{\pm 0.1}}$ \\
\bottomrule
\end{tabular}
\end{table}
\begin{figure}[!htbp]
\centering
\includegraphics[width=\linewidth]{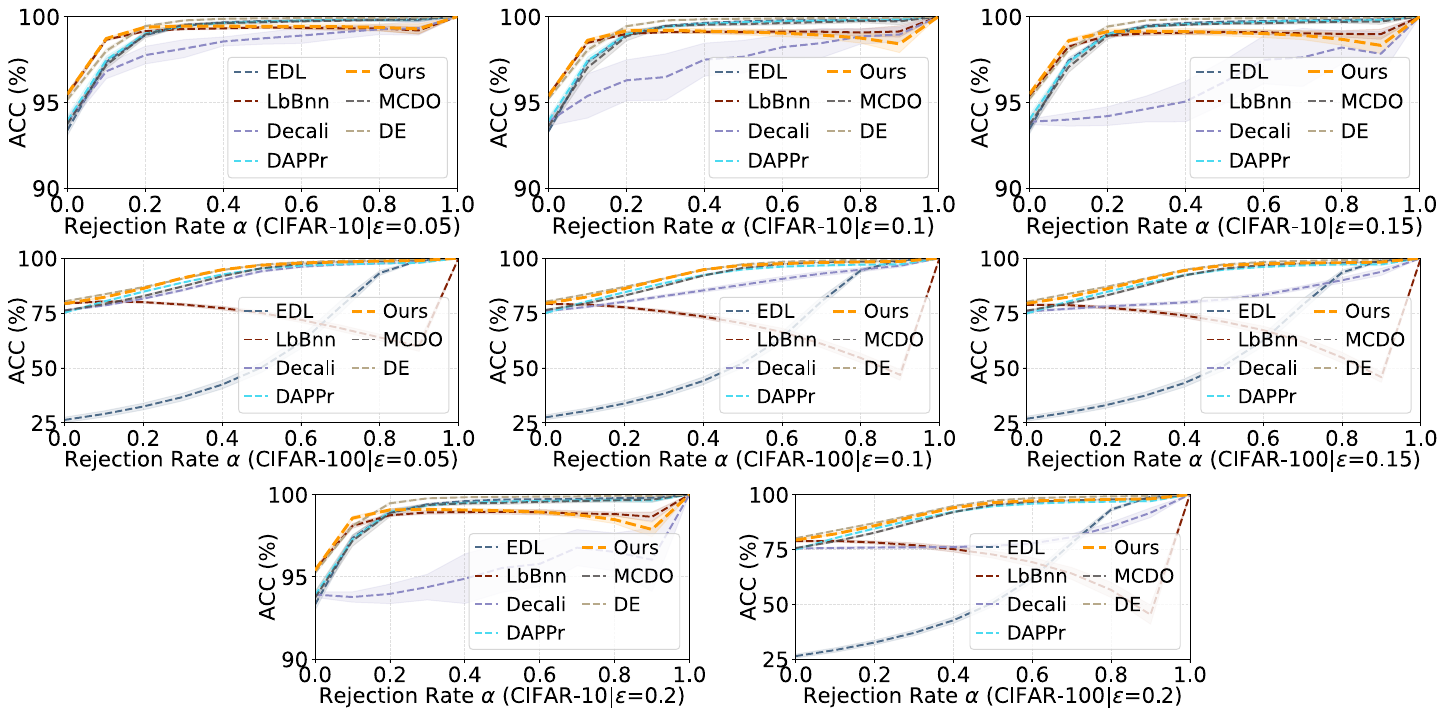}
\caption{Accuracy rejection curves under label-smoothing annotation imprecision with varying $\epsilon$.}
\label{Figure: arc_label}
\end{figure}
\begin{table}[!htbp]
\caption{Performance comparison (in $\%$)  using the teachers' original predictions ($T=1.0$) as the source of annotation imprecision, averaged over 10 runs. These results consistently confirm the robustness of POCC in achieving the highest BQS scores across settings.}
\label{Table: AblationTeacher_original}
\centering
\small
\setlength\tabcolsep{4pt}
\begin{tabular}{lccc|c||ccc|c}
\toprule
\multirow{2}{*}{Method} & \multicolumn{4}{c||}{CIFAR-10} & \multicolumn{4}{c}{CIFAR-100} \\
\cmidrule(lr){2-5} \cmidrule(lr){6-9}
& ACC $\uparrow$ & ECE $\downarrow$ & AUARC $\uparrow$ & BQS $\uparrow$ & ACC $\uparrow$ & ECE $\downarrow$ & AUARC $\uparrow$ & BQS $\uparrow$ \\
\midrule \midrule
\multicolumn{9}{c}{Teacher: ResNet-20}\\ \midrule\midrule
EDL    & $93.3 \scriptstyle{\pm 0.2}$ & $9.1 \scriptstyle{\pm 0.3}$ & $99.1 \scriptstyle{\pm 0.1}$ & $1.2 \scriptstyle{\pm 2.4}$ & $33.7 \scriptstyle{\pm 5.0}$ & $23.7 \scriptstyle{\pm 3.4}$ & $66.5 \scriptstyle{\pm 5.5}$ & $14.3 \scriptstyle{\pm 3.1}$ \\
LbBnn  & $\underline{95.4 \scriptstyle{\pm 0.1}}$ & $1.5 \scriptstyle{\pm 0.1}$ & $\underline{99.5 \scriptstyle{\pm 0.0}}$ & $\underline{91.4 \scriptstyle{\pm 2.7}}$ & $\underline{77.3 \scriptstyle{\pm 0.3}}$ & $19.5 \scriptstyle{\pm 0.5}$ & $90.3 \scriptstyle{\pm 0.2}$ & $81.3 \scriptstyle{\pm 0.7}$ \\
Decali & $93.8 \scriptstyle{\pm 0.2}$ & $1.5 \scriptstyle{\pm 0.1}$ & $99.2 \scriptstyle{\pm 0.0}$ & $42.2 \scriptstyle{\pm 8.0}$ & $74.3 \scriptstyle{\pm 0.4}$ & $\mathbf{2.4 \scriptstyle{\pm 0.3}}$ & $89.1 \scriptstyle{\pm 0.1}$ & $92.9 \scriptstyle{\pm 1.2}$ \\
DAPPr  & $94.0 \scriptstyle{\pm 0.2}$ & $7.0 \scriptstyle{\pm 0.1}$ & $99.2 \scriptstyle{\pm 0.1}$ & $27.2 \scriptstyle{\pm 7.4}$ & $75.4 \scriptstyle{\pm 0.3}$ & $39.7 \scriptstyle{\pm 0.3}$ & $90.1 \scriptstyle{\pm 0.3}$ & $61.6 \scriptstyle{\pm 1.2}$ \\
\rowcolor{OxfordBlueLight}
Ours   & $\mathbf{95.5 \scriptstyle{\pm 0.1}}$ & $1.8 \scriptstyle{\pm 0.1}$ & $\mathbf{99.5 \scriptstyle{\pm 0.0}}$ & $\mathbf{94.5 \scriptstyle{\pm 2.1}}$ & $\mathbf{77.9 \scriptstyle{\pm 0.2}}$ & $3.9 \scriptstyle{\pm 0.4}$ & $\mathbf{92.4 \scriptstyle{\pm 0.1}}$ & $\mathbf{98.6 \scriptstyle{\pm 0.4}}$ \\
\midrule
MCDO   & $93.8 \scriptstyle{\pm 0.1}$ & $\underline{1.1 \scriptstyle{\pm 0.1}}$ & $99.2 \scriptstyle{\pm 0.0}$ & $48.2 \scriptstyle{\pm 7.6}$ & $73.8 \scriptstyle{\pm 0.4}$ & $\underline{2.9 \scriptstyle{\pm 0.3}}$ & $90.3 \scriptstyle{\pm 0.2}$ & $93.5 \scriptstyle{\pm 1.2}$ \\
DE     & $94.8 \scriptstyle{\pm 0.0}$ & $\mathbf{0.7 \scriptstyle{\pm 0.1}}$ & $99.4 \scriptstyle{\pm 0.0}$ & $77.9 \scriptstyle{\pm 3.9}$ & $76.2 \scriptstyle{\pm 0.2}$ & $5.9 \scriptstyle{\pm 0.2}$ & $\underline{91.7 \scriptstyle{\pm 0.1}}$ & $\underline{94.6 \scriptstyle{\pm 0.7}}$ \\
\midrule \midrule
\multicolumn{9}{c}{Teacher: ShuffleNetV2-0.5×}\\ \midrule\midrule 
EDL    & $92.7 \scriptstyle{\pm 0.1}$ & $9.1 \scriptstyle{\pm 0.2}$ & $99.0 \scriptstyle{\pm 0.0}$ & $3.0 \scriptstyle{\pm 2.1}$ & $32.7 \scriptstyle{\pm 5.9}$ & $23.1 \scriptstyle{\pm 4.2}$ & $65.6 \scriptstyle{\pm 6.4}$ & $14.1 \scriptstyle{\pm 3.8}$ \\
LbBnn  & $\underline{95.0 \scriptstyle{\pm 0.2}}$ & $3.1 \scriptstyle{\pm 0.1}$ & $\underline{99.3 \scriptstyle{\pm 0.0}}$ & $\underline{76.2 \scriptstyle{\pm 3.1}}$ & $\underline{77.7 \scriptstyle{\pm 0.2}}$ & $21.7 \scriptstyle{\pm 0.4}$ & $90.8 \scriptstyle{\pm 0.1}$ & $79.0 \scriptstyle{\pm 0.9}$ \\
Decali & $93.2 \scriptstyle{\pm 0.2}$ & $\mathbf{0.6 \scriptstyle{\pm 0.1}}$ & $98.9 \scriptstyle{\pm 0.0}$ & $40.9 \scriptstyle{\pm 3.1}$ & $74.8 \scriptstyle{\pm 0.3}$ & $\mathbf{2.1 \scriptstyle{\pm 0.3}}$ & $90.4 \scriptstyle{\pm 0.1}$ & $94.0 \scriptstyle{\pm 1.2}$ \\
DAPPr  & $93.6 \scriptstyle{\pm 0.2}$ & $8.9 \scriptstyle{\pm 0.2}$ & $99.2 \scriptstyle{\pm 0.0}$ & $28.4 \scriptstyle{\pm 5.4}$ & $75.6 \scriptstyle{\pm 0.3}$ & $38.5 \scriptstyle{\pm 0.2}$ & $91.0 \scriptstyle{\pm 0.1}$ & $62.3 \scriptstyle{\pm 0.8}$ \\
\rowcolor{OxfordBlueLight}
Ours   & $\mathbf{95.0 \scriptstyle{\pm 0.1}}$ & $\underline{0.8 \scriptstyle{\pm 0.1}}$ & $\mathbf{99.5 \scriptstyle{\pm 0.0}}$ & $\mathbf{98.4 \scriptstyle{\pm 1.5}}$ & $\mathbf{78.4 \scriptstyle{\pm 0.1}}$ & $3.8 \scriptstyle{\pm 0.2}$ & $\mathbf{92.8 \scriptstyle{\pm 0.1}}$ & $\mathbf{98.3 \scriptstyle{\pm 0.4}}$ \\
\midrule
MCDO   & $93.2 \scriptstyle{\pm 0.3}$ & $0.9 \scriptstyle{\pm 0.1}$ & $99.0 \scriptstyle{\pm 0.1}$ & $48.0 \scriptstyle{\pm 6.1}$ & $74.2 \scriptstyle{\pm 0.4}$ & $\underline{2.3 \scriptstyle{\pm 0.4}}$ & $90.8 \scriptstyle{\pm 0.3}$ & $94.0 \scriptstyle{\pm 1.6}$ \\
DE     & $94.0 \scriptstyle{\pm 0.1}$ & $1.9 \scriptstyle{\pm 0.1}$ & $99.2 \scriptstyle{\pm 0.0}$ & $65.5 \scriptstyle{\pm 2.3}$ & $77.1 \scriptstyle{\pm 0.2}$ & $5.9 \scriptstyle{\pm 0.2}$ & $\underline{92.3 \scriptstyle{\pm 0.1}}$ & $\underline{94.8 \scriptstyle{\pm 0.5}}$ \\
\midrule \midrule
\multicolumn{9}{c}{Teacher: MobileNetV2-x0-5}\\ \midrule\midrule
EDL    & $93.3 \scriptstyle{\pm 0.2}$ & $9.2 \scriptstyle{\pm 0.2}$ & $99.1 \scriptstyle{\pm 0.1}$ & $2.7 \scriptstyle{\pm 3.4}$ & $38.3 \scriptstyle{\pm 1.6}$ & $27.0 \scriptstyle{\pm 1.1}$ & $71.7 \scriptstyle{\pm 1.4}$ & $9.2 \scriptstyle{\pm 1.1}$ \\
LbBnn  & $\mathbf{95.5 \scriptstyle{\pm 0.2}}$ & $1.6 \scriptstyle{\pm 0.1}$ & $\underline{99.5 \scriptstyle{\pm 0.0}}$ & $\underline{91.0 \scriptstyle{\pm 3.2}}$ & $\mathbf{79.2 \scriptstyle{\pm 0.3}}$ & $20.9 \scriptstyle{\pm 0.6}$ & $91.9 \scriptstyle{\pm 0.1}$ & $79.2 \scriptstyle{\pm 0.5}$ \\
Decali & $93.8 \scriptstyle{\pm 0.2}$ & $1.4 \scriptstyle{\pm 0.1}$ & $99.1 \scriptstyle{\pm 0.0}$ & $42.8 \scriptstyle{\pm 8.4}$ & $76.2 \scriptstyle{\pm 0.2}$ & $\mathbf{2.3 \scriptstyle{\pm 0.2}}$ & $91.7 \scriptstyle{\pm 0.1}$ & $94.4 \scriptstyle{\pm 0.3}$ \\
DAPPr  & $93.9 \scriptstyle{\pm 0.2}$ & $7.2 \scriptstyle{\pm 0.1}$ & $99.2 \scriptstyle{\pm 0.0}$ & $27.7 \scriptstyle{\pm 5.7}$ & $76.1 \scriptstyle{\pm 0.3}$ & $36.6 \scriptstyle{\pm 0.3}$ & $91.8 \scriptstyle{\pm 0.2}$ & $61.2 \scriptstyle{\pm 0.6}$ \\
\rowcolor{OxfordBlueLight}
Ours   & $\underline{95.5 \scriptstyle{\pm 0.2}}$ & $1.7 \scriptstyle{\pm 0.1}$ & $\mathbf{99.5 \scriptstyle{\pm 0.0}}$ & $\mathbf{94.6 \scriptstyle{\pm 3.6}}$ & $\underline{79.1 \scriptstyle{\pm 0.3}}$ & $3.6 \scriptstyle{\pm 0.3}$ & $\mathbf{93.7 \scriptstyle{\pm 0.1}}$ & $\mathbf{98.4 \scriptstyle{\pm 0.5}}$ \\
\midrule
MCDO   & $93.6 \scriptstyle{\pm 0.2}$ & $\underline{1.0 \scriptstyle{\pm 0.2}}$ & $99.2 \scriptstyle{\pm 0.1}$ & $47.7 \scriptstyle{\pm 6.5}$ & $75.4 \scriptstyle{\pm 0.4}$ & $\underline{2.4 \scriptstyle{\pm 0.2}}$ & $91.6 \scriptstyle{\pm 0.3}$ & $93.5 \scriptstyle{\pm 0.9}$ \\
DE     & $94.8 \scriptstyle{\pm 0.1}$ & $\mathbf{0.8 \scriptstyle{\pm 0.1}}$ & $99.4 \scriptstyle{\pm 0.0}$ & $77.8 \scriptstyle{\pm 2.4}$ & $78.9 \scriptstyle{\pm 0.3}$ & $5.2 \scriptstyle{\pm 0.2}$ & $\underline{93.4 \scriptstyle{\pm 0.1}}$ & $\underline{96.4 \scriptstyle{\pm 0.3}}$ \\
\bottomrule
\end{tabular}
\end{table}

\newpage
\section{Evaluation on Statistical Significance Tests}
\label{App:StatisticalTest}
\textbf{Setup.}
In this section, we conduct pairwise one-sided Wilcoxon signed-rank tests at the \(5\%\) significance level across three evaluation criteria: accuracy, \(1-\mathrm{ECE}\), and AUARC. For each ordered pair of methods \((m_i, m_j)\), \(i \neq j\), the null hypothesis \(H_0\) assumes no systematic performance difference between \(m_i\) and \(m_j\), while the alternative hypothesis \(H_1\) posits that \(m_i\) stochastically outperforms \(m_j\). Tests are performed over 10 independent runs, and results are deemed statistically significant when \(p < 0.05\), in which case \(m_i\) is considered to outperform \(m_j\).

To provide an interpretable summary of the statistical comparisons, we construct a ranking scheme as follows. For each method and each evaluation criterion, we count the number of significant wins and losses across all pairwise comparisons: a significant win contributes \(+1\), a significant loss contributes \(-1\), and a non-significant outcome contributes \(0\). The net score for method \(m\) is defined as \(\text{wins}(m) - \text{losses}(m)\), reflecting its relative dominance. To obtain a holistic measure of balanced performance, we min-max normalise the net scores of the three criteria and take their average, denoted as \emph{normalised balanced quality score for statistical test} (BQS-ST), reflecting equal weighting across accuracy, calibration, and EU quantification.

\textbf{Results.}
Consistent with the main evaluation settings in Section~\ref{sec: Exp}, we evaluate our POCC method across three annotation imprecision settings: annotator disagreement, teacher prediction, and label smoothing. Tables~\ref{Table: ST_disagreement}, \ref{Table: ST-Teacher}, and \ref{Table: ST-smoothing} report rankings derived from pairwise statistical comparisons, demonstrating that \emph{POCC consistently achieves the best trade-off performance}---as measured by the highest BQS-ST scores---across prediction accuracy, calibration, and EU quantification. The detailed statistical significance analyses for the three evaluation scenarios are presented in Figures~\ref{Figure: ST-Disagreement}, \ref{Figure: ST-teacher}, and \ref{Figure: ST-smoothing}, respectively.

\begin{table}[!htbp]
\caption{Performance comparison (in $\%$, $\uparrow$) under annotator disagreements.}
\label{Table: ST_disagreement}
\centering
\small
\setlength\tabcolsep{1.5pt}
\begin{tabular}{lccc|c||ccc|c||ccc|c}
\toprule
\multirow{2}{*}{Method} & \multicolumn{4}{c||}{CIFAR-10H} & \multicolumn{4}{c||}{MiceBone} & \multicolumn{4}{c}{TreeVersity} \\
\cmidrule(lr){2-5} \cmidrule(lr){6-9} \cmidrule(lr){10-13}
& ACC & 1-ECE & AUARC & BQS-ST & ACC & 1-ECE & AUARC & BQS-ST & ACC & 1-ECE & AUARC & BQS-ST\\ 
\midrule
EDL    & $-5.0$ & $-6.0$ & $-1.0$ & $15.2$ & $-1.0$ & $-2.0$ & $\mathbf{3.0}$ & $\underline{58.3}$ & $-2.0$ & $-3.0$ & $\underline{1.0}$ & $30.4$ \\
LbBnn  & $\underline{3.0}$ & $-2.0$ & $-4.0$ & $47.8$ & -- & -- & -- & -- & -- & -- & -- & -- \\
Decali & $-5.0$ & $\mathbf{5.0}$ & $-6.0$ & $33.3$ & $\underline{1.0}$ & $\underline{2.0}$ & $-4.0$ & $52.8$ & $\mathbf{3.0}$ & $\mathbf{3.0}$ & $-2.0$ & $\underline{75.0}$ \\
DAPPr  & $1.0$ & $-4.0$ & $\mathbf{5.0}$ & $61.6$ & $\mathbf{2.0}$ & $-4.0$ & $\underline{0.0}$ & $52.4$ & $\mathbf{3.0}$ & $-3.0$ & $\underline{1.0}$ & $54.2$ \\
\rowcolor{OxfordBlueLight}
Ours   & $\mathbf{4.0}$ & $\underline{2.0}$ & $\mathbf{5.0}$ & $\mathbf{90.9}$ & $\mathbf{2.0}$ & $\mathbf{4.0}$ & $\mathbf{3.0}$ & $\mathbf{100.0}$ & $\underline{0.0}$ & $\mathbf{3.0}$ & $\mathbf{4.0}$ & $\mathbf{85.7}$ \\
\midrule
MCDO   & $-2.0$ & $\mathbf{5.0}$ & $-1.0$ & $59.6$ & -- & -- & -- & -- & -- & -- & -- & -- \\
DE     & $\mathbf{4.0}$ & $0.0$ & $\underline{2.0}$ & $\underline{75.8}$ & $-4.0$ & $0.0$ & $-2.0$ & $26.2$ & $-4.0$ & $\underline{0.0}$ & $-4.0$ & $16.7$ \\
\bottomrule
\end{tabular}
\end{table}

\begin{table}[!htbp]
\caption{Performance comparison (in $\%$, $\uparrow$) under teacher-prediction annotation imprecision.}
\label{Table: ST-Teacher}
\centering
\small
\begin{tabular}{cl ccc|c||ccc|c}
\toprule
\multirow{2}{*}{} & \multirow{2}{*}{Method} & \multicolumn{4}{c||}{CIFAR-10} & \multicolumn{4}{c}{CIFAR-100} \\
\cmidrule(lr){3-6} \cmidrule(lr){7-10}
& & ACC & 1-ECE & AUARC & BQS-ST & ACC & 1-ECE & AUARC & BQS-ST \\
\midrule\midrule
\multirow{7}{*}{\rotatebox[origin=c]{90}{ResNet-20}}
& EDL    & $-3.0$ & $0.0$ & $0.0$ & $33.3$ & $-6.0$ & $\underline{4.0}$ & $-5.0$ & $27.8$ \\
& LbBnn  & $\underline{4.0}$ & $-6.0$ & $\underline{4.0}$ & $\underline{53.7}$ & $\underline{4.0}$ & $-6.0$ & $-5.0$ & $27.8$ \\
& Decali & $-3.0$ & $\underline{3.0}$ & $-6.0$ & $25.0$ & $-3.0$ & $2.0$ & $-2.0$ & $39.6$ \\
& DAPPr  & $-3.0$ & $-4.0$ & $-3.0$ & $13.9$ & $1.0$ & $-4.0$ & $0.0$ & $40.2$ \\
& \cellcolor{OxfordBlueLight}Ours & \cellcolor{OxfordBlueLight}$\mathbf{6.0}$ & \cellcolor{OxfordBlueLight}$\mathbf{6.0}$ & \cellcolor{OxfordBlueLight}$\mathbf{6.0}$ & \cellcolor{OxfordBlueLight}$\mathbf{100.0}$ & \cellcolor{OxfordBlueLight}$\mathbf{6.0}$ & \cellcolor{OxfordBlueLight}$\mathbf{6.0}$ & \cellcolor{OxfordBlueLight}$\mathbf{6.0}$ & \cellcolor{OxfordBlueLight}$\mathbf{100.0}$ \\
\cmidrule(l){2-10}
& MCDO   & $-3.0$ & $\underline{3.0}$ & $-3.0$ & $33.3$ & $-3.0$ & $0.0$ & $2.0$ & $46.2$ \\
& DE     & $2.0$ & $-2.0$ & $2.0$ & $51.9$ & $1.0$ & $-2.0$ & $\underline{4.0}$ & $\underline{57.8}$ \\
\midrule\midrule
\multirow{7}{*}{\rotatebox[origin=c]{90}{ShuffleNetV2-0.5x}}
& EDL    & $1.0$ & $\underline{4.0}$ & $\underline{4.0}$ & $\underline{73.7}$ & $-6.0$ & $\underline{4.0}$ & $-5.0$ & $27.8$ \\
& LbBnn  & $\underline{4.0}$ & $-6.0$ & $-4.0$ & $32.8$ & $1.0$ & $-6.0$ & $-5.0$ & $19.4$ \\
& Decali & $-5.0$ & $1.0$ & $-6.0$ & $19.4$ & $-3.0$ & $2.0$ & $-2.0$ & $39.6$ \\
& DAPPr  & $-2.0$ & $-4.0$ & $1.0$ & $34.1$ & $1.0$ & $-4.0$ & $0.0$ & $40.2$ \\
& \cellcolor{OxfordBlueLight}Ours & \cellcolor{OxfordBlueLight}$\mathbf{6.0}$ & \cellcolor{OxfordBlueLight}$\mathbf{6.0}$ & \cellcolor{OxfordBlueLight}$\mathbf{6.0}$ & \cellcolor{OxfordBlueLight}$\mathbf{100.0}$ & \cellcolor{OxfordBlueLight}$\mathbf{6.0}$ & \cellcolor{OxfordBlueLight}$\mathbf{6.0}$ & \cellcolor{OxfordBlueLight}$\mathbf{6.0}$ & \cellcolor{OxfordBlueLight}$\mathbf{100.0}$ \\
\cmidrule(l){2-10}
& MCDO   & $-5.0$ & $1.0$ & $-2.0$ & $30.6$ & $-3.0$ & $0.0$ & $2.0$ & $46.2$ \\
& DE     & $1.0$ & $-2.0$ & $1.0$ & $48.7$ & $\underline{4.0}$ & $-2.0$ & $\underline{4.0}$ & $\underline{66.2}$ \\
\midrule\midrule
\multirow{7}{*}{\rotatebox[origin=c]{90}{MobileNetV2-x0-5}}
& EDL    & $-6.0$ & $\underline{4.0}$ & $2.0$ & $50.0$ & $-6.0$ & $\underline{2.0}$ & $-4.0$ & $27.8$ \\
& LbBnn  & $\underline{4.0}$ & $-6.0$ & $-3.0$ & $36.1$ & $2.0$ & $-6.0$ & $-6.0$ & $22.2$ \\
& Decali & $-1.0$ & $1.0$ & $-6.0$ & $33.3$ & $-2.0$ & $\underline{2.0}$ & $-2.0$ & $44.4$ \\
& DAPPr  & $-1.0$ & $-4.0$ & $-1.0$ & $33.3$ & $0.0$ & $-4.0$ & $0.0$ & $38.9$ \\
& \cellcolor{OxfordBlueLight}Ours & \cellcolor{OxfordBlueLight}$\mathbf{6.0}$ & \cellcolor{OxfordBlueLight}$\mathbf{6.0}$ & \cellcolor{OxfordBlueLight}$\mathbf{6.0}$ & \cellcolor{OxfordBlueLight}$\mathbf{100.0}$ & \cellcolor{OxfordBlueLight}$\mathbf{6.0}$ & \cellcolor{OxfordBlueLight}$\mathbf{6.0}$ & \cellcolor{OxfordBlueLight}$\mathbf{6.0}$ & \cellcolor{OxfordBlueLight}$\mathbf{100.0}$ \\
\cmidrule(l){2-10}
& MCDO   & $-4.0$ & $1.0$ & $-2.0$ & $36.1$ & $-4.0$ & $\underline{2.0}$ & $2.0$ & $50.0$ \\
& DE     & $2.0$ & $-2.0$ & $\underline{4.0}$ & $\underline{61.1}$ & $\underline{4.0}$ & $-2.0$ & $\underline{4.0}$ & $\underline{66.7}$ \\
\bottomrule
\end{tabular}
\end{table}
\begin{table}[!htbp]
\caption{Performance comparison (in $\%$, $\uparrow$) under label-smoothing annotation imprecision.}
\label{Table: ST-smoothing}
\centering
\small
\begin{tabular}{cl ccc|c||ccc|c}
\toprule
\multirow{2}{*}{} & \multirow{2}{*}{Method} & \multicolumn{4}{c||}{CIFAR-10} & \multicolumn{4}{c}{CIFAR-100} \\
\cmidrule(lr){3-6} \cmidrule(lr){7-10}
& & ACC & 1-ECE & AUARC & BQS-ST & ACC & 1-ECE & AUARC & BQS-ST \\
\midrule\midrule
\multirow{7}{*}{\rotatebox[origin=c]{90}{$\epsilon=0.05$}}
& EDL    & $-6.0$ & $-4.0$ & $-1.0$ & $19.4$ & $-6.0$ & $-2.0$ & $-6.0$ & $11.1$ \\
& LbBnn  & $\mathbf{5.0}$ & $-6.0$ & $-2.0$ & $44.4$ & $\underline{3.0}$ & $-6.0$ & $-4.0$ & $30.6$ \\
& Decali & $-1.0$ & $\underline{3.0}$ & $-6.0$ & $40.2$ & $-1.0$ & $\underline{3.0}$ & $-2.0$ & $50.0$ \\
& DAPPr  & $-2.0$ & $-2.0$ & $\underline{2.0}$ & $45.5$ & $-4.0$ & $-4.0$ & $1.0$ & $30.6$ \\
& \cellcolor{OxfordBlueLight}Ours & \cellcolor{OxfordBlueLight}$\mathbf{5.0}$ & \cellcolor{OxfordBlueLight}$\mathbf{6.0}$ & \cellcolor{OxfordBlueLight}$1.0$ & \cellcolor{OxfordBlueLight}$\mathbf{86.1}$ & \cellcolor{OxfordBlueLight}$\underline{3.0}$ & \cellcolor{OxfordBlueLight}$\mathbf{6.0}$ & \cellcolor{OxfordBlueLight}$\underline{4.0}$ & \cellcolor{OxfordBlueLight}$\mathbf{86.1}$ \\
\cmidrule(l){2-10}
& MCDO   & $-3.0$ & $\underline{3.0}$ & $0.0$ & $50.8$ & $-1.0$ & $\underline{3.0}$ & $1.0$ & $58.3$ \\
& DE     & $\underline{2.0}$ & $0.0$ & $\mathbf{6.0}$ & $\underline{74.2}$ & $\mathbf{6.0}$ & $0.0$ & $\mathbf{6.0}$ & $\underline{83.3}$ \\
\midrule\midrule
\multirow{7}{*}{\rotatebox[origin=c]{90}{$\epsilon=0.1$}}
& EDL    & $-6.0$ & $0.0$ & $\underline{3.0}$ & $41.7$ & $-6.0$ & $-2.0$ & $-6.0$ & $11.1$ \\
& LbBnn  & $\mathbf{5.0}$ & $-6.0$ & $-3.0$ & $41.7$ & $2.0$ & $-6.0$ & $-4.0$ & $27.8$ \\
& Decali & $-1.0$ & $\underline{4.0}$ & $-6.0$ & $42.9$ & $-2.0$ & $\underline{4.0}$ & $-2.0$ & $50.0$ \\
& DAPPr  & $-1.0$ & $-4.0$ & $\underline{3.0}$ & $45.7$ & $-4.0$ & $-4.0$ & $0.0$ & $27.8$ \\
& \cellcolor{OxfordBlueLight}Ours & \cellcolor{OxfordBlueLight}$\mathbf{5.0}$ & \cellcolor{OxfordBlueLight}$\mathbf{6.0}$ & \cellcolor{OxfordBlueLight}$-3.0$ & \cellcolor{OxfordBlueLight}$\mathbf{75.0}$ & \cellcolor{OxfordBlueLight}$\underline{4.0}$ & \cellcolor{OxfordBlueLight}$\mathbf{6.0}$ & \cellcolor{OxfordBlueLight}$\underline{4.0}$ & \cellcolor{OxfordBlueLight}$\mathbf{88.9}$ \\
\cmidrule(l){2-10}
& MCDO   & $-4.0$ & $2.0$ & $0.0$ & $44.9$ & $0.0$ & $2.0$ & $2.0$ & $61.1$ \\
& DE     & $\underline{2.0}$ & $-2.0$ & $\mathbf{6.0}$ & $\underline{68.7}$ & $\mathbf{6.0}$ & $0.0$ & $\mathbf{6.0}$ & $\underline{83.3}$ \\
\midrule\midrule
\multirow{7}{*}{\rotatebox[origin=c]{90}{$\epsilon=0.15$}}
& EDL    & $-6.0$ & $\underline{4.0}$ & $\underline{3.0}$ & $52.8$ & $-6.0$ & $-2.0$ & $-6.0$ & $11.1$ \\
& LbBnn  & $\underline{4.0}$ & $-6.0$ & $-3.0$ & $36.1$ & $2.0$ & $-6.0$ & $-4.0$ & $27.8$ \\
& Decali & $-2.0$ & $2.0$ & $-6.0$ & $33.3$ & $-1.0$ & $\underline{4.0}$ & $-2.0$ & $52.8$ \\
& DAPPr  & $-2.0$ & $-4.0$ & $\underline{3.0}$ & $41.7$ & $-4.0$ & $-4.0$ & $1.0$ & $30.6$ \\
& \cellcolor{OxfordBlueLight}Ours & \cellcolor{OxfordBlueLight}$\mathbf{6.0}$ & \cellcolor{OxfordBlueLight}$\mathbf{6.0}$ & \cellcolor{OxfordBlueLight}$-3.0$ & \cellcolor{OxfordBlueLight}$\mathbf{75.0}$ & \cellcolor{OxfordBlueLight}$\underline{4.0}$ & \cellcolor{OxfordBlueLight}$\mathbf{6.0}$ & \cellcolor{OxfordBlueLight}$\underline{4.0}$ & \cellcolor{OxfordBlueLight}$\mathbf{88.9}$ \\
\cmidrule(l){2-10}
& MCDO   & $-2.0$ & $0.0$ & $0.0$ & $44.4$ & $-1.0$ & $2.0$ & $1.0$ & $55.6$ \\
& DE     & $2.0$ & $-2.0$ & $\mathbf{6.0}$ & $\underline{66.7}$ & $\mathbf{6.0}$ & $0.0$ & $\mathbf{6.0}$ & $\underline{83.3}$ \\
\midrule\midrule
\multirow{7}{*}{\rotatebox[origin=c]{90}{$\epsilon=0.2$}}
& EDL    & $-6.0$ & $\underline{4.0}$ & $\underline{4.0}$ & $55.6$ & $-6.0$ & $0.0$ & $-6.0$ & $16.7$ \\
& LbBnn  & $\mathbf{5.0}$ & $-6.0$ & $-3.0$ & $41.7$ & $2.0$ & $-6.0$ & $-4.0$ & $27.8$ \\
& Decali & $-2.0$ & $2.0$ & $-6.0$ & $34.3$ & $-1.0$ & $\underline{4.0}$ & $-2.0$ & $52.8$ \\
& DAPPr  & $-2.0$ & $-4.0$ & $1.0$ & $37.1$ & $-4.0$ & $-4.0$ & $1.0$ & $30.6$ \\
& \cellcolor{OxfordBlueLight}Ours & \cellcolor{OxfordBlueLight}$\mathbf{5.0}$ & \cellcolor{OxfordBlueLight}$\mathbf{6.0}$ & \cellcolor{OxfordBlueLight}$-3.0$ & \cellcolor{OxfordBlueLight}$\mathbf{75.0}$ & \cellcolor{OxfordBlueLight}$\underline{4.0}$ & \cellcolor{OxfordBlueLight}$\mathbf{6.0}$ & \cellcolor{OxfordBlueLight}$\underline{4.0}$ & \cellcolor{OxfordBlueLight}$\mathbf{88.9}$ \\
\cmidrule(l){2-10}
& MCDO   & $-2.0$ & $0.0$ & $1.0$ & $48.2$ & $-1.0$ & $2.0$ & $1.0$ & $55.6$ \\
& DE     & $\underline{2.0}$ & $-2.0$ & $\mathbf{6.0}$ & $\underline{68.7}$ & $\mathbf{6.0}$ & $-2.0$ & $\mathbf{6.0}$ & $\underline{77.8}$ \\
\bottomrule
\end{tabular}
\end{table}

\begin{figure}[!htbp]
\centering
\includegraphics[width=0.95\linewidth]{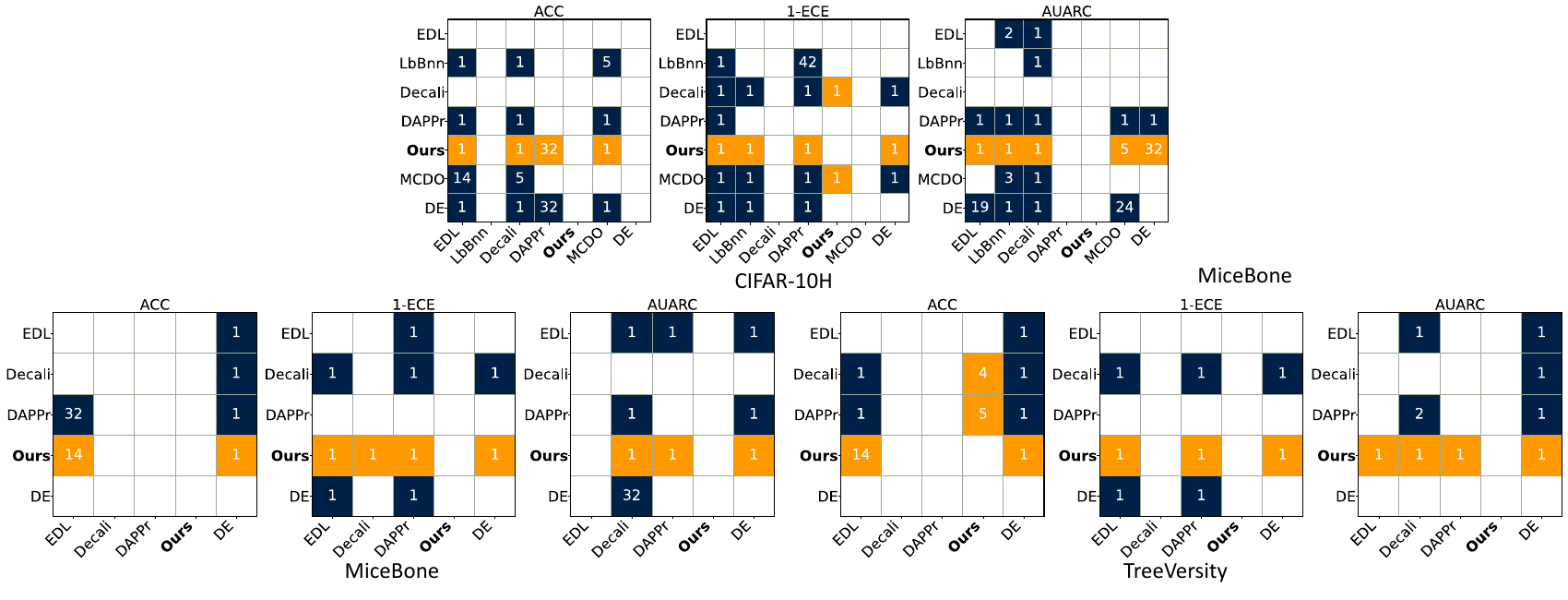}
\caption{Pairwise one-sided Wilcoxon signed-rank test results on various datasets under annotator disagreement. Cell $(i,j)$ indicates whether row method $i$ significantly outperforms column method
$j$; coloured cells show the scaled $p$-value ($\times 10^3$). Orange (\textcolor[HTML]{FF9900}{$\blacksquare$}) mark comparisons involving our method; dark navy (\textcolor[HTML]{002147}{$\blacksquare$}) mark all other significant pairs. }
\label{Figure: ST-Disagreement}
\end{figure}
\begin{figure}[!htbp]
\centering
\includegraphics[width=0.95\linewidth]{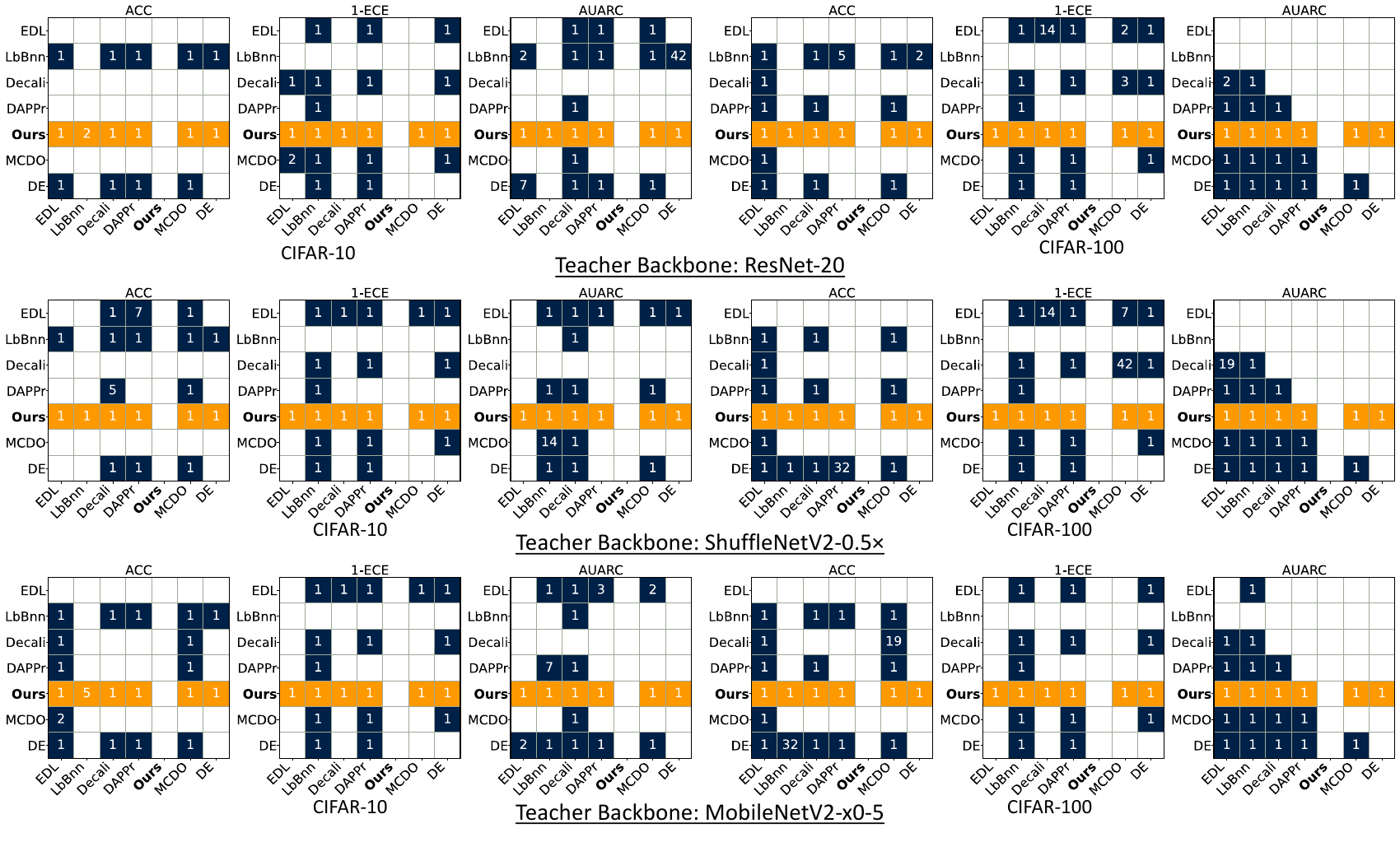}
\caption{Pairwise one-sided Wilcoxon signed-rank test results on various datasets under teacher-prediction imprecision. Cell $(i,j)$ indicates whether row method $i$ significantly outperforms column method $j$; coloured cells show the scaled $p$-value ($\times 10^3$). Orange (\textcolor[HTML]{FF9900}{$\blacksquare$}) mark comparisons involving our method; dark navy (\textcolor[HTML]{002147}{$\blacksquare$}) mark all other significant pairs. }
\label{Figure: ST-teacher}
\end{figure}

\begin{figure}[!htbp]
\centering
\includegraphics[width=0.95\linewidth]{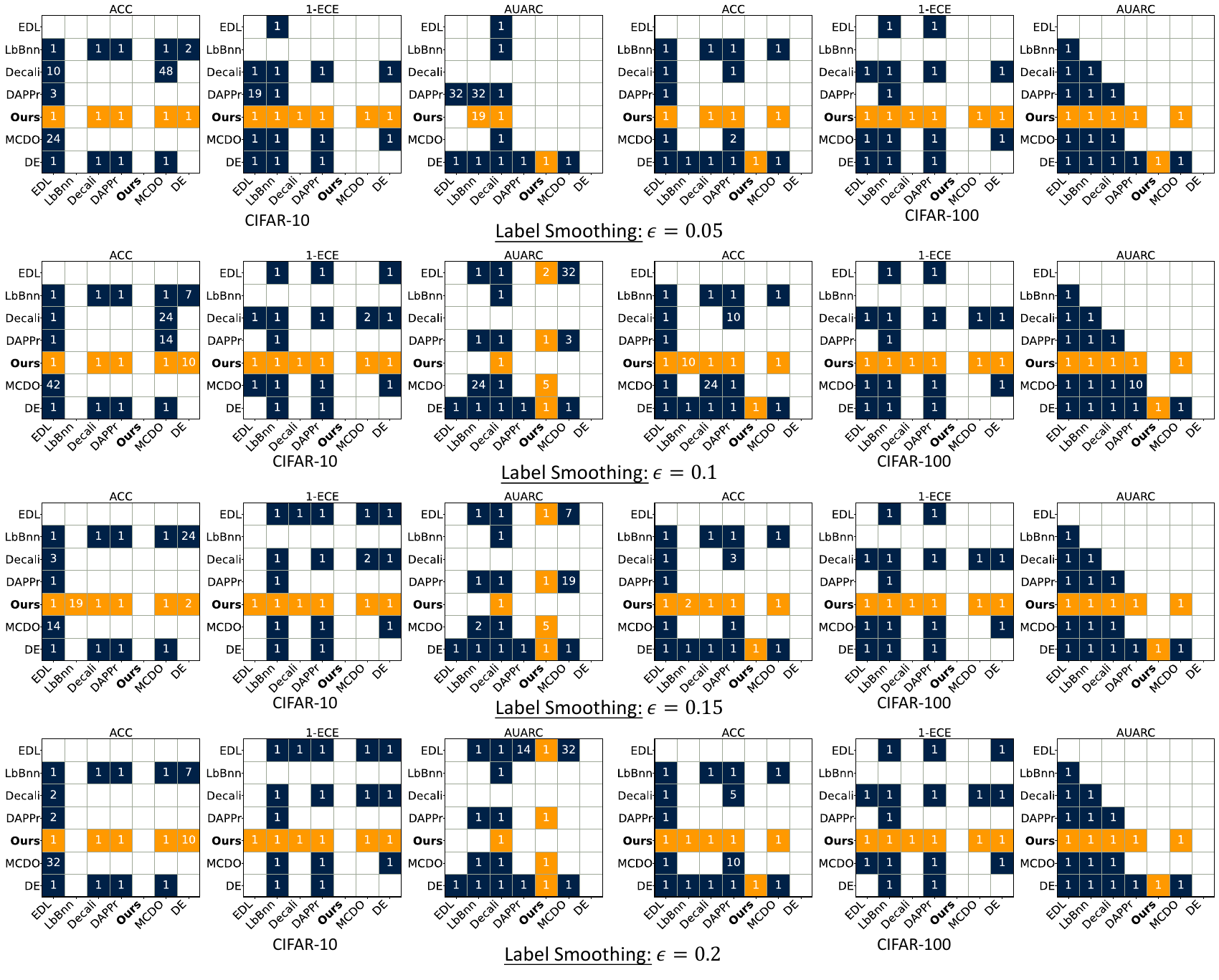}
\caption{Pairwise one-sided Wilcoxon signed-rank test results on various datasets under label-smoothing imprecision. Cell $(i,j)$ indicates whether row method $i$ significantly outperforms column method $j$; coloured cells show the scaled $p$-value ($\times 10^3$). Orange (\textcolor[HTML]{FF9900}{$\blacksquare$}) mark comparisons involving our method; dark navy (\textcolor[HTML]{002147}{$\blacksquare$}) mark all other significant pairs. }
\label{Figure: ST-smoothing}
\end{figure}
\newpage
\section{Comparison with Baselines Trained on Collapsed Labels}
\label{App:FurtherComparison}
\textbf{Setup.}
In this section, we further compare POCC against baseline methods trained on collapsed hard labels derived from imprecise annotations, covering the annotator-disagreement setting. Specifically, given a soft label $q_L$, each baseline is trained using the one-hot label $e_{j}$, where $j = \arg\max_{k \in [K]} q_L$, i.e., majority voting. All other training configurations follow those of the main evaluation in Section~\ref{sec: Exp}, and all results are averaged over 10 independent runs with seeds $1$--$10$.

We exclude the collapsed-label baseline from the knowledge distillation setting, as discarding the teacher's soft predictions contradicts standard practice in knowledge distillation~\citep{hinton2015distilling, phuong2019towards, zhou2021rethinking}. We exclude the experiment from the label smoothing setting as well, since collapsing a smoothed label simply recovers the original one-hot annotation, containing no annotation imprecision by construction and therefore falling outside the scope of our problem formulation.

\textbf{Results.}
Table~\ref{Table: Collaps-Disagree} reports ACC, ECE, AUARC, BQS, and the normalised balanced quality score under statistical significance testing (BQS-ST); see Appendix~\ref{App:StatisticalTest} for setup details. Across both settings, POCC achieves the highest BQS and BQS-ST values, demonstrating that it consistently yields the best trade-off among prediction accuracy, calibration, and EU quantification compared to baselines trained on collapsed hard labels.
\begin{table}[!htbp]
\caption{Performance comparison (in $\%$)  under annotator disagreement, where baseline methods are trained on collapsed hard labels derived from imprecise annotations.}
\label{Table: Collaps-Disagree}
\centering
\small
\begin{tabular}{lccc|cc}
\toprule
Method & ACC $\uparrow$ & ECE $\downarrow$ & AUARC $\uparrow$ & BQS $\uparrow$ & BQS-ST $\uparrow$ \\ 
\midrule\midrule
\multicolumn{6}{c}{CIFAR-10H}\\
\midrule\midrule
EDL    & $84.8 \scriptstyle{\pm 0.2}$ & $9.5 \scriptstyle{\pm 0.5}$ & $96.6 \scriptstyle{\pm 0.1}$ & $41.4 \scriptstyle{\pm 9.2}$ & $33.3$ \\
LbBnn  & $85.9 \scriptstyle{\pm 0.5}$ & $6.9 \scriptstyle{\pm 0.4}$ & $96.2 \scriptstyle{\pm 0.3}$ & $53.5 \scriptstyle{\pm 8.9}$ & $48.5$ \\
Decali & $82.7 \scriptstyle{\pm 0.9}$ & $9.5 \scriptstyle{\pm 0.5}$ & $95.5 \scriptstyle{\pm 0.4}$ & $3.4 \scriptstyle{\pm 5.1}$ & $0.0$ \\
DAPPr  & $85.5 \scriptstyle{\pm 0.4}$ & $6.2 \scriptstyle{\pm 0.4}$ & $96.6 \scriptstyle{\pm 0.2}$ & $59.5 \scriptstyle{\pm 7.5}$ & $62.9$ \\
\rowcolor{OxfordBlueLight}
Ours   & $\mathbf{87.0 \scriptstyle{\pm 0.4}}$ & $\underline{2.9 \scriptstyle{\pm 0.3}}$ & $\mathbf{97.1 \scriptstyle{\pm 0.2}}$ & $\mathbf{95.7 \scriptstyle{\pm 1.1}}$ & $\mathbf{93.9}$ \\
\midrule
MCDO   & $83.8 \scriptstyle{\pm 0.8}$ & $7.6 \scriptstyle{\pm 0.5}$ & $96.0 \scriptstyle{\pm 0.3}$ & $28.6 \scriptstyle{\pm 17.8}$ & $23.0$ \\
DE     & $\underline{85.9 \scriptstyle{\pm 0.4}}$ & $\mathbf{1.9 \scriptstyle{\pm 0.2}}$ & $\underline{96.7 \scriptstyle{\pm 0.2}}$ & $\underline{84.1 \scriptstyle{\pm 10.9}}$ & $\underline{77.8}$ \\
\midrule\midrule
\multicolumn{6}{c}{MiceBone}\\
\midrule\midrule
EDL    & $88.5 \scriptstyle{\pm 0.9}$ & $9.9 \scriptstyle{\pm 1.0}$ & $\mathbf{97.4 \scriptstyle{\pm 0.3}}$ & $58.3 \scriptstyle{\pm 5.4}$ & $50.0$ \\
Decali & $88.5 \scriptstyle{\pm 0.8}$ & $4.9 \scriptstyle{\pm 0.7}$ & $96.8 \scriptstyle{\pm 0.4}$ & $\underline{75.0 \scriptstyle{\pm 7.1}}$ & $\underline{52.4}$ \\
DAPPr  & $\underline{88.5 \scriptstyle{\pm 1.2}}$ & $8.8 \scriptstyle{\pm 0.6}$ & $95.1 \scriptstyle{\pm 0.9}$ & $36.0 \scriptstyle{\pm 10.4}$ & $26.2$ \\
\rowcolor{OxfordBlueLight}
Ours   & $\mathbf{89.1 \scriptstyle{\pm 0.9}}$ & $\mathbf{2.7 \scriptstyle{\pm 0.4}}$ & $\underline{97.3 \scriptstyle{\pm 0.3}}$ & $\mathbf{96.2 \scriptstyle{\pm 5.8}}$ & $\mathbf{100.0}$ \\
\midrule
DE     & $85.6 \scriptstyle{\pm 1.0}$ & $\underline{2.9 \scriptstyle{\pm 0.7}}$ & $95.3 \scriptstyle{\pm 0.4}$ & $37.3 \scriptstyle{\pm 7.3}$ & $33.3$ \\
\midrule\midrule
\multicolumn{6}{c}{TreeVersity}\\
\midrule\midrule
EDL    & $85.4 \scriptstyle{\pm 0.6}$ & $17.1 \scriptstyle{\pm 0.7}$ & $\underline{95.7 \scriptstyle{\pm 0.3}}$ & $59.5 \scriptstyle{\pm 1.9}$ & $39.3$ \\
Decali & $\underline{85.9 \scriptstyle{\pm 0.3}}$ & $8.5 \scriptstyle{\pm 0.4}$ & $95.3 \scriptstyle{\pm 0.3}$ & $\underline{78.4 \scriptstyle{\pm 2.0}}$ & $\underline{61.9}$ \\
DAPPr  & $85.7 \scriptstyle{\pm 0.8}$ & $13.2 \scriptstyle{\pm 0.5}$ & $94.4 \scriptstyle{\pm 0.5}$ & $63.6 \scriptstyle{\pm 1.8}$ & $35.7$ \\
\rowcolor{OxfordBlueLight}
Ours   & $\mathbf{86.2 \scriptstyle{\pm 0.6}}$ & $\mathbf{2.5 \scriptstyle{\pm 0.5}}$ & $\mathbf{96.7 \scriptstyle{\pm 0.2}}$ & $\mathbf{99.4 \scriptstyle{\pm 0.9}}$ & $\mathbf{100.0}$ \\
\midrule
DE     & $75.1 \scriptstyle{\pm 1.2}$ & $\underline{3.2 \scriptstyle{\pm 0.6}}$ & $88.8 \scriptstyle{\pm 0.7}$ & $31.7 \scriptstyle{\pm 1.6}$ & $25.0$ \\
\bottomrule
\end{tabular}
\end{table} 

\section{Evaluation with Normalised AUARC}
\label{App:normalisedAUARC}
To assess the quality of EU quantification relative to each model's base classification accuracy in selective classification, we report normalised AUARC values, where accuracy gains are measured relative to the test accuracy without rejection~\citep{wang2026learning}. Specifically, let \(A(\alpha)\) denote the accuracy at rejection rate \(\alpha\); we define 
the relative accuracy improvement as
\(\frac{A(\alpha) - A(0)}{1 - A(0)},\)
where \(A(0)\) is the initial model accuracy without rejection. Under this setting, we compute the balanced quality score for statistical testing (BQS-ST) (see Appendix~\ref{App:StatisticalTest} for setup details) from ACC, ECE, and normalised AUARC, and report the results in Figures~\ref{Figure: DisagreementAblationNor}, \ref{Figure: SmoothingAblationNor}, and~\ref{Figure: TeacherAblationNor} across the three annotation imprecision scenarios. The results show that our method consistently achieves the best trade-off among accuracy, calibration, and EU quantification.
\begin{figure}[!htbp]
\centering
\includegraphics[width=0.75\linewidth]{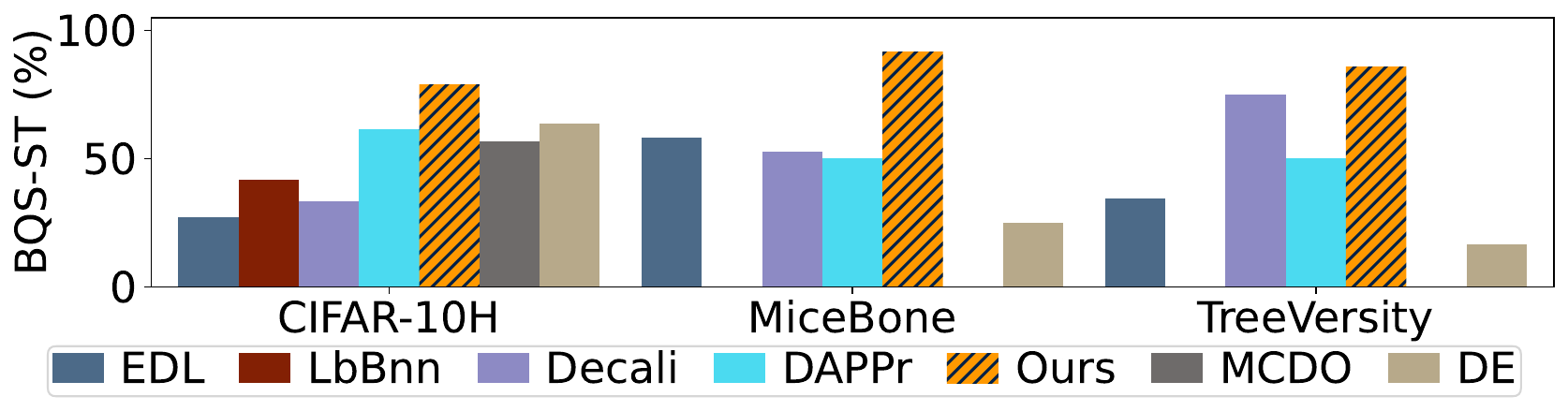}
\caption{BQS-ST comparison under annotator-disagreement annotation imprecision across various datasets, considering normalised AUARC.}
\label{Figure: DisagreementAblationNor}
\end{figure}

\begin{figure}[!htpb]
\centering
\includegraphics[width=\linewidth]{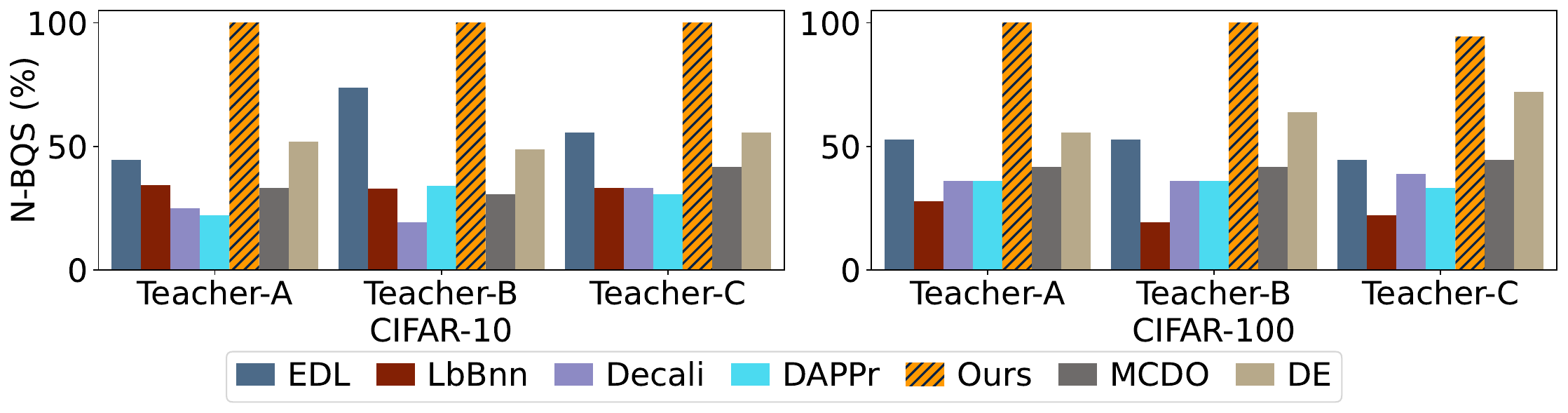}
\caption{BQS-ST comparison under teacher-prediction annotation imprecision across various settings, considering normalised AUARC. Teacher-A: ResNet-20; Teacher-B: ShuffleNetV2-0.5x; Teacher-C: MobileNetV2-x0-5.}
\label{Figure: TeacherAblationNor}
\end{figure}

\begin{figure}[!htbp]
\centering
\includegraphics[width=\linewidth]{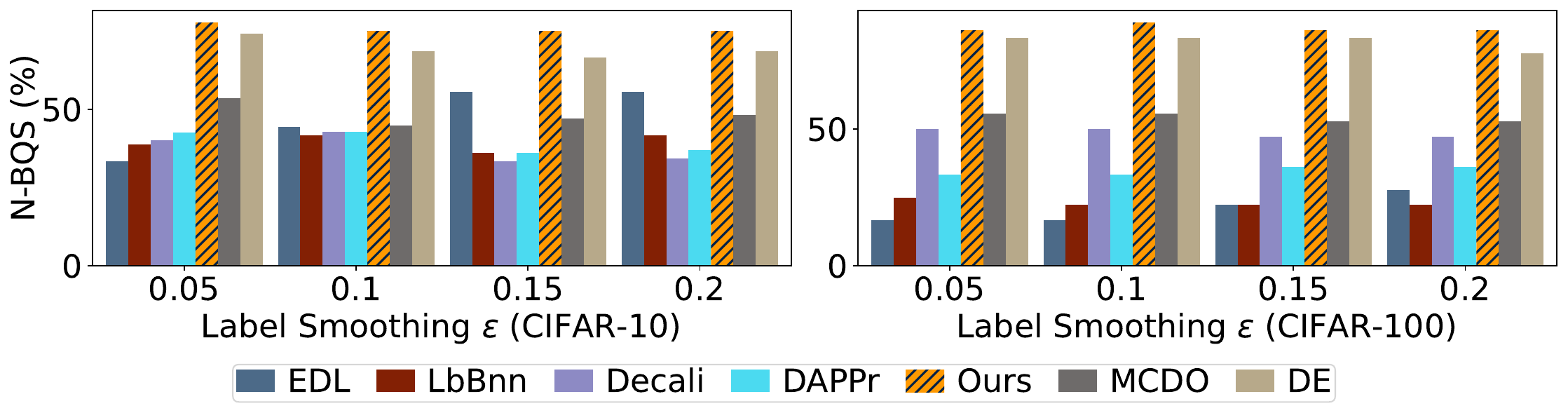}
\caption{BQS-ST comparison under label-smoothing annotation imprecision with varying $\epsilon$, considering normalised AUARC.}
\label{Figure: SmoothingAblationNor}
\end{figure}
\newpage
\section{Evaluation with EU--Prediction Error Correlation}
\label{App:SP}
\textbf{Setup.}
In this section, we apply an additional evaluation perspective---the alignment between epistemic uncertainty (EU) and prediction error---in place of AUARC in the selective classification task. Specifically, we measure this alignment via the Spearman rank correlation. Let \(\mathcal{D}_{\text{test}} = \{(x_i, y_i)\}_{i=1}^{N}\) denote the test set. For each sample \(i\), let \(u_i \in \mathbb{R}_{\geq 0}\) denote the predicted epistemic uncertainty and \(\ell_i = -\log \hat{p}(y_i \mid x_i)\) the cross-entropy loss. The Spearman rank correlation, denoted \(\text{Sp(EU,CE)}\), is computed as
\begin{equation}
    \text{Sp(EU,CE)}
    = \rho\!\left(r(\mathbf{u}),\, r(\boldsymbol{\ell})\right)
    = \frac{\mathrm{Cov}\!\left(r(\mathbf{u}),\, r(\boldsymbol{\ell})\right)}
           {\sigma_{r(\mathbf{u})}\,\sigma_{r(\boldsymbol{\ell})}},
\end{equation}
where \(r(\cdot)\) denotes the rank operator, \(\mathbf{u} = [u_1, \ldots, u_N]^{\top}\), and \(\boldsymbol{\ell} = [\ell_1, \ldots, \ell_N]^{\top}\). A higher value indicates stronger rank association with predictive loss: samples assigned higher epistemic uncertainty tend to incur greater prediction error. We adopt rank correlation rather than Pearson correlation because epistemic uncertainty quantities across different methods operate on incomparable scales; ranking removes this dependence while preserving ordinal structure.

A high value of \(\text{Sp(EU,CE)}\) does not, however, necessarily indicate a well-performing uncertainty estimator: when a model achieves very low accuracy, the cross-entropy loss \(\ell_i\) varies substantially across samples, and a spurious rank alignment with \(u_i\) can inflate \(\text{Sp(EU,CE)}\) even if the uncertainty estimates carry little practical value. This shortcoming highlights the necessity of a balanced evaluation across predictive accuracy, calibration, and \(\text{Sp(EU,CE)}\). Accordingly, we compute the balanced quality score (BQS) and the score for statistical testing (BQS-ST) from ACC, ECE, and \(\text{Sp(EU,CE)}\) (see Appendix~\ref{App:StatisticalTest} for setup details).

\textbf{Results.}
Tables~\ref{Table: Correlation-Disagree}, \ref{Table: Correlation-Teacher}, and~\ref{Table: Correlation-Smooth} across the three annotation imprecision scenarios. The results show that our method consistently achieves the best trade-off among accuracy, calibration, and EU quantification.  In contrast, none of the baselines exhibits a consistent pattern that achieves stable secondary or competitive balance scores across settings. This stems from its explicit treatment of annotation imprecision: baselines trained under standard supervised classification collapse the label into a single distribution, so their EU quantification mainly reflects model-internal uncertainty (e.g., initialisation or approximate inference) rather than supervision ambiguity. In contrast, POCC trains on credal labels and jointly optimises pessimistic and optimistic heads, so the resulting predictive credal set directly encodes annotation ambiguity. This enables simultaneous improvements in accuracy, calibration, and EU quality rather than (significantly) trading one metric against another. 
\begin{table}[!htbp]
\caption{Annotator-disagreement results (\%). Best in bold; second best underlined.}
\label{Table: Correlation-Disagree}
\centering
\small
\setlength\tabcolsep{10pt}
\begin{tabular}{llccc|cc}
\toprule
& Method & ACC $\uparrow$ & ECE $\downarrow$ & \(\text{Sp(EU,CE)}\) $\uparrow$ & BQS $\uparrow$ & BQS-ST $\uparrow$ \\ 
\midrule\midrule
\multirow{7}{*}{\rotatebox[origin=c]{90}{CIFAR-10H}}
& EDL    & $85.0 \scriptstyle{\pm 0.4}$ & $9.5 \scriptstyle{\pm 0.3}$ & $\mathbf{95.3 \scriptstyle{\pm 0.2}}$ & $40.8 \scriptstyle{\pm 6.4}$ & $33.3$ \\
& LbBnn  & $87.0 \scriptstyle{\pm 0.8}$ & $6.6 \scriptstyle{\pm 0.3}$ & $88.8 \scriptstyle{\pm 0.7}$ & $51.7 \scriptstyle{\pm 4.4}$ & $47.8$ \\
& Decali & $84.3 \scriptstyle{\pm 1.0}$ & $\underline{1.0 \scriptstyle{\pm 0.2}}$ & $85.6 \scriptstyle{\pm 0.9}$ & $35.4 \scriptstyle{\pm 4.5}$ & $33.3$ \\
& DAPPr  & $86.6 \scriptstyle{\pm 0.3}$ & $6.9 \scriptstyle{\pm 0.2}$ & $\mathbf{95.3 \scriptstyle{\pm 0.2}}$ & $68.5 \scriptstyle{\pm 5.6}$ & $61.6$ \\
& \cellcolor{OxfordBlueLight}Ours   & \cellcolor{OxfordBlueLight}$\mathbf{87.0 \scriptstyle{\pm 0.4}}$ & \cellcolor{OxfordBlueLight}$2.9 \scriptstyle{\pm 0.3}$ & \cellcolor{OxfordBlueLight}$94.4 \scriptstyle{\pm 0.5}$ & \cellcolor{OxfordBlueLight}$\mathbf{84.9 \scriptstyle{\pm 5.7}}$ & \cellcolor{OxfordBlueLight}$\mathbf{81.8}$ \\
\cmidrule{2-7}
& MCDO   & $85.6 \scriptstyle{\pm 0.7}$ & $\mathbf{0.9 \scriptstyle{\pm 0.1}}$ & $90.4 \scriptstyle{\pm 0.4}$ & $63.8 \scriptstyle{\pm 8.7}$ & $56.6$ \\
& DE     & $\underline{87.0 \scriptstyle{\pm 0.3}}$ & $4.3 \scriptstyle{\pm 0.2}$ & $92.0 \scriptstyle{\pm 0.2}$ & $\underline{71.0 \scriptstyle{\pm 4.5}}$ & $\underline{69.7}$ \\
\midrule\midrule
\multirow{5}{*}{\rotatebox[origin=c]{90}{MiceBone}}
& EDL    & $88.5 \scriptstyle{\pm 0.8}$ & $9.9 \scriptstyle{\pm 0.8}$ & $\mathbf{96.1 \scriptstyle{\pm 0.5}}$ & $\underline{69.4 \scriptstyle{\pm 3.1}}$ & $\underline{58.3}$ \\
& Decali & $88.7 \scriptstyle{\pm 1.2}$ & $\underline{5.8 \scriptstyle{\pm 1.0}}$ & $72.4 \scriptstyle{\pm 0.9}$ & $50.2 \scriptstyle{\pm 5.1}$ & $52.8$ \\
& DAPPr  & $\underline{89.0 \scriptstyle{\pm 1.1}}$ & $13.2 \scriptstyle{\pm 0.9}$ & $\underline{87.1 \scriptstyle{\pm 1.3}}$ & $49.9 \scriptstyle{\pm 4.2}$ & $58.3$ \\
& \cellcolor{OxfordBlueLight}Ours   & \cellcolor{OxfordBlueLight}$\mathbf{89.1 \scriptstyle{\pm 0.9}}$ & \cellcolor{OxfordBlueLight}$\mathbf{2.7 \scriptstyle{\pm 0.4}}$ & \cellcolor{OxfordBlueLight}$82.2 \scriptstyle{\pm 1.8}$ & \cellcolor{OxfordBlueLight}$\mathbf{77.7 \scriptstyle{\pm 4.2}}$ & \cellcolor{OxfordBlueLight}$\mathbf{83.3}$ \\
\cmidrule{2-7}
& DE     & $85.5 \scriptstyle{\pm 0.9}$ & $7.5 \scriptstyle{\pm 0.8}$ & $75.9 \scriptstyle{\pm 2.1}$ & $22.6 \scriptstyle{\pm 4.4}$ & $25.0$ \\
\midrule\midrule
\multirow{5}{*}{\rotatebox[origin=c]{90}{TreeVersity}}
& EDL    & $85.5 \scriptstyle{\pm 0.6}$ & $17.1 \scriptstyle{\pm 0.7}$ & $\mathbf{92.8 \scriptstyle{\pm 0.6}}$ & $62.6 \scriptstyle{\pm 1.6}$ & $42.9$ \\
& Decali & $\mathbf{86.9 \scriptstyle{\pm 0.4}}$ & $\mathbf{2.3 \scriptstyle{\pm 0.3}}$ & $81.7 \scriptstyle{\pm 1.7}$ & $\underline{83.4 \scriptstyle{\pm 2.4}}$ & $\underline{75.0}$ \\
& DAPPr  & $\underline{86.6 \scriptstyle{\pm 0.5}}$ & $17.0 \scriptstyle{\pm 0.4}$ & $85.8 \scriptstyle{\pm 1.1}$ & $56.3 \scriptstyle{\pm 1.5}$ & $50.0$ \\
& \cellcolor{OxfordBlueLight}Ours   & \cellcolor{OxfordBlueLight}$86.2 \scriptstyle{\pm 0.6}$ & \cellcolor{OxfordBlueLight}$\underline{2.5 \scriptstyle{\pm 0.5}}$ & \cellcolor{OxfordBlueLight}$\underline{90.8 \scriptstyle{\pm 0.5}}$ & \cellcolor{OxfordBlueLight}$\mathbf{94.2 \scriptstyle{\pm 1.8}}$ & \cellcolor{OxfordBlueLight}$\mathbf{77.4}$ \\
\cmidrule{2-7}
& DE     & $75.7 \scriptstyle{\pm 1.1}$ & $6.2 \scriptstyle{\pm 0.9}$ & $69.6 \scriptstyle{\pm 1.7}$ & $24.4 \scriptstyle{\pm 1.7}$ & $16.7$ \\
\bottomrule
\end{tabular}
\end{table}

\begin{table}[!htbp]
\caption{Teacher-prediction results (\%). Best in bold; second best underlined.}
\label{Table: Correlation-Teacher}
\centering
\small
\setlength\tabcolsep{0.5pt}
\begin{tabular}{lccc|cc||ccc|cc}
\toprule
\multirow{2}{*}{Method} & \multicolumn{5}{c||}{CIFAR-10} & \multicolumn{5}{c}{CIFAR-100} \\
\cmidrule(lr){2-6} \cmidrule(lr){7-11}
& ACC & ECE & \(\text{Sp(EU,CE)}\) & BQS & BQS-ST & ACC  & ECE & \(\text{Sp(EU,CE)}\) & BQS & BQS-ST \\
\midrule \midrule
\multicolumn{11}{c}{Teacher: ResNet-20}\\ \midrule\midrule
EDL    & $93.4 \scriptstyle{\pm 0.2}$ & $9.0 \scriptstyle{\pm 0.2}$ & $\mathbf{98.9 \scriptstyle{\pm 0.1}}$ & $\underline{60.7 \scriptstyle{\pm 2.9}}$ & $\underline{50.0}$ & $33.1 \scriptstyle{\pm 5.7}$ & $\underline{23.3 \scriptstyle{\pm 3.9}}$ & $62.7 \scriptstyle{\pm 5.6}$ & $56.6 \scriptstyle{\pm 1.0}$ & $50.0$ \\
LbBnn  & $\underline{95.1 \scriptstyle{\pm 0.1}}$ & $28.1 \scriptstyle{\pm 0.2}$ & $96.0 \scriptstyle{\pm 0.1}$ & $51.3 \scriptstyle{\pm 2.9}$ & $44.1$ & $\underline{73.1 \scriptstyle{\pm 0.2}}$ & $63.1 \scriptstyle{\pm 0.2}$ & $-43.4 \scriptstyle{\pm 1.4}$ & $30.2 \scriptstyle{\pm 0.3}$ & $27.8$ \\
Decali & $93.5 \scriptstyle{\pm 0.1}$ & $8.5 \scriptstyle{\pm 0.1}$ & $88.7 \scriptstyle{\pm 0.4}$ & $29.7 \scriptstyle{\pm 2.5}$ & $25.0$ & $71.0 \scriptstyle{\pm 0.2}$ & $27.8 \scriptstyle{\pm 0.2}$ & $6.1 \scriptstyle{\pm 0.4}$ & $65.7 \scriptstyle{\pm 0.6}$ & $36.1$ \\
DAPPr  & $93.4 \scriptstyle{\pm 0.2}$ & $14.6 \scriptstyle{\pm 0.2}$ & $97.6 \scriptstyle{\pm 0.1}$ & $49.8 \scriptstyle{\pm 2.4}$ & $29.8$ & $72.5 \scriptstyle{\pm 0.4}$ & $52.8 \scriptstyle{\pm 0.3}$ & $42.8 \scriptstyle{\pm 0.7}$ & $61.7 \scriptstyle{\pm 0.5}$ & $36.1$ \\
\rowcolor{OxfordBlueLight}
Ours   & $\mathbf{95.5 \scriptstyle{\pm 0.1}}$ & $\mathbf{3.0 \scriptstyle{\pm 0.1}}$ & $\mathbf{98.9 \scriptstyle{\pm 0.1}}$ & $\mathbf{99.7 \scriptstyle{\pm 0.6}}$ & $\mathbf{100.0}$ & $\mathbf{77.2 \scriptstyle{\pm 0.2}}$ & $\mathbf{10.6 \scriptstyle{\pm 0.4}}$ & $\underline{63.0 \scriptstyle{\pm 1.4}}$ & $\mathbf{98.1 \scriptstyle{\pm 0.4}}$ & $\mathbf{91.7}$ \\
\midrule
MCDO   & $93.5 \scriptstyle{\pm 0.2}$ & $\underline{8.5 \scriptstyle{\pm 0.2}}$ & $90.3 \scriptstyle{\pm 0.2}$ & $34.5 \scriptstyle{\pm 3.3}$ & $31.1$ & $71.1 \scriptstyle{\pm 0.3}$ & $28.2 \scriptstyle{\pm 0.3}$ & $58.3 \scriptstyle{\pm 1.6}$ & $80.9 \scriptstyle{\pm 0.8}$ & $44.4$ \\
DE     & $94.2 \scriptstyle{\pm 0.1}$ & $9.5 \scriptstyle{\pm 0.1}$ & $91.8 \scriptstyle{\pm 0.1}$ & $49.4 \scriptstyle{\pm 2.4}$ & $41.8$ & $72.7 \scriptstyle{\pm 0.3}$ & $29.9 \scriptstyle{\pm 0.3}$ & $\mathbf{69.3 \scriptstyle{\pm 0.6}}$ & $\underline{84.2 \scriptstyle{\pm 0.4}}$ & $\underline{63.9}$ \\
\midrule \midrule
\multicolumn{11}{c}{Teacher: ShuffleNetV2-0.5×}\\ \midrule\midrule 
EDL    & $92.9 \scriptstyle{\pm 0.3}$ & $\underline{9.1 \scriptstyle{\pm 0.1}}$ & $\mathbf{98.7 \scriptstyle{\pm 0.1}}$ & $\underline{73.0 \scriptstyle{\pm 2.5}}$ & $\underline{79.3}$ & $33.0 \scriptstyle{\pm 5.6}$ & $\underline{23.3 \scriptstyle{\pm 3.8}}$ & $\underline{63.1 \scriptstyle{\pm 6.0}}$ & $57.3 \scriptstyle{\pm 0.8}$ & $52.8$ \\
LbBnn  & $\underline{94.1 \scriptstyle{\pm 0.1}}$ & $33.2 \scriptstyle{\pm 0.2}$ & $89.8 \scriptstyle{\pm 0.2}$ & $46.3 \scriptstyle{\pm 1.4}$ & $38.4$ & $72.6 \scriptstyle{\pm 0.2}$ & $65.4 \scriptstyle{\pm 0.2}$ & $-51.1 \scriptstyle{\pm 1.0}$ & $29.3 \scriptstyle{\pm 0.6}$ & $19.4$ \\
Decali & $92.1 \scriptstyle{\pm 0.2}$ & $12.1 \scriptstyle{\pm 0.2}$ & $73.4 \scriptstyle{\pm 0.3}$ & $26.9 \scriptstyle{\pm 1.7}$ & $19.4$ & $71.0 \scriptstyle{\pm 0.2}$ & $27.4 \scriptstyle{\pm 0.2}$ & $-1.1 \scriptstyle{\pm 0.6}$ & $65.4 \scriptstyle{\pm 0.6}$ & $36.1$ \\
DAPPr  & $92.5 \scriptstyle{\pm 0.2}$ & $18.4 \scriptstyle{\pm 0.2}$ & $95.7 \scriptstyle{\pm 0.1}$ & $53.6 \scriptstyle{\pm 3.1}$ & $36.9$ & $72.6 \scriptstyle{\pm 0.4}$ & $52.3 \scriptstyle{\pm 0.4}$ & $48.6 \scriptstyle{\pm 0.8}$ & $64.8 \scriptstyle{\pm 0.6}$ & $36.1$ \\
\rowcolor{OxfordBlueLight}
Ours   & $\mathbf{94.9 \scriptstyle{\pm 0.1}}$ & $\mathbf{5.4 \scriptstyle{\pm 0.2}}$ & $\underline{98.4 \scriptstyle{\pm 0.1}}$ & $\mathbf{99.6 \scriptstyle{\pm 0.2}}$ & $\mathbf{94.4}$ & $\mathbf{77.8 \scriptstyle{\pm 0.3}}$ & $\mathbf{11.7 \scriptstyle{\pm 0.3}}$ & $58.2 \scriptstyle{\pm 1.3}$ & $\mathbf{96.5 \scriptstyle{\pm 0.3}}$ & $\mathbf{88.9}$ \\
\midrule
MCDO   & $92.0 \scriptstyle{\pm 0.2}$ & $12.2 \scriptstyle{\pm 0.2}$ & $88.9 \scriptstyle{\pm 0.2}$ & $46.6 \scriptstyle{\pm 1.2}$ & $25.0$ & $70.8 \scriptstyle{\pm 0.4}$ & $27.7 \scriptstyle{\pm 0.3}$ & $56.9 \scriptstyle{\pm 1.2}$ & $80.9 \scriptstyle{\pm 0.8}$ & $44.4$ \\
DE     & $92.9 \scriptstyle{\pm 0.1}$ & $13.1 \scriptstyle{\pm 0.1}$ & $90.7 \scriptstyle{\pm 0.2}$ & $57.6 \scriptstyle{\pm 1.5}$ & $46.0$ & $\underline{73.0 \scriptstyle{\pm 0.2}}$ & $29.9 \scriptstyle{\pm 0.1}$ & $\mathbf{71.0 \scriptstyle{\pm 0.5}}$ & $\underline{85.1 \scriptstyle{\pm 0.5}}$ & $\underline{72.2}$ \\
\midrule \midrule
\multicolumn{11}{c}{Teacher: MobileNetV2-x0-5}\\ \midrule\midrule
EDL    & $93.2 \scriptstyle{\pm 0.2}$ & $\underline{9.1 \scriptstyle{\pm 0.2}}$ & $\mathbf{98.9 \scriptstyle{\pm 0.1}}$ & $\underline{61.8 \scriptstyle{\pm 0.3}}$ & $\underline{61.1}$ & $37.8 \scriptstyle{\pm 1.5}$ & $26.6 \scriptstyle{\pm 1.1}$ & $\underline{68.6 \scriptstyle{\pm 1.1}}$ & $56.0 \scriptstyle{\pm 0.4}$ & $50.0$ \\
LbBnn  & $\underline{95.3 \scriptstyle{\pm 0.1}}$ & $37.4 \scriptstyle{\pm 0.3}$ & $89.7 \scriptstyle{\pm 0.4}$ & $46.5 \scriptstyle{\pm 1.9}$ & $44.4$ & $76.1 \scriptstyle{\pm 0.2}$ & $68.2 \scriptstyle{\pm 0.3}$ & $-62.3 \scriptstyle{\pm 0.8}$ & $31.0 \scriptstyle{\pm 0.3}$ & $22.2$ \\
Decali & $93.6 \scriptstyle{\pm 0.1}$ & $10.3 \scriptstyle{\pm 0.1}$ & $80.9 \scriptstyle{\pm 0.2}$ & $33.7 \scriptstyle{\pm 2.1}$ & $33.3$ & $74.4 \scriptstyle{\pm 0.3}$ & $\underline{25.8 \scriptstyle{\pm 0.3}}$ & $13.5 \scriptstyle{\pm 1.3}$ & $72.6 \scriptstyle{\pm 0.5}$ & $38.9$ \\
DAPPr  & $93.8 \scriptstyle{\pm 0.2}$ & $16.1 \scriptstyle{\pm 0.2}$ & $97.1 \scriptstyle{\pm 0.1}$ & $60.1 \scriptstyle{\pm 2.4}$ & $41.7$ & $75.3 \scriptstyle{\pm 0.4}$ & $51.9 \scriptstyle{\pm 0.4}$ & $63.0 \scriptstyle{\pm 0.8}$ & $70.2 \scriptstyle{\pm 0.3}$ & $41.7$ \\
\rowcolor{OxfordBlueLight}
Ours   & $\mathbf{95.5 \scriptstyle{\pm 0.1}}$ & $\mathbf{4.2 \scriptstyle{\pm 0.2}}$ & $\underline{98.6 \scriptstyle{\pm 0.1}}$ & $\mathbf{99.3 \scriptstyle{\pm 0.5}}$ & $\mathbf{94.4}$ & $\mathbf{79.0 \scriptstyle{\pm 0.2}}$ & $\mathbf{10.9 \scriptstyle{\pm 0.2}}$ & $59.4 \scriptstyle{\pm 1.0}$ & $\mathbf{96.2 \scriptstyle{\pm 0.3}}$ & $\mathbf{77.8}$ \\
\midrule
MCDO   & $93.4 \scriptstyle{\pm 0.1}$ & $10.2 \scriptstyle{\pm 0.2}$ & $84.9 \scriptstyle{\pm 0.1}$ & $37.4 \scriptstyle{\pm 2.0}$ & $30.6$ & $73.9 \scriptstyle{\pm 0.3}$ & $25.8 \scriptstyle{\pm 0.2}$ & $64.0 \scriptstyle{\pm 1.2}$ & $84.5 \scriptstyle{\pm 0.5}$ & $47.2$ \\
DE     & $94.2 \scriptstyle{\pm 0.1}$ & $11.4 \scriptstyle{\pm 0.1}$ & $86.1 \scriptstyle{\pm 0.2}$ & $50.4 \scriptstyle{\pm 1.6}$ & $44.4$ & $\underline{76.4 \scriptstyle{\pm 0.1}}$ & $28.4 \scriptstyle{\pm 0.1}$ & $\mathbf{75.1 \scriptstyle{\pm 0.4}}$ & $\underline{87.7 \scriptstyle{\pm 0.3}}$ & $\underline{72.2}$ \\
\bottomrule
\end{tabular}
\end{table}

\begin{table}[!htbp]
\caption{Label-smoothing results (\%). Best in bold; second best underlined.}
\label{Table: Correlation-Smooth}
\centering
\small
\setlength\tabcolsep{0.75pt}
\begin{tabular}{lccc|cc||ccc|cc}
\toprule
\multirow{2}{*}{Method} & \multicolumn{5}{c||}{CIFAR-10} & \multicolumn{5}{c}{CIFAR-100} \\
\cmidrule(lr){2-6} \cmidrule(lr){7-11}
& ACC & ECE & \(\text{Sp(EU,CE)}\) & BQS & BQS-ST & ACC  & ECE & \(\text{Sp(EU,CE)}\) & BQS & BQS-ST \\
\midrule \midrule
\multicolumn{11}{c}{$\epsilon=0.05$}\\ \midrule\midrule
EDL    & $93.4 \scriptstyle{\pm 0.3}$ & $9.1 \scriptstyle{\pm 0.2}$ & $\mathbf{99.0 \scriptstyle{\pm 0.1}}$ & $58.9 \scriptstyle{\pm 1.9}$ & $38.9$ & $26.3 \scriptstyle{\pm 1.2}$ & $18.7 \scriptstyle{\pm 1.0}$ & $58.6 \scriptstyle{\pm 1.8}$ & $47.3 \scriptstyle{\pm 0.6}$ & $16.7$ \\
LbBnn  & $\underline{95.4 \scriptstyle{\pm 0.2}}$ & $29.9 \scriptstyle{\pm 0.3}$ & $64.8 \scriptstyle{\pm 2.6}$ & $40.3 \scriptstyle{\pm 3.8}$ & $45.5$ & $\underline{79.3 \scriptstyle{\pm 0.3}}$ & $61.6 \scriptstyle{\pm 0.5}$ & $-5.8 \scriptstyle{\pm 4.0}$ & $32.7 \scriptstyle{\pm 0.3}$ & $25.0$ \\
Decali & $93.9 \scriptstyle{\pm 0.3}$ & $3.6 \scriptstyle{\pm 0.1}$ & $52.3 \scriptstyle{\pm 6.5}$ & $39.5 \scriptstyle{\pm 5.6}$ & $40.2$ & $76.2 \scriptstyle{\pm 0.3}$ & $\underline{3.5 \scriptstyle{\pm 0.4}}$ & $76.6 \scriptstyle{\pm 0.9}$ & $93.1 \scriptstyle{\pm 0.7}$ & $50.0$ \\
DAPPr  & $93.9 \scriptstyle{\pm 0.3}$ & $8.9 \scriptstyle{\pm 0.2}$ & $\underline{98.9 \scriptstyle{\pm 0.1}}$ & $67.1 \scriptstyle{\pm 4.0}$ & $56.6$ & $75.1 \scriptstyle{\pm 0.4}$ & $35.3 \scriptstyle{\pm 0.4}$ & $\mathbf{87.6 \scriptstyle{\pm 0.4}}$ & $78.4 \scriptstyle{\pm 0.2}$ & $44.4$ \\
\rowcolor{OxfordBlueLight}
Ours   & $\mathbf{95.5 \scriptstyle{\pm 0.1}}$ & $\mathbf{1.5 \scriptstyle{\pm 0.1}}$ & $97.5 \scriptstyle{\pm 0.4}$ & $\mathbf{97.9 \scriptstyle{\pm 1.7}}$ & $\mathbf{90.9}$ & $79.3 \scriptstyle{\pm 0.2}$ & $\mathbf{2.7 \scriptstyle{\pm 0.2}}$ & $85.7 \scriptstyle{\pm 0.4}$ & $\mathbf{98.7 \scriptstyle{\pm 0.3}}$ & $\mathbf{80.6}$ \\
\midrule
MCDO   & $93.7 \scriptstyle{\pm 0.2}$ & $\underline{3.5 \scriptstyle{\pm 0.1}}$ & $61.3 \scriptstyle{\pm 1.6}$ & $43.6 \scriptstyle{\pm 5.2}$ & $40.2$ & $76.1 \scriptstyle{\pm 0.3}$ & $3.8 \scriptstyle{\pm 0.4}$ & $82.4 \scriptstyle{\pm 0.5}$ & $95.0 \scriptstyle{\pm 0.6}$ & $55.6$ \\
DE     & $95.2 \scriptstyle{\pm 0.1}$ & $5.5 \scriptstyle{\pm 0.1}$ & $69.8 \scriptstyle{\pm 0.9}$ & $\underline{68.8 \scriptstyle{\pm 3.9}}$ & $\underline{59.1}$ & $\mathbf{80.3 \scriptstyle{\pm 0.2}}$ & $9.1 \scriptstyle{\pm 0.2}$ & $\underline{86.8 \scriptstyle{\pm 0.3}}$ & $\underline{96.1 \scriptstyle{\pm 0.2}}$ & $\underline{77.8}$ \\
\midrule \midrule
\multicolumn{11}{c}{$\epsilon=0.1$}\\ \midrule\midrule 
EDL    & $93.3 \scriptstyle{\pm 0.1}$ & $9.2 \scriptstyle{\pm 0.3}$ & $\mathbf{99.0 \scriptstyle{\pm 0.1}}$ & $62.6 \scriptstyle{\pm 0.6}$ & $50.0$ & $27.5 \scriptstyle{\pm 1.1}$ & $19.6 \scriptstyle{\pm 0.9}$ & $60.2 \scriptstyle{\pm 1.1}$ & $51.5 \scriptstyle{\pm 0.4}$ & $22.2$ \\
LbBnn  & $\mathbf{95.4 \scriptstyle{\pm 0.1}}$ & $45.2 \scriptstyle{\pm 0.1}$ & $43.6 \scriptstyle{\pm 6.0}$ & $38.3 \scriptstyle{\pm 4.6}$ & $33.3$ & $79.1 \scriptstyle{\pm 0.2}$ & $70.0 \scriptstyle{\pm 0.3}$ & $-52.6 \scriptstyle{\pm 4.3}$ & $32.6 \scriptstyle{\pm 0.2}$ & $22.2$ \\
Decali & $93.9 \scriptstyle{\pm 0.3}$ & $\underline{7.3 \scriptstyle{\pm 0.1}}$ & $37.1 \scriptstyle{\pm 13.5}$ & $42.3 \scriptstyle{\pm 5.9}$ & $42.9$ & $75.9 \scriptstyle{\pm 0.5}$ & $\underline{5.0 \scriptstyle{\pm 0.3}}$ & $57.2 \scriptstyle{\pm 2.9}$ & $88.5 \scriptstyle{\pm 0.8}$ & $44.4$ \\
DAPPr  & $93.9 \scriptstyle{\pm 0.2}$ & $11.8 \scriptstyle{\pm 0.2}$ & $\underline{98.9 \scriptstyle{\pm 0.1}}$ & $69.6 \scriptstyle{\pm 3.4}$ & $\underline{54.0}$ & $75.0 \scriptstyle{\pm 0.4}$ & $36.8 \scriptstyle{\pm 0.4}$ & $\mathbf{86.9 \scriptstyle{\pm 0.6}}$ & $79.5 \scriptstyle{\pm 0.2}$ & $44.4$ \\
\rowcolor{OxfordBlueLight}
Ours   & $\underline{95.4 \scriptstyle{\pm 0.1}}$ & $\mathbf{3.7 \scriptstyle{\pm 0.2}}$ & $95.8 \scriptstyle{\pm 0.5}$ & $\mathbf{96.9 \scriptstyle{\pm 1.8}}$ & $\mathbf{90.0}$ & $\underline{79.4 \scriptstyle{\pm 0.2}}$ & $\mathbf{1.5 \scriptstyle{\pm 0.2}}$ & $83.7 \scriptstyle{\pm 0.4}$ & $\mathbf{98.7 \scriptstyle{\pm 0.2}}$ & $\mathbf{83.3}$ \\
\midrule
MCDO   & $93.6 \scriptstyle{\pm 0.3}$ & $7.8 \scriptstyle{\pm 0.2}$ & $58.3 \scriptstyle{\pm 1.4}$ & $47.6 \scriptstyle{\pm 4.1}$ & $38.3$ & $76.4 \scriptstyle{\pm 0.3}$ & $7.3 \scriptstyle{\pm 0.3}$ & $80.0 \scriptstyle{\pm 0.3}$ & $93.1 \scriptstyle{\pm 0.2}$ & $55.6$ \\
DE     & $95.2 \scriptstyle{\pm 0.1}$ & $10.0 \scriptstyle{\pm 0.1}$ & $68.1 \scriptstyle{\pm 1.3}$ & $\underline{75.6 \scriptstyle{\pm 2.8}}$ & $52.0$ & $\mathbf{80.2 \scriptstyle{\pm 0.2}}$ & $12.8 \scriptstyle{\pm 0.2}$ & $\underline{85.5 \scriptstyle{\pm 0.2}}$ & $\underline{94.2 \scriptstyle{\pm 0.2}}$ & $\underline{77.8}$ \\
\midrule \midrule
\multicolumn{11}{c}{$\epsilon=0.15$}\\ \midrule\midrule
EDL    & $93.5 \scriptstyle{\pm 0.3}$ & $\underline{9.2 \scriptstyle{\pm 0.3}}$ & $\mathbf{98.9 \scriptstyle{\pm 0.1}}$ & $65.1 \scriptstyle{\pm 1.4}$ & $\underline{61.1}$ & $26.8 \scriptstyle{\pm 1.2}$ & $19.1 \scriptstyle{\pm 0.9}$ & $59.5 \scriptstyle{\pm 1.3}$ & $52.0 \scriptstyle{\pm 0.5}$ & $22.2$ \\
LbBnn  & $\underline{95.3 \scriptstyle{\pm 0.1}}$ & $54.5 \scriptstyle{\pm 0.1}$ & $32.2 \scriptstyle{\pm 3.2}$ & $35.9 \scriptstyle{\pm 3.7}$ & $33.3$ & $78.8 \scriptstyle{\pm 0.3}$ & $72.7 \scriptstyle{\pm 0.3}$ & $-51.7 \scriptstyle{\pm 3.9}$ & $32.5 \scriptstyle{\pm 0.3}$ & $22.2$ \\
Decali & $93.9 \scriptstyle{\pm 0.2}$ & $11.6 \scriptstyle{\pm 0.3}$ & $21.5 \scriptstyle{\pm 11.6}$ & $37.1 \scriptstyle{\pm 4.8}$ & $33.3$ & $75.7 \scriptstyle{\pm 0.5}$ & $\underline{8.9 \scriptstyle{\pm 0.3}}$ & $34.2 \scriptstyle{\pm 2.1}$ & $81.3 \scriptstyle{\pm 0.8}$ & $47.2$ \\
DAPPr  & $94.0 \scriptstyle{\pm 0.3}$ & $15.4 \scriptstyle{\pm 0.2}$ & $\underline{98.8 \scriptstyle{\pm 0.1}}$ & $69.3 \scriptstyle{\pm 4.0}$ & $44.4$ & $74.8 \scriptstyle{\pm 0.5}$ & $38.6 \scriptstyle{\pm 0.4}$ & $\mathbf{86.9 \scriptstyle{\pm 0.7}}$ & $79.4 \scriptstyle{\pm 0.2}$ & $44.4$ \\
\rowcolor{OxfordBlueLight}
Ours   & $\mathbf{95.5 \scriptstyle{\pm 0.1}}$ & $\mathbf{6.0 \scriptstyle{\pm 0.2}}$ & $93.9 \scriptstyle{\pm 0.5}$ & $\mathbf{97.6 \scriptstyle{\pm 0.5}}$ & $\mathbf{88.9}$ & $\underline{79.4 \scriptstyle{\pm 0.1}}$ & $\mathbf{2.0 \scriptstyle{\pm 0.2}}$ & $81.5 \scriptstyle{\pm 0.6}$ & $\mathbf{98.2 \scriptstyle{\pm 0.2}}$ & $\mathbf{83.3}$ \\
\midrule
MCDO   & $93.7 \scriptstyle{\pm 0.4}$ & $12.2 \scriptstyle{\pm 0.2}$ & $55.2 \scriptstyle{\pm 1.4}$ & $48.7 \scriptstyle{\pm 3.2}$ & $38.9$ & $76.1 \scriptstyle{\pm 0.2}$ & $10.9 \scriptstyle{\pm 0.2}$ & $77.7 \scriptstyle{\pm 0.4}$ & $91.1 \scriptstyle{\pm 0.2}$ & $52.8$ \\
DE     & $95.2 \scriptstyle{\pm 0.1}$ & $14.3 \scriptstyle{\pm 0.1}$ & $66.1 \scriptstyle{\pm 1.6}$ & $\underline{75.6 \scriptstyle{\pm 2.4}}$ & $50.0$ & $\mathbf{80.2 \scriptstyle{\pm 0.2}}$ & $17.1 \scriptstyle{\pm 0.3}$ & $\underline{84.5 \scriptstyle{\pm 0.2}}$ & $\underline{92.3 \scriptstyle{\pm 0.2}}$ & $\underline{77.8}$ \\
\midrule \midrule
\multicolumn{11}{c}{$\epsilon=0.2$}\\ \midrule\midrule
EDL    & $93.3 \scriptstyle{\pm 0.4}$ & $\underline{9.2 \scriptstyle{\pm 0.1}}$ & $\mathbf{98.9 \scriptstyle{\pm 0.1}}$ & $66.9 \scriptstyle{\pm 2.0}$ & $\underline{61.1}$ & $26.4 \scriptstyle{\pm 0.9}$ & $18.8 \scriptstyle{\pm 0.7}$ & $59.2 \scriptstyle{\pm 1.5}$ & $52.4 \scriptstyle{\pm 0.3}$ & $27.8$ \\
LbBnn  & $\mathbf{95.4 \scriptstyle{\pm 0.1}}$ & $61.0 \scriptstyle{\pm 0.1}$ & $26.2 \scriptstyle{\pm 4.6}$ & $37.0 \scriptstyle{\pm 4.4}$ & $33.3$ & $78.7 \scriptstyle{\pm 0.3}$ & $74.2 \scriptstyle{\pm 0.3}$ & $-47.4 \scriptstyle{\pm 5.4}$ & $32.5 \scriptstyle{\pm 0.2}$ & $22.2$ \\
Decali & $93.9 \scriptstyle{\pm 0.2}$ & $16.0 \scriptstyle{\pm 0.2}$ & $18.0 \scriptstyle{\pm 18.5}$ & $40.3 \scriptstyle{\pm 4.5}$ & $34.3$ & $75.4 \scriptstyle{\pm 0.3}$ & $\underline{13.5 \scriptstyle{\pm 0.3}}$ & $17.1 \scriptstyle{\pm 2.1}$ & $74.8 \scriptstyle{\pm 0.8}$ & $47.2$ \\
DAPPr  & $93.9 \scriptstyle{\pm 0.2}$ & $18.9 \scriptstyle{\pm 0.2}$ & $\underline{98.7 \scriptstyle{\pm 0.1}}$ & $68.8 \scriptstyle{\pm 2.2}$ & $44.9$ & $75.0 \scriptstyle{\pm 0.4}$ & $40.8 \scriptstyle{\pm 0.3}$ & $\mathbf{86.6 \scriptstyle{\pm 0.6}}$ & $79.0 \scriptstyle{\pm 0.2}$ & $44.4$ \\
\rowcolor{OxfordBlueLight}
Ours   & $\underline{95.4 \scriptstyle{\pm 0.1}}$ & $\mathbf{7.6 \scriptstyle{\pm 0.1}}$ & $91.6 \scriptstyle{\pm 0.6}$ & $\mathbf{95.6 \scriptstyle{\pm 1.9}}$ & $\mathbf{87.9}$ & $\underline{79.3 \scriptstyle{\pm 0.3}}$ & $\mathbf{2.7 \scriptstyle{\pm 0.3}}$ & $79.2 \scriptstyle{\pm 0.7}$ & $\mathbf{97.6 \scriptstyle{\pm 0.2}}$ & $\mathbf{83.3}$ \\
\midrule
MCDO   & $93.7 \scriptstyle{\pm 0.4}$ & $16.5 \scriptstyle{\pm 0.2}$ & $52.9 \scriptstyle{\pm 1.5}$ & $49.9 \scriptstyle{\pm 5.6}$ & $37.9$ & $75.6 \scriptstyle{\pm 0.3}$ & $14.6 \scriptstyle{\pm 0.3}$ & $75.4 \scriptstyle{\pm 0.6}$ & $88.9 \scriptstyle{\pm 0.4}$ & $52.8$ \\
DE     & $95.2 \scriptstyle{\pm 0.1}$ & $18.5 \scriptstyle{\pm 0.1}$ & $65.1 \scriptstyle{\pm 0.9}$ & $\underline{76.1 \scriptstyle{\pm 3.2}}$ & $50.5$ & $\mathbf{80.1 \scriptstyle{\pm 0.2}}$ & $21.3 \scriptstyle{\pm 0.3}$ & $\underline{84.1 \scriptstyle{\pm 0.3}}$ & $\underline{90.7 \scriptstyle{\pm 0.2}}$ & $\underline{72.2}$ \\
\bottomrule
\end{tabular}
\end{table}

\newpage
\section{Ablations on Epistemic Uncertainty Quantification}
\label{App:EUQ}
\subsection{Maximum Mean Imprecision}
\label{App:mmi}
The exact maximum mean imprecision (MMI) for measuring EU under total variation~\citep{chau2025integral} is defined as
\begin{equation}
\mathrm{MMI}(\underline{P}) := \sup_{A\in2^\mathcal{Y}}\big(\overline{P}(A)-\underline{P}(A)\big),
\label{eq:extractMMI}
\end{equation}
where $2^\mathcal{Y}$ denotes the power set over the output space $\mathcal{Y}$, and the upper and lower probabilities $\overline{P}(A)$ and $\underline{P}(A)$ bound the probability of any event $A\subseteq\mathcal{Y}$. Given our credal set defined by~\cref{eq:credal_prediction}, i.e.,
\begin{equation}
C_{\theta}=\bigl\{ (1-\epsilon)\,{p}_{+, \theta} + \epsilon\,{p}_{-, \theta} \mid \epsilon \in [0,1] \bigr\},
\nonumber
\end{equation}
the induced probability $p_\epsilon(A)$ of any event $A$ is affine in $\epsilon$, so its extrema over $C$ are attained at $\epsilon \in \{0,1\}$, as follows:
\[
\overline{P}(A) = \max\big\{ p_{+, \theta} (A),\,p_{-, \theta} (A)\big\}, \qquad
\underline{P}(A) = \min\big\{(p_{+, \theta} (A),\,p_{-, \theta} (A)\big\}.
\]
Therefore, the MMI in~\cref{eq:extractMMI} reduces exactly to the total variation distance between the two predictive distributions:
\begin{equation}
\mathrm{MMI}(\underline{P}) = \sup_{A\in2^\mathcal{Y}}\big\lvert p_{+, \theta} (A)-p_{-, \theta} (A)\big\rvert
= \frac{1}{2}\sum\nolimits_{k=1}^{K} \bigl\lvert p_{+, \theta, k} - p_{-, \theta, k} \bigr\rvert,
\label{eq:mmi-tv}
\end{equation}
where the last equality is the standard characterisation of the total variation distance in terms of the underlying probability mass functions.

\subsection{Shannon Entropy Difference}
\label{App:EntropyDifference}
Since reliable uncertainty quantification depends jointly on the representation and the uncertainty measure~\citep{wang2026set}, in addition to MMI, we also consider the Shannon entropy difference (\(\mathrm{H_{diff}}\))~\citep{abellan2006disaggregated} as an alternative, defined as the range of entropy over the credal set:
\begin{equation}
\mathrm{H_{diff}}(x) = \max_{{p}\in C_{\theta}(x)} H({p}) - \min_{{p}\in C_{\theta}(x)} H({p}),
\label{eq:eu_h_diff}
\end{equation}
where \(H({p}) := -\sum_{k=1}^K p_k \log_2 p_k\) denotes the Shannon entropy of a predictive distribution. Unlike MMI, however, the maximisation in~\cref{eq:eu_h_diff} does not admit a closed-form solution and requires numerical optimisation, incurring a small additional computational overhead.

\textbf{Experimental Evaluation.}
We compare MMI and \(\mathrm{H_{diff}}\) for EU quantification using our POCC method on AUARC in selective classification under three annotation imprecision scenarios: annotator disagreement (\tablename~\ref{tab:mmi_hdiff_dis}), teacher-model predictions (\tablename~\ref{tab:mmi_hdiff_teacher}), and label smoothing (\tablename~\ref{tab:mmi_hdiff_smoothing}), corresponding to the evaluation settings in Section~\ref{sec: Exp}. The results show that MMI consistently achieves better EU quantification performance while being computationally more efficient.
\begin{table}[!htbp]
\centering
\small
\caption{Comparison of MMI and \(\mathrm{H_{diff}}\) on AUARC ($\uparrow$) in selective classification under annotator-disagreement annotation imprecision.}
\begin{tabular}{lccc}
\toprule
Metric & CIFAR-10H & MiceBone & TreeVersity \\
\midrule
MMI & $97.1 \scriptstyle{\pm 0.2}$ & $97.3 \scriptstyle{\pm 0.3}$ & $96.7 \scriptstyle{\pm 0.2}$ \\
H-Diff & $96.6 \scriptstyle{\pm 0.2}$ & $95.0 \scriptstyle{\pm 0.7}$ & $95.7 \scriptstyle{\pm 0.2}$ \\
\bottomrule
\end{tabular}
\label{tab:mmi_hdiff_dis}
\end{table}

\begin{table}[!htbp]
\centering
\small
\caption{Comparison of MMI and \(\mathrm{H_{diff}}\) on AUARC ($\uparrow$) in selective classification under annotation imprecision from teacher model predictions.}
\begin{tabular}{lcccccc}
\toprule
 & \multicolumn{2}{c}{ShuffleNetV2-0.5x} & \multicolumn{2}{c}{ResNet-20} & \multicolumn{2}{c}{MobileNetV2-x0-5} \\
\cmidrule(lr){2-3} \cmidrule(lr){4-5} \cmidrule(lr){6-7}
Dataset & MMI  & H-Diff & MMI  & H-Diff  & MMI & H-Diff  \\
\midrule
CIFAR-10  & $99.4 \scriptstyle{\pm 0.0}$ & $99.1 \scriptstyle{\pm 0.0}$ & $99.5 \scriptstyle{\pm 0.0}$ & $99.3 \scriptstyle{\pm 0.0}$ & $99.5 \scriptstyle{\pm 0.0}$ & $99.2 \scriptstyle{\pm 0.0}$ \\
CIFAR-100 & $90.2 \scriptstyle{\pm 0.3}$ & $83.0 \scriptstyle{\pm 0.7}$ & $90.6 \scriptstyle{\pm 0.2}$ & $84.5 \scriptstyle{\pm 0.6}$ & $90.9 \scriptstyle{\pm 0.2}$ & $83.7 \scriptstyle{\pm 0.4}$ \\
\bottomrule
\end{tabular}
\label{tab:mmi_hdiff_teacher}
\end{table}

\begin{table}[!htbp]
\centering
\small
\caption{Comparison of MMI and \(\mathrm{H_{diff}}\) on AUARC ($\uparrow$) in selective classification under annotation imprecision from label smoothing.}
\begin{tabular}{lccc}
\toprule
Dataset & Smoothing $\epsilon$ & MMI  & H-Diff \\
\midrule
\multirow{4}{*}{CIFAR-10} & 0.05 & $99.2 \scriptstyle{\pm 0.1}$ & $99.1 \scriptstyle{\pm 0.1}$ \\
 & 0.10 & $98.8 \scriptstyle{\pm 0.2}$ & $98.6 \scriptstyle{\pm 0.2}$ \\
 & 0.15 & $98.8 \scriptstyle{\pm 0.2}$ & $98.1 \scriptstyle{\pm 0.2}$ \\
 & 0.20 & $98.7 \scriptstyle{\pm 0.1}$ & $97.3 \scriptstyle{\pm 0.2}$ \\
\midrule
\multirow{4}{*}{CIFAR-100} & 0.05 & $93.4 \scriptstyle{\pm 0.1}$ & $92.2 \scriptstyle{\pm 0.1}$ \\
 & 0.10 & $93.3 \scriptstyle{\pm 0.1}$ & $91.5 \scriptstyle{\pm 0.2}$ \\
 & 0.15 & $93.1 \scriptstyle{\pm 0.1}$ & $90.7 \scriptstyle{\pm 0.2}$ \\
 & 0.20 & $92.8 \scriptstyle{\pm 0.2}$ & $89.6 \scriptstyle{\pm 0.3}$ \\
\bottomrule
\end{tabular}
\label{tab:mmi_hdiff_smoothing}
\end{table}

\newpage
\textbf{Robustness of POCC on EU Quantification Measures.}
When using \(\mathrm{H_{diff}}\) as the EU measure, our method consistently achieves the best balanced quality scores (BQSs) across all scenarios (Figures~\ref{fig:balance-dis}, ~\ref{fig:balance-teacher}), and ~\ref{fig:balance-smoothing}, with one exception: on CIFAR-10 under label smoothing, it ranks second, behind the 5-member Deep Ensemble. 
\begin{figure}[!htbp]
    \centering
    \includegraphics[width=0.5\linewidth]{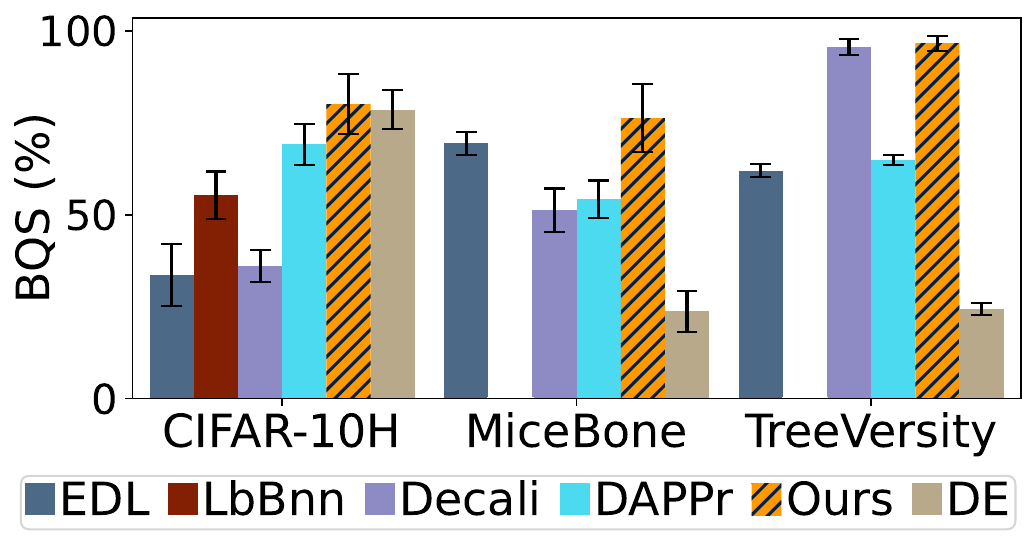}
    \caption{BQS of our method under annotator-disagreement annotation imprecision, using \(\mathrm{H_{diff}}\) for EU quantification.}
    \label{fig:balance-dis}
\end{figure}
\begin{figure}[!htbp]
    \centering
    \includegraphics[width=\linewidth]{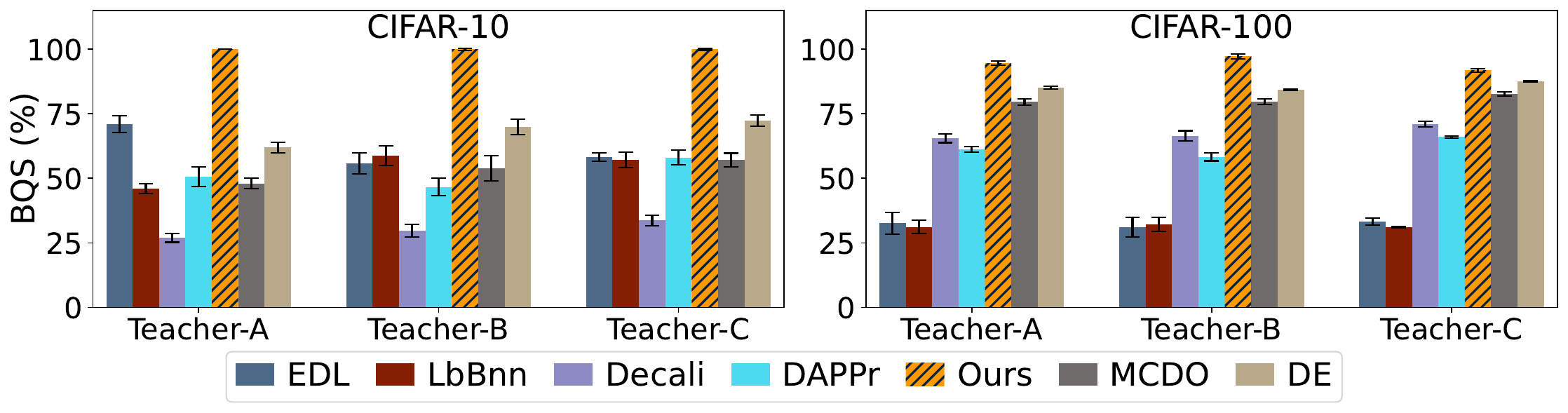}
    \caption{BQS of our method under teacher-prediction annotation imprecision, using \(\mathrm{H_{diff}}\) for EU quantification. Teacher-A: ResNet-20; Teacher-B: ShuffleNetV2-0.5x; Teacher-C: MobileNetV2-x0-5.}
    \label{fig:balance-teacher}
\end{figure}
\begin{figure}[!htbp]
    \centering
    \includegraphics[width=\linewidth]{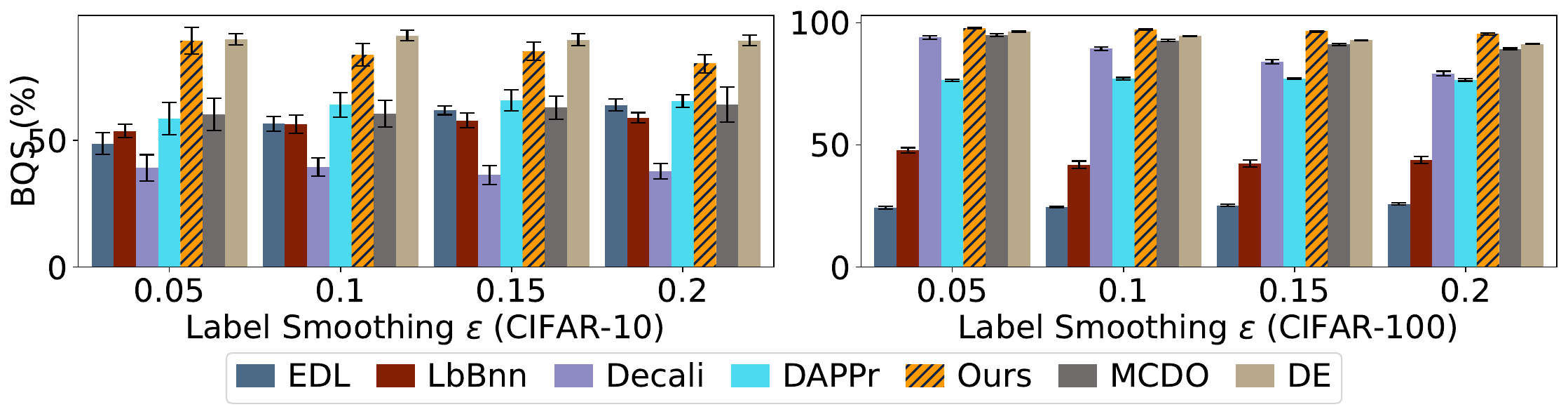}
    \caption{BQS of our method under label-smoothing annotation imprecision with varying $\epsilon$, using \(\mathrm{H_{diff}}\) for EU quantification.}
    \label{fig:balance-smoothing}
\end{figure}

\newpage
\section{On the Exclusion of OOD Detection for EU Evaluation}
\label{App:WhyNotOOD}
Below we detail why out-of-distribution (OOD) detection is not an appropriate benchmark for EU assessment in our work.

\textbf{Conceptual Mismatch.}
OOD detection asks whether a test input \(x\) was drawn from the same marginal distribution \(P_X\) as the training data---an \emph{input-space} question. The EU studied in our framework asks something fundamentally different: \emph{given \(x\), how much remains underdetermined about the true conditional label distribution because the supervision itself is imprecise?}---a \emph{label-space} question. These two sources of uncertainty are conceptually orthogonal. A perfectly in-distribution image can receive highly divergent human annotations (high annotation-induced EU, yet assuredly in-distribution), while a clearly OOD input can be unambiguous to any annotator (low annotation-induced EU, yet OOD). Using OOD-detection performance as a proxy for EU therefore conflates two distinct epistemic phenomena that need not correlate.

\textbf{Methodological Critique.}
This mismatch is not unique to our setting. \citet{li2025position} argue that supervised classifiers answer the wrong question for OOD detection: high predictive uncertainty and large feature distance are conflated with being OOD, yielding irreducible errors that neither scaling, epistemic uncertainty representations, nor outlier exposure can remove. The critique applies with particular force to POCC, whose EU measures the size of the predictive credal set induced by imprecise annotations. Nothing in its training objective ties this quantity to deviation from \(P_X\).

\textbf{Absence of a Meaningful Reference Signal.}
Even setting aside the above, standard OOD benchmarks provide no reference signal for annotation-induced EU. The proxy rests on the assumption that OOD samples \emph{should} exhibit higher EU than in-distribution ones---an assumption that is debatable in conventional settings~\citep{li2025position} and untenable here: all test samples in our experiments are drawn from the same input distribution as the training data. EU differences arise from variation in the informativeness and consistency of supervision, not from variation in input provenance, so an OOD benchmark cannot distinguish between instances that are genuinely annotation-ambiguous and those that are not.

Selective classification, by contrast, offers a task-oriented evaluation with a clear operational interpretation aligned with our framework's objective. A good EU estimate should rank genuinely annotation-ambiguous instances higher, so that abstaining on high-uncertainty predictions improves accuracy on the retained subset. This provides operational evidence that the learned disagreement between pessimistic and optimistic objectives is useful for ranking prediction errors---without claiming that the predictive segment faithfully recovers the true credal set induced by annotation imprecision, or that OOD detection serves as a valid proxy for EU.

\textbf{Additional Experiments.}
Although OOD detection is not an appropriate benchmark for EU evaluation in our work (as discussed above), one may still be interested in how our method performs on OOD detection. We therefore conduct the following additional experiments.

We consider models trained under three annotation imprecision settings: (i) CIFAR-10H; (ii) CIFAR-10 with ResNet-20 as the teacher in a knowledge distillation setting; and (iii) CIFAR-10 with label smoothing at \(\epsilon=0.05\). For OOD detection, CIFAR-10 is used as the in-distribution dataset, with several OOD datasets: SVHN~\citep{hendrycks2021natural}, Places365~\citep{Places}, FMNIST~\citep{xiao2017fashionmnist}, and ImageNet~\citep{deng2009imagenet}. For a fair comparison with multiple-forward-pass baselines, we additionally include an ensemble of $5$ POCC models, denoted Our-Ens.

Let \(M\) denote the number of independently trained POCC models (via random initialisation, same as DE) with parameters \(\theta^{(1)}, \ldots, \theta^{(M)}\). Each model \(m\) produces two vertices in the
probability simplex via its pessimistic and optimistic heads, \(p_{-,\theta^{(m)}}(x), \quad p_{+,\theta^{(m)}}(x) \in \Delta^{K-1},\) which together define a line-segment credal set \(C_{\theta^{(m)}}(x)\) as in
~\cref{eq:credal_prediction}. Aggregating across all \(M\) models, we collect the full set of \(2M\) vertices,
\begin{equation}
    \mathcal{V}(x)
    \;=\;
    \bigcup_{m=1}^{M}
    \bigl\{
        p_{-,\theta^{(m)}}(x),\;
        p_{+,\theta^{(m)}}(x)
    \bigr\}
    \;\subset\;
    \Delta^{K-1}.
    \label{eq:ensemble_vertices}
\end{equation}
The ensemble credal set is then defined as the convex hull of \(\mathcal{V}(x)\),
\begin{equation}
    \mathrm{conv}\!\left(\mathcal{V}(x)\right)
    \!=\!\!
    \left\{
        \sum_{m=1}^{M}
        \bigl(
            \alpha_{m}^{-}\, p_{-,\theta^{(m)}}(x)
            \!\!+\!\!
            \alpha_{m}^{+}\, p_{+,\theta^{(m)}}(x)
        \bigr)
        \!\Big|\!
        \alpha_{m}^{-}, \alpha_{m}^{+} \geq 0,\;
        \!\!\!\!\sum_{m=1}^{M}\!\left(\alpha_{m}^{-} + \alpha_{m}^{+}\right) \!=\! 1
    \!\right\},
    \label{eq:ensemble_credal}
\end{equation}
which is a convex polytope in \(\Delta^{K-1}\) with at most \(2M\) vertices. EU is quantified as the Shannon entropy difference in ~\cref{eq:eu_h_diff}.

\textbf{\emph{Results.}}
Tables~\ref{Tab: OOD-Dis}, \ref{Tab: OOD-Teacher}, and~\ref{Tab: OOD-Smooth} report the OOD detection results. POCC achieves comparable performance to the single-model baselines (ranking within the top 2), while the POCC ensemble attains the best OOD detection performance overall.
\begin{table}[!htbp]
\centering
\small
\caption{AUROC (\%) for OOD detection using epistemic uncertainty, with CIFAR-10 (models trained using CIFAR-10H) as the ID dataset, averaged over 10 runs. Methods are grouped into \textit{single-model} approaches (above the line) and \textit{multiple-forward-pass} approaches, including an ensemble of our POCCs (Our-Ens), MCDO, and DE (below the line). Bold and underlined entries denote the best and second-best results within each group, respectively.}
\label{Tab: OOD-Dis}
\begin{tabular}{@{}lcccc@{}}
\toprule
& SVHN        & Places      & FMNIST      & ImageNet    \\ \midrule 
EDL        & $0.735 \scriptstyle{\pm 0.004}$ & $0.717 \scriptstyle{\pm 0.004}$ & $0.721 \scriptstyle{\pm 0.007}$ & $0.701 \scriptstyle{\pm 0.004}$ \\
Decali     & $\mathbf{0.854 \scriptstyle{\pm 0.011}}$ & $\mathbf{0.832 \scriptstyle{\pm 0.012}}$ & $\mathbf{0.846 \scriptstyle{\pm 0.014}}$ & $\mathbf{0.807 \scriptstyle{\pm 0.011}}$ \\
DAPPr      & $0.730 \scriptstyle{\pm 0.007}$ & $0.707 \scriptstyle{\pm 0.006}$ & $0.699 \scriptstyle{\pm 0.009}$ & $0.680 \scriptstyle{\pm 0.006}$ \\
\rowcolor{OxfordBlueLight}
Ours       & $\underline{0.840 \scriptstyle{\pm 0.009}}$ & $\underline{0.816 \scriptstyle{\pm 0.009}}$ & $\underline{0.830 \scriptstyle{\pm 0.013}}$ & $\underline{0.785 \scriptstyle{\pm 0.007}}$ \\
\midrule
\rowcolor{OxfordBlueLight}
Our-Ens    & $\mathbf{1.000 \scriptstyle{\pm 0.000}}$ & $\mathbf{0.992 \scriptstyle{\pm 0.001}}$ & $\mathbf{0.997 \scriptstyle{\pm 0.001}}$ & $\mathbf{0.990 \scriptstyle{\pm 0.001}}$ \\
MCDO       & $0.957 \scriptstyle{\pm 0.009}$ & $0.930 \scriptstyle{\pm 0.009}$ & $0.945 \scriptstyle{\pm 0.009}$ & $0.904 \scriptstyle{\pm 0.010}$ \\
DE         & $\underline{0.981 \scriptstyle{\pm 0.003}}$ & $\underline{0.958 \scriptstyle{\pm 0.004}}$ & $\underline{0.971 \scriptstyle{\pm 0.003}}$ & $\underline{0.937 \scriptstyle{\pm 0.005}}$ \\
\bottomrule
\end{tabular}
\end{table}

\begin{table}[!htbp]
\centering
\small
\caption{AUROC (\%) for OOD detection using epistemic uncertainty, with CIFAR-10 (model trained using the ResNet-20 as the teacher model with $T=2.5$) as the ID dataset, averaged over 10 runs. Methods are grouped into \textit{single-model} approaches (above the line) and \textit{multiple-forward-pass} approaches, including an ensemble of our POCCs (Our-Ens), MCDO, and DE (below the line). Bold and underlined entries denote the best and second-best results within each group, respectively.}
\label{Tab: OOD-Teacher}
\begin{tabular}{@{}lcccc@{}}
\toprule
& SVHN        & Places      & FMNIST      & ImageNet    \\ \midrule 
EDL        & $0.717 \scriptstyle{\pm 0.008}$ & $0.696 \scriptstyle{\pm 0.006}$ & $0.693 \scriptstyle{\pm 0.010}$ & $0.670 \scriptstyle{\pm 0.005}$ \\
Decali     & $\mathbf{0.750 \scriptstyle{\pm 0.003}}$ & $\underline{0.713 \scriptstyle{\pm 0.003}}$ & $\mathbf{0.737 \scriptstyle{\pm 0.005}}$ & $\mathbf{0.694 \scriptstyle{\pm 0.002}}$ \\
DAPPr      & $0.609 \scriptstyle{\pm 0.004}$ & $0.605 \scriptstyle{\pm 0.003}$ & $0.603 \scriptstyle{\pm 0.003}$ & $0.601 \scriptstyle{\pm 0.003}$ \\
\rowcolor{OxfordBlueLight}
Ours       & $\underline{0.743 \scriptstyle{\pm 0.003}}$ & $\mathbf{0.721 \scriptstyle{\pm 0.003}}$ & $\underline{0.730 \scriptstyle{\pm 0.006}}$ & $\underline{0.688 \scriptstyle{\pm 0.003}}$ \\
\midrule
\rowcolor{OxfordBlueLight}
Our-Ens    & $\mathbf{1.000 \scriptstyle{\pm 0.000}}$ & $\mathbf{0.989 \scriptstyle{\pm 0.001}}$ & $\mathbf{0.997 \scriptstyle{\pm 0.000}}$ & $\mathbf{0.980 \scriptstyle{\pm 0.001}}$ \\
MCDO       & $0.944 \scriptstyle{\pm 0.001}$ & $0.902 \scriptstyle{\pm 0.002}$ & $0.940 \scriptstyle{\pm 0.004}$ & $0.852 \scriptstyle{\pm 0.002}$ \\
DE         & $\underline{0.968 \scriptstyle{\pm 0.001}}$ & $\underline{0.917 \scriptstyle{\pm 0.001}}$ & $\underline{0.956 \scriptstyle{\pm 0.003}}$ & $\underline{0.865 \scriptstyle{\pm 0.002}}$ \\
\bottomrule
\end{tabular}
\end{table}

\begin{table}[!htbp]
\centering
\small
\caption{AUROC (\%) for OOD detection using epistemic uncertainty, with CIFAR-10 (model trained using label smoothing $\varepsilon = 0.05$) as the ID dataset, averaged over 10 runs. Methods are grouped into \textit{single-model} approaches (above the line) and \textit{multiple-forward-pass} approaches, including an ensemble of our POCCs (Our-Ens), MCDO, and DE (below the line). Bold and underlined entries denote the best and second-best results within each group, respectively.}
\label{Tab: OOD-Smooth}
\begin{tabular}{@{}lcccc@{}}
\toprule
& SVHN        & Places      & FMNIST      & ImageNet    \\ \midrule 
EDL        & $0.715 \scriptstyle{\pm 0.008}$ & $0.696 \scriptstyle{\pm 0.006}$ & $0.691 \scriptstyle{\pm 0.008}$ & $0.671 \scriptstyle{\pm 0.004}$ \\
Decali     & $\underline{0.814 \scriptstyle{\pm 0.020}}$ & $\underline{0.780 \scriptstyle{\pm 0.026}}$ & $\underline{0.784 \scriptstyle{\pm 0.020}}$ & $\underline{0.755 \scriptstyle{\pm 0.013}}$ \\
DAPPr      & $0.744 \scriptstyle{\pm 0.003}$ & $0.706 \scriptstyle{\pm 0.002}$ & $0.720 \scriptstyle{\pm 0.007}$ & $0.664 \scriptstyle{\pm 0.002}$ \\
\rowcolor{OxfordBlueLight}
Ours       & $\mathbf{0.872 \scriptstyle{\pm 0.009}}$ & $\mathbf{0.836 \scriptstyle{\pm 0.008}}$ & $\mathbf{0.846 \scriptstyle{\pm 0.007}}$ & $\mathbf{0.798 \scriptstyle{\pm 0.010}}$ \\
\midrule \midrule
\rowcolor{OxfordBlueLight}
Our-Ens    & $\mathbf{0.999 \scriptstyle{\pm 0.000}}$ & $\mathbf{0.989 \scriptstyle{\pm 0.001}}$ & $\mathbf{0.997 \scriptstyle{\pm 0.000}}$ & $\mathbf{0.976 \scriptstyle{\pm 0.001}}$ \\
MCDO       & $0.920 \scriptstyle{\pm 0.002}$ & $0.891 \scriptstyle{\pm 0.002}$ & $0.906 \scriptstyle{\pm 0.005}$ & $0.861 \scriptstyle{\pm 0.003}$ \\
DE         & $\underline{0.947 \scriptstyle{\pm 0.001}}$ & $\underline{0.911 \scriptstyle{\pm 0.001}}$ & $\underline{0.933 \scriptstyle{\pm 0.003}}$ & $\underline{0.884 \scriptstyle{\pm 0.001}}$ \\
\bottomrule
\end{tabular}
\end{table}

\end{document}